\documentclass{article}

\usepackage{microtype}
\usepackage{graphicx}
\usepackage{subfigure}
\usepackage{booktabs} 

\usepackage{hyperref}

\usepackage[accepted]{icml2023}

\usepackage{amsmath}
\usepackage{amssymb}
\usepackage{mathtools}
\usepackage{amsthm}

\usepackage[capitalize,noabbrev]{cleveref}

\theoremstyle{plain}
\newtheorem{theorem}{Theorem}[section]

\newtheorem{lemma}[theorem]{Lemma}

\theoremstyle{definition}
\newtheorem{definition}[theorem]{Definition}

\theoremstyle{remark}

\usepackage[textsize=tiny]{todonotes}

\usepackage{multirow}
\usepackage{adjustbox}
\usepackage{colortbl}
\definecolor{c1}{HTML}{44cef6}
\usepackage{enumitem}
\def\Std{\text{Std}}
\usepackage{xcolor}
\usepackage{minitoc}

\icmltitlerunning{What is Essential for Unseen Goal Generalization of Offline Goal-conditioned RL?}

\begin{document}

\twocolumn[
\icmltitle{What is Essential for Unseen Goal Generalization of \\ Offline Goal-conditioned RL?}



\icmlsetsymbol{equal}{*}

\begin{icmlauthorlist}
\icmlauthor{Rui Yang}{yyy}
\icmlauthor{Yong Lin}{yyy}
\icmlauthor{Xiaoteng Ma}{ttt}
\icmlauthor{Hao Hu}{ttt}
\icmlauthor{Chongjie Zhang}{ttt}
\icmlauthor{Tong Zhang}{yyy}
\end{icmlauthorlist}

\icmlaffiliation{yyy}{The Hong Kong University of Science and Technology}
\icmlaffiliation{ttt}{Tsinghua University}

\icmlcorrespondingauthor{Tong Zhang}{tongzhang@tongzhang-ml.org}

\icmlkeywords{Machine Learning, ICML}

\vskip 0.3in
]



\printAffiliationsAndNotice{} 

\begin{abstract}
Offline goal-conditioned RL (GCRL) offers a way to train general-purpose agents from fully offline datasets. In addition to being conservative within the dataset, the generalization ability to achieve unseen goals is another fundamental challenge for offline GCRL. However, to the best of our knowledge, this problem has not been well studied yet. In this paper, we study out-of-distribution (OOD) generalization of offline GCRL both theoretically and empirically to identify factors that are important. In a number of experiments, we observe that weighted imitation learning enjoys better generalization than pessimism-based offline RL method. Based on this insight, we derive a theory for OOD generalization, which characterizes several important design choices. We then propose a new offline GCRL method, \textbf{G}eneralizable \textbf{O}ffline go\textbf{A}l-condi\textbf{T}ioned RL (\textbf{GOAT}), by combining the findings from our theoretical and empirical studies. On a new benchmark containing 9 independent identically distributed (IID) tasks and 17 OOD tasks, GOAT outperforms current state-of-the-art methods by a large margin. 
\end{abstract}

\begin{figure*}[h]
    \centering
\includegraphics[width=1\linewidth]{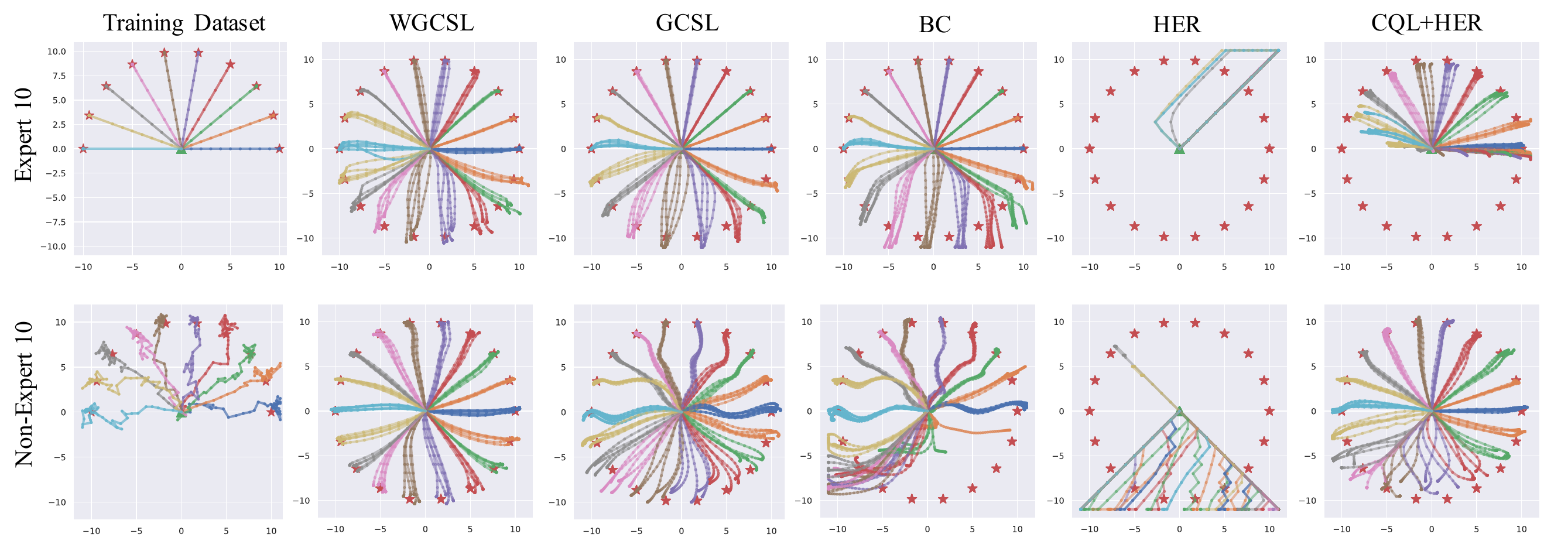}
    \caption{Training datasets and trajectories generated by different agents trained on ``Expert 10" and ``Non-Expert 10" datasets.}
    \label{fig:visual_trajs_example}
\end{figure*}

\section{Introduction}
Deep reinforcement learning (DRL) makes it possible for a learning agent to achieve superhuman performance on a range of challenging tasks \cite{silver2016mastering, silver2018general,vinyals2019grandmaster,li2020suphx}. However, recent studies have found that DRL is prone to overfitting the training tasks and is sensitive to environmental changes \cite{cobbe2019quantifying, wang2020improving,han2021learning,kirk2023survey}. Goal-conditioned reinforcement learning (GCRL) is gaining increasing attention because it enables learning general-purpose decision-making rather than overfitting to a single task \cite{andrychowicz2017hindsight,ghosh2019learning,li2020generalized}. Particularly, offline GCRL \cite{chebotar2021actionable,yang2022rethinking}, which learns as many skills as possible from previously collected datasets without any exploration in the environment, is promising for large-scale and general-purpose pre-training. 
Nevertheless, prior works \cite{chebotar2021actionable,yang2022rethinking,ma2022far} have largely focused on reaching goals in the dataset, without systematically studying the problem of out-of-distribution (OOD) goal generalization. There are a number of questions: what is the OOD generalization performance of current offline GCRL algorithms? And more importantly, \textbf{what is essential for OOD generalization of offline GCRL?} 

To answer these questions, we first design a 2D goal-reaching task with different types of offline data. We find that (1) pessimism-based offline RL is restrained from generalizing to OOD goals and (2) imitation learning overfits the data noise and fails to generalize when given non-expert data. On the contrary, (3) weighted imitation learning is a strong baseline for OOD generalization across different types of training data. The observation motivates us to derive a generalization theory from the perspective of domain generalization \cite{muandet2013domain,zhang2012generalization,zhou2021domain}. Through analyzing our theory, we find several techniques that are essential to minimize the generalization bound, including advantage re-weighting, data selection, density re-weighting, and goal-relabeling. Particularly, we find re-weighting the training state-goal distribution with the reciprocal of its density can minimize the worst-case distribution shift. Based on these results, we propose, \textbf{G}eneralizable \textbf{O}ffline go\textbf{A}l-condi\textbf{T}ioned RL (\textbf{GOAT}), by integrating these techniques into a general weighted imitation learning framework, which encourages optimistic goal sampling while still maintaining pessimism on action selection.

Due to the lack of benchmarks for evaluating the OOD generalization performance of offline GCRL, we develop a challenging robot manipulation benchmark based on a robotic arm or an anthropomorphic hand. The benchmark comprises nine offline datasets and 26 evaluation tasks, 9 of which contain independent and identically distributed (IID) goals, while the rest 17 tasks involve various types of OOD goals. In our experiments\footnote{Code is available at \href{https://github.com/YangRui2015/GOAT}{https://github.com/YangRui2015/GOAT}}, we demonstrate that GOAT considerably improves the OOD generalization performance of existing offline GCRL methods, as well as enhances efficiency in online fine-tuning for unseen goals. Furthermore, we conduct in-depth ablation studies to validate the effectiveness of each component used in GOAT, which may benefit future research on OOD generalization for offline RL.

\section{Preliminaries}

\subsection{Goal-conditioned RL} 
Goal-conditioned RL (GCRL) considers a goal-augmented Markov Decision Process (GMDP), denoted by a tuple $(\mathcal{S},\mathcal{A}, \mathcal{G}, \mathcal{P},r,\gamma)$. $\mathcal{S}$, $\mathcal{G}$ $\mathcal{A}$ refer to state, goal, and action spaces, respectively. $\gamma$ is the discount factor, and $r: \mathcal{S} \times \mathcal{G} \times \mathcal{A} \to \mathbb{R}$ is the goal-conditioned reward function. Generally, we consider a sparse and binary reward function $r(s,a,g)=1[\|\phi(s)-g\|_2^2 \leq \delta]$, where $\delta$ is a threshold and $\phi$ is a known state-to-goal mapping \cite{andrychowicz2017hindsight}. A policy $\pi:\mathcal{S} \times \mathcal{G} \to \mathcal{A}$ aims to maximize the expected return:
\begin{equation*}
\begin{aligned}
     J(\pi)=\mathbb{E}_{g\sim p(g),s_0 \sim \mu(s_0),\atop a_t \sim \pi(\cdot |s_t,g),
     s_{t+1}\sim \mathcal{P}(\cdot|s_t,a_t)} \big[ 
     \sum^{\infty}_{t=0} \gamma^{t} r(s_t,a_t, g) \big],
\end{aligned}
\end{equation*}
where $\mu(s_0)$ is the distribution of initial states. The value function is defined as 
$V^{\pi}(s,g) = \mathbb{E}_{a_t \sim \pi(\cdot|s_t,g), s_{t+1}\sim \mathcal{P}(\cdot|s_t,a_t)}\big[\sum^{\infty}_{t=0} \gamma^{t} r(s_t,a_t, g) |s_0=s\big]$. 
For offline GCRL, the agent cannot interact with the environment during training, and the training data is sampled from a static dataset $D=\{(s_t,a_t,g,r_t,s_{t+1})\}$.

\subsection{Domain Generalization}
Domain Generalization (DG) was first studied in the supervised learning setting \cite{blanchard2011generalizing}. A domain is defined as a joint distribution $P_{XY}$ on $\mathcal{X} \times \mathcal{Y}$, where $\mathcal{X}$ is the input space and $\mathcal{Y}$ is the label space. DG learns a model from $K$ different training domains $\mathcal{S}=\{(x^{(k)}, y^{(k)})\}_{k=1}^{K}$ that aims to generalize on unseen testing domains $\mathcal{T}=\{x^{\mathcal{T}}\}, P_{XY}^{\mathcal{T}} \neq P_{XY}^{k}, k\in\{1,\cdots,K\}$. DG mainly handles covariate shift \cite{zhou2021domain}, assuming that the labeling function $P_{Y|X}$ is stable across domains \cite{muandet2013domain} and only the marginal distribution changes $P_X^{\mathcal{T}} \neq  P_X^k, k\in \{1,\cdots, K\}$.

\section{OOD Generalization for Offline GCRL}
In this section, we first compare different GCRL algorithms in a 2D goal-reaching environment, showing that weighted imitation learning method is preferable to other methods across different data settings. Based on the observations, we formulate the OOD generalization problem as domain generalization, and then derive a theoretical framework to analyze the essential techniques for OOD generalization.

\begin{figure}[t]
\centering
\subfigure[]{\includegraphics[width=1\linewidth]{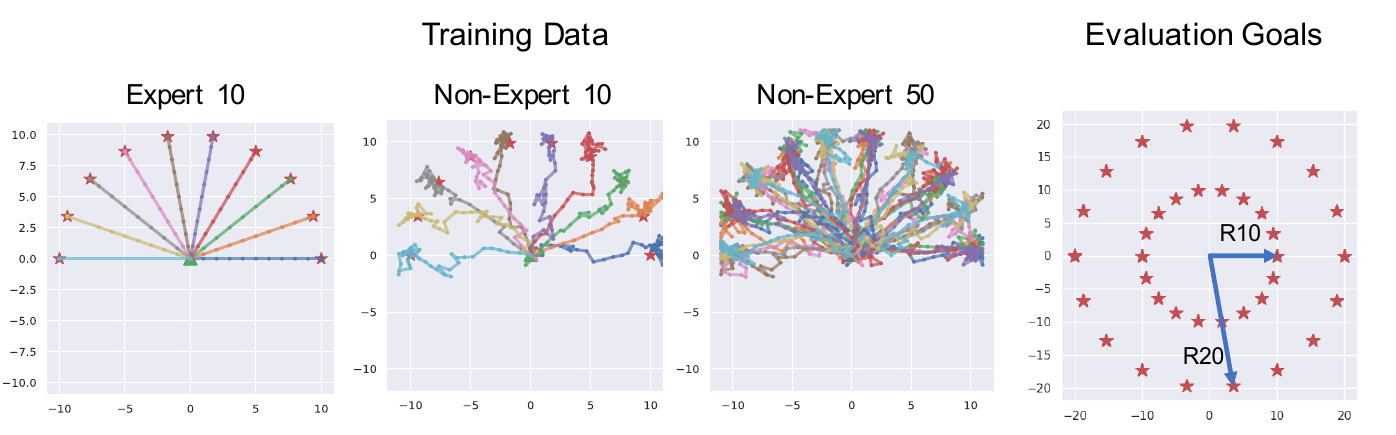}\label{fig:point_datas}}

\subfigure[]{\includegraphics[width=1\linewidth]{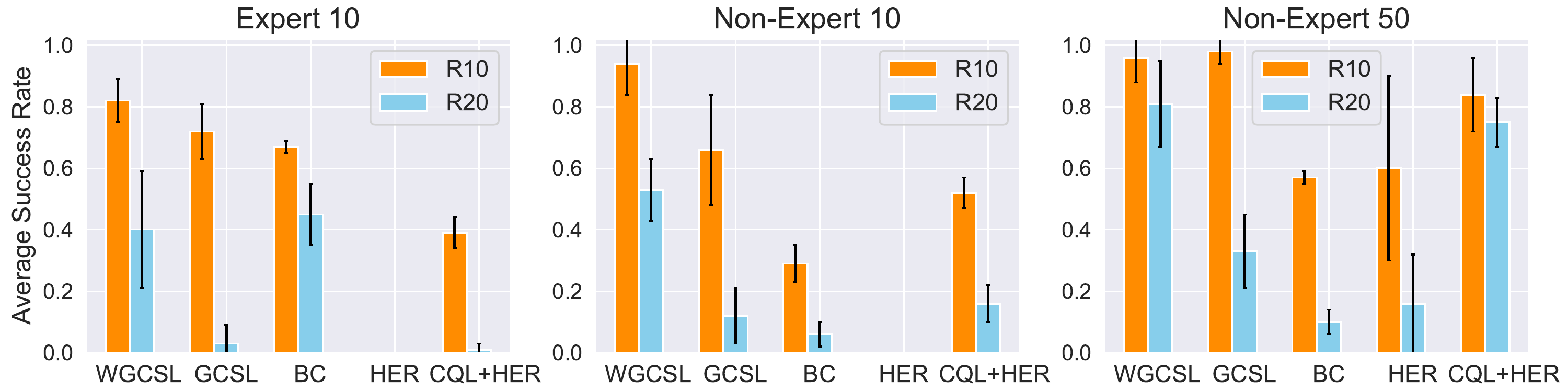}\label{fig:bars1}}
\caption{(a) Visualization of three 2D goal-reaching datasets and two groups of evaluation goals. ``R10" and ``R20" refer to the radius (10 or 20) of the desired goals for evaluation. (b) Average success rates of different agents over 5 random seeds. }
\label{fig:point_bars}
\end{figure}

\subsection{Didactic Example}
\label{sec:didactic_example}
We design a 2D point environment as shown in Figure \ref{fig:point_datas} to characterize the generalization ability of different offline GCRL algorithms, including BC, GCSL \cite{ghosh2019learning}, WGCSL \cite{yang2022rethinking}, DDPG+HER \cite{andrychowicz2017hindsight}, and CQL+HER \cite{chebotar2021actionable}. There are three types of training data, namely ``Expert $N$" and ``Non-Expert $N$", where $N$ refers to the number of trajectories in the dataset. In the training datasets, trajectories and goals are mainly distributed on the top semicircle with a radius of 10. Unlike the training data, the evaluation goals are on the full circles of radius 10 and 20. Both states and goals in this environment are represented as 2D coordinates indicating their positions, while actions are 2D vectors of the displacement. In this example, the optimal policy is $\pi(s,g)=\text{clip}(g-s,0,1)$, where the maximum movement in one dimension is $1$. If the agent learns the optimal policy, it can successfully generalize to any unseen goal.

From the results in Figure \ref{fig:visual_trajs_example} and Figure \ref{fig:bars1}, we can draw the following conclusions:
\begin{itemize}[itemsep=6pt,topsep=0pt,parsep=0pt]
    \item Given a clean expert dataset, BC generalizes well for OOD goals. However, in the case of training with non-expert and noisy data, it can overfit the noise and thus fail to generalize.
    \item DDPG+HER (short for ``HER") suffers from overestimating values of OOD actions. As a result, it avoids in-dataset actions and produces odd trajectories.
  \item For the pessimism-based approach CQL+HER, its trajectories are restricted to the upper semicircle and fail to generalize to the lower part when given clean expert data. It can only generalize relatively well when the data size and coverage are sufficiently large.
  \item WGCSL significantly improves the OOD generalization ability over GCSL by re-weighting samples and performs consistently well across different datasets.
\end{itemize}

The designed task is simple but representative for characterizing the characteristics of different algorithms. More results can be found in Appendix \ref{ap:2d_reach}. As suggested by the empirical results, the weighted imitation-based method enjoys better OOD generalization than pessimism-based method. Moreover, pessimism-based offline RL methods are inhibited from reaching OOD area in theory \cite{jin2021pessimism,kumar2021should}. In contrast, weighted imitation learning method has theoretical guarantees for OOD generalization, which we will show in Section \ref{sec:domain_generalization_theory}.

\subsection{Problem Formulation}
We define $\mathcal{X}=\mathcal{S} \times \mathcal{G}$ as the input space, $\mathcal{Y}=\mathcal{A}$ as the action space. The offline data $D=\{(s_t,a_t,g,r_t,s_{t+1})\}$ is collected by any behavior policy $\pi_b$, where $(s_t,g)\sim P_{X}^{\mathcal{S}}$. In the testing phase, initial states and desired goals can be sampled from any unknown distribution $P_{X}^{\mathcal{T}}, P_{X}^{\mathcal{T}}\neq P_{X}^{\mathcal{S}}$, which is named ``OOD distribution" in this paper. 
We assume the expert policy $\pi_E(a|s,g)$ (or $P_{Y|X}$) is stable with $P_{X}$ and generalizes well across different state-goal pairs, which is reasonable because OOD generalization is meaningless when $\pi_E$ cannot generalize. The objective is to minimize the suboptimality on the testing domain $P_{X}^{\mathcal{T}}$: 
\begin{equation}
\label{eq:suboptimal}
    \text{SubOpt}(\pi_E, \pi) =\mathbb{E}_{(s_0,g) \sim P_{X}^{\mathcal{T}}} [V^{\pi_E}(s_0,g) - V^{\pi}(s_0,g)] 
\end{equation}

\subsection{A Domain Generalization View}
\label{sec:domain_generalization_theory}
By establishing a link between weighted imitation learning and supervised learning, we can analyze the OOD generalization performance according to the domain generalization bound \cite{ben2010theory,zhang2012generalization, mansour2009domain}. 

Our following analysis is based on the Total Variation Distance $D_{\rm TV}$ between any two policies $\pi_1$ and $\pi_2$:
\begin{equation*}
    \begin{aligned}
        D_{\rm TV} &(\pi_1(\cdot|s,g) , \pi_2(\cdot|s,g)) = \\
        &\sup_{ B\subset \mathcal{A}} | \int_{a\in B} (\pi_1(a|s,g) - \pi_2(a|s,g))| ,
    \end{aligned}
\end{equation*}
where $B$ is any measurable subset of the action space $\mathcal{A}$.
Denote the discounted occupancy of state as 
$
    d_{\pi}(s|s_0, g) = (1-\gamma)\sum\nolimits_{t=0}^{\infty} \gamma^{t} \Pr(s_t=s|\pi,s_0, g) 
$. We define the policy discrepancy on any state-goal distribution $\rho$ as:
\begin{equation*}
    \varepsilon^{\rho}(\pi_{1}, \pi_2)=\mathbb{E}_{(s_0,g) \sim P_{X}^{\rho} \atop s \sim d_{\pi_{E}}(s|s_0,g)}\left[ D_{\mathrm{TV}}\big(\pi_1(\cdot|s,g), \pi_2(\cdot|s,g) \big)\right]
\end{equation*}
Generally, we do not have access to the true expert policy $\pi_E$, but we can imitate a surrogate policy $\hat \pi_E$ instead. Then, we provide the following OOD generalization theorem.
\begin{theorem}
\label{tm:performance_gap_bound_finite}
    Consider finite hypothesis space $\Pi$ and we minimize the empirical loss function $\hat \varepsilon^{\mathcal{S}}$ with $m$ samples. For a policy $\pi$ and a surrogate expert policy $\hat \pi_E$, with probability at least $1-\delta$, the following bound holds:
    \begin{equation*}
    \begin{aligned}
        &\mathrm{SubOpt}(\pi_E, \pi) \leq  \frac{2R_{max}}{(1-\gamma)^2} \bigg[ \underbrace{ \hat \varepsilon^{\mathcal{S}}(\hat \pi_{E}, \pi)}_{\text{empirical imitation loss}}\\
        &+ \underbrace{\varepsilon^{\mathcal{S}}(\hat \pi_{E}, \pi_E) }_{\text{expert estimation gap}}
        + \underbrace{d_1(\mathcal{T}, \mathcal{S})}_{\text{distribution shift}} + \sqrt{\frac{\log 2|\Pi|+\log \frac{1}{\delta}}{2m}}  \bigg]
    \end{aligned}
    \end{equation*}
\end{theorem}
where  $d_1(\cdot, \cdot)$ is the variation divergence defined as follows:
\begin{align*}
    d_1(S_1, S_2) = 2\sup_{\mathcal{J} \subset \mathcal{X}} \left|\int_{x \in \mathcal{J}} \left( P_{S_1}(x) -  P_{S_2}(x) \right) dx\right|,
\end{align*}
here $\mathcal{J}$ is any measurable subset of $\mathcal{X}$.

The proof is deferred to Appendix \ref{ap:proof_main}. Theorem \ref{tm:performance_gap_bound_finite} suggests that the overall OOD generalization suboptimality can be controlled by minimizing the empirical imitation learning loss, the distance between $\pi_E$ and $\hat \pi_E$, and controlling the distribution shift between training and testing domains. We now analyze how to minimize each term in this bound.

\paragraph{Empirical Imitation Loss} We can use a weighted behavior policy as the surrogate policy: $\hat \pi_E(a|s,g) \propto w(s,a,g) \pi_b (a|s,g)$. According to Pinsker’s inequality \cite{csiszar2011information}, this loss can be bounded by $\mathrm{KL}$-divergence. Thus, we have
\begin{equation*}
    \begin{aligned}
        &\min_{\theta} \mathbb{E}_{(s_0,g) \sim P_{X}^{\mathcal{S}}, s \sim d_{\pi_{E}}(s|s_0,g)}\left[ D_{\mathrm{KL}}\big(\hat \pi_E(\cdot|s,g),\pi_{\theta}(\cdot|s,g)  \big)\right] \\
        &\iff \max_{\theta} \mathbb{E}_{(s_0,g) \sim P_{X}^{\mathcal{S}}, s \sim d_{\pi_{E}}(s|s_0,g), a\sim \hat \pi_E} \left[\log \pi_{\theta}(a|s,g) \right] \\
        &\iff \max_{\theta} \mathbb{E}_{(s_0,g) \sim P_{X}^{\mathcal{S}} \atop s \sim d_{\pi_{E}}(s|s_0,g), a\sim  \pi_b} \left[\log \pi_{\theta}(a|s,g) \cdot w(s,a,g) \right]
    \end{aligned}
\end{equation*}
Empirically, following \cite{wang2018exponentially,nair2020awac} we omit the difference in $d_{\pi_E}$ and conduct weighted imitation learning on the offline data to minimize this loss.

\paragraph{Expert Estimation Gap} Although we do not have access to $\pi_E$, we know $\pi_E$ has the highest expected value. Instead of minimizing the TV distance to $\pi_E$, this problem can be reformulated as maximizing the expected value of the surrogate policy $\hat \pi_E$. Following \cite{wang2018exponentially,peng2019advantage}, advantage re-weighting $\hat \pi_E(a|s,g) \propto \pi_b (a|s,g) \exp(\beta \cdot A(s,a,g))$ brings improved expected value over $\pi_b$. However, when the behavior policy is multi-modal and the expert policy is deterministic, as often encountered in multi-goal RL, there is a risk of interpolating between modalities, leading to a widened expert estimation gap. A viable solution to this issue is to eliminate samples from inferior modalities, $\hat \pi_{E}(a|s,g)=\pi_{b}(a|s,g)  \exp (A(s,a,g))\cdot 1[A(s,a,g) \geq c]$, which is the Best Advantage Weight introduce by \cite{yang2022rethinking}. Ideally, we can eliminate all data from other modalities to obtain a minimum expert estimation gap, but the size of the training data decreases as $c$ grows. There is a trade-off of balancing data quality versus quantity when setting $c$. Note that our analysis considers an oracle advantage function, but in practice, an imprecise estimation of the advantage function can exacerbate the expert estimation gap. Therefore, an improved method for estimating the advantage function is also crucial.

\paragraph{Distribution Shift} The distribution shift term is hard to minimize without any information about the testing distribution $\mathcal{T}$. Instead, we consider minimizing the worst-case of this term by re-weighting the training distribution $\mathcal{S}$.

Define a family of possible testing distributions as $
    \mathcal{Z}:=\left\{Z\bigm|\int_x P_Z(x) =1; 0 \leq P_Z(x) \leq C, \forall x \in \mathcal{X}\right\}$.
    Here $C > 1/ |\mathcal{X}|$ is a universal positive constant.  
Our goal is to re-weight the training distribution $S$ that can minimize the worst-case distribution shift, i.e., $\sup_{Z \in \mathcal{Z}} d_1(Z, S)$.

 Let $\mathcal{S}$ denote the family of distributions that are generated by re-weighting $S$, i.e., $\mathcal{S} := \{S'|   P_{S'}(x) = h(x) P_S(x); h(x) > 0, \forall x \in \mathcal{X}; \int_x P_{S'}(x) = 1\}$.

Let $\bar S$ denote the uniform distribution, i.e., $
    P_{\bar S} (x) = 1/|\mathcal{X}|, \forall x \in \mathcal{X}, a.s.$.
We denote the subset of $\mathcal{S}$ that contains all ``non-uniform" distributions as  $\mathcal{S}^{-}$, i.e., 
\begin{align}
   & \mathcal{S}^{-} := \{S'|   P_{S'}(x) = h(x) P_S(x);  h(x) > 0, \forall x \in \mathcal{X}; \nonumber \\
    & \exists \mathcal{J} \subset \mathcal{X}, \int_{x \in \mathcal{J}} P_{S'}(x) dx <  |\mathcal{J}| / |\mathcal{X}|; 
 \nonumber \int_x P_{S'}(x) = 1\}
\end{align}

\begin{theorem} 
\label{tm:distribution_shift}
For all $\forall S \in \mathcal{S}^{-}$, we have
\begin{align*}
    \sup_{Z \in \mathcal{Z}} d_1 (Z, S) > \sup_{Z \in \mathcal{Z}} d_1 (Z, \bar S)
\end{align*}
\end{theorem}
The proof can be found in Appendix \ref{ap:proof_distribution}. Theorem \ref{tm:distribution_shift} suggests that we can re-weight the training distribution $\mathcal{S}$ to a uniform distribution to obtain a smaller worst-case distribution shift. To achieve this, we can approximate the reciprocal of density or uncertainty via the kernel density estimator \cite{zhao2019maximum,pitis2020maximum} or ensemble \cite{pathak2019self,bai2022pessimistic}.

\paragraph{The Last Term} Note that the last term in the above bound is dependent on the dataset size $m$. Therefore, increasing the size of the dataset through augmentation techniques can lead to a more tighter upper bound. This gives justification to use goal relabeling \cite{andrychowicz2017hindsight,li2020generalized} for offline GCRL. Relabeling goals with achieved goals expands the size of the offline dataset, which enables training agents on more diverse state-goal pairs, subsequently improving an agent's ability to achieve goals in unknown testing distributions.

\subsection{A Brief Summary}
In this section, we have discussed several useful techniques for OOD generalization from the generalization theory. These techniques include: (1) weighted imitation learning, which minimizes the empirical imitation loss; (2) advantage re-weighting and data selection, which narrow the expert estimation gap; (3) re-weighting with the reciprocal of density, which minimizes the worst-case distribution shift; and (4) goal relabeling, which minimizes the last term related to the dataset size. Based on our analysis, a weighted imitation learning framework that integrates all of these techniques is highly desirable for OOD goal generalization.

\begin{figure*}[ht]
    \centering
    \subfigure[]{\includegraphics[width=0.17\linewidth]{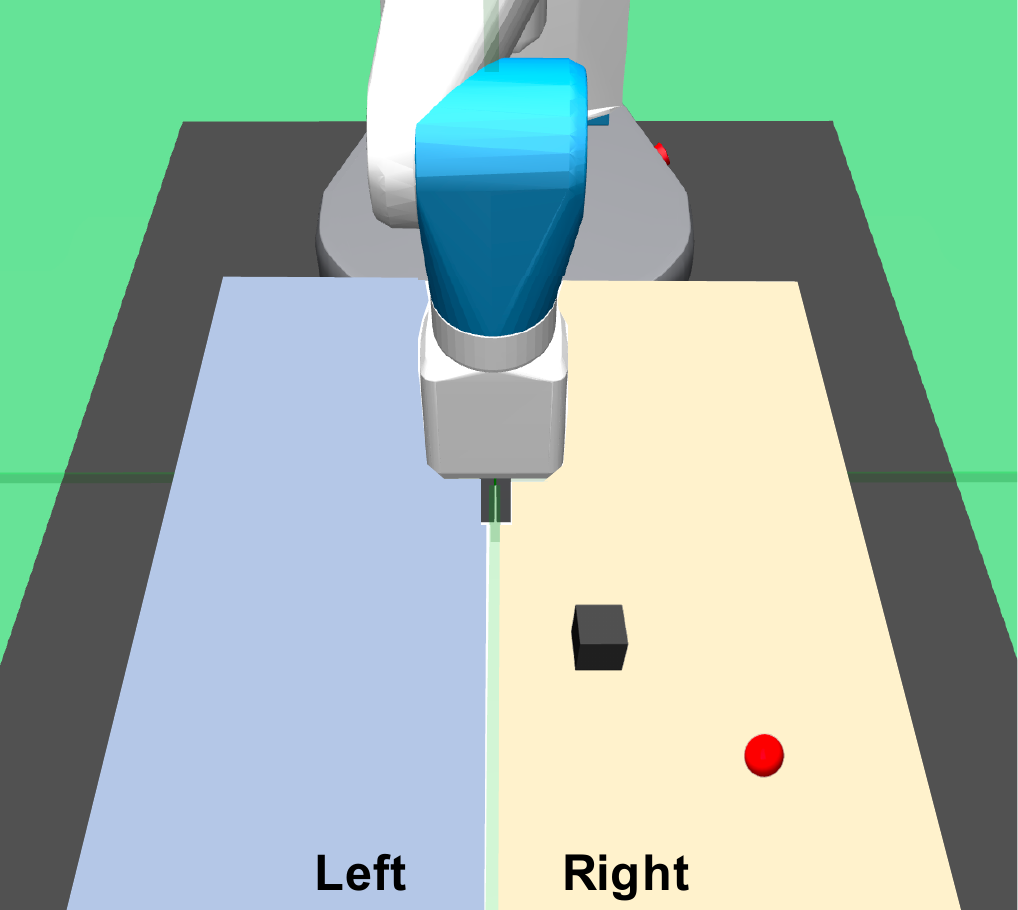}\label{fig:env_push}}
    \hspace{.14in}
    \subfigure[]{\includegraphics[width=0.17\linewidth]{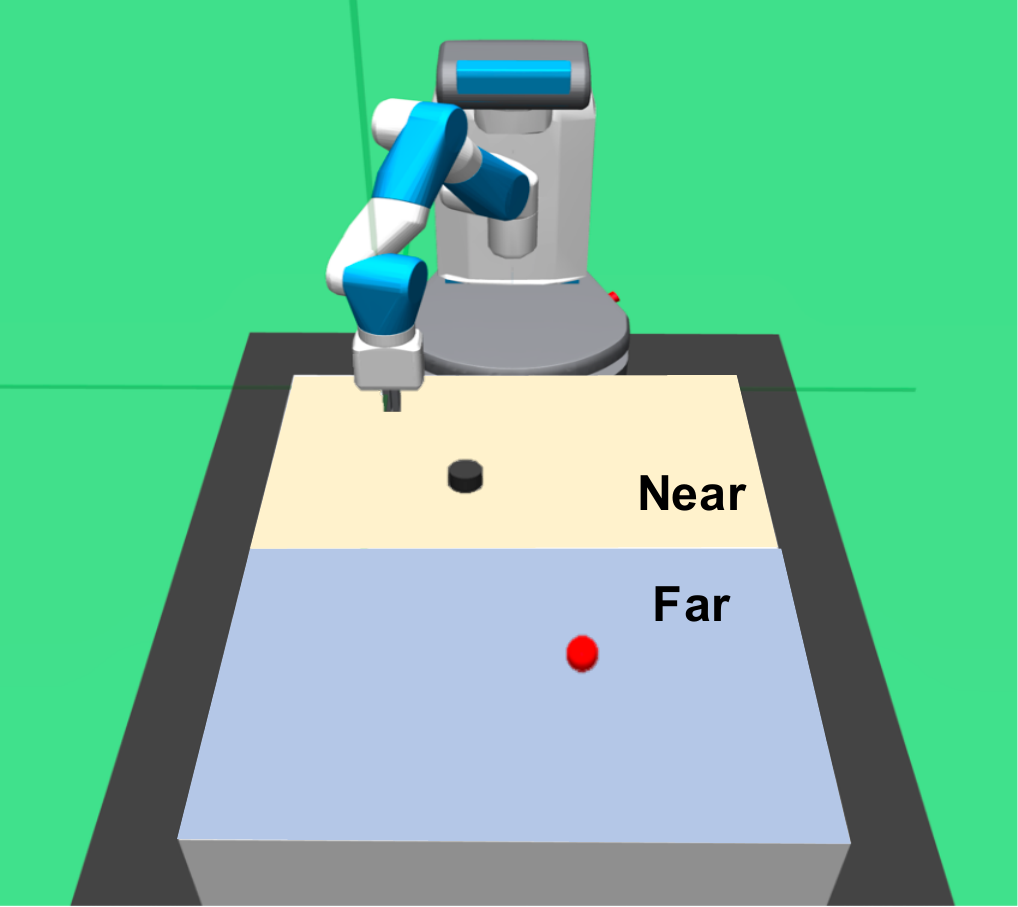}\label{fig:env_slide}}
    \hspace{.14in}
    \subfigure[]{\includegraphics[width=0.17\linewidth]{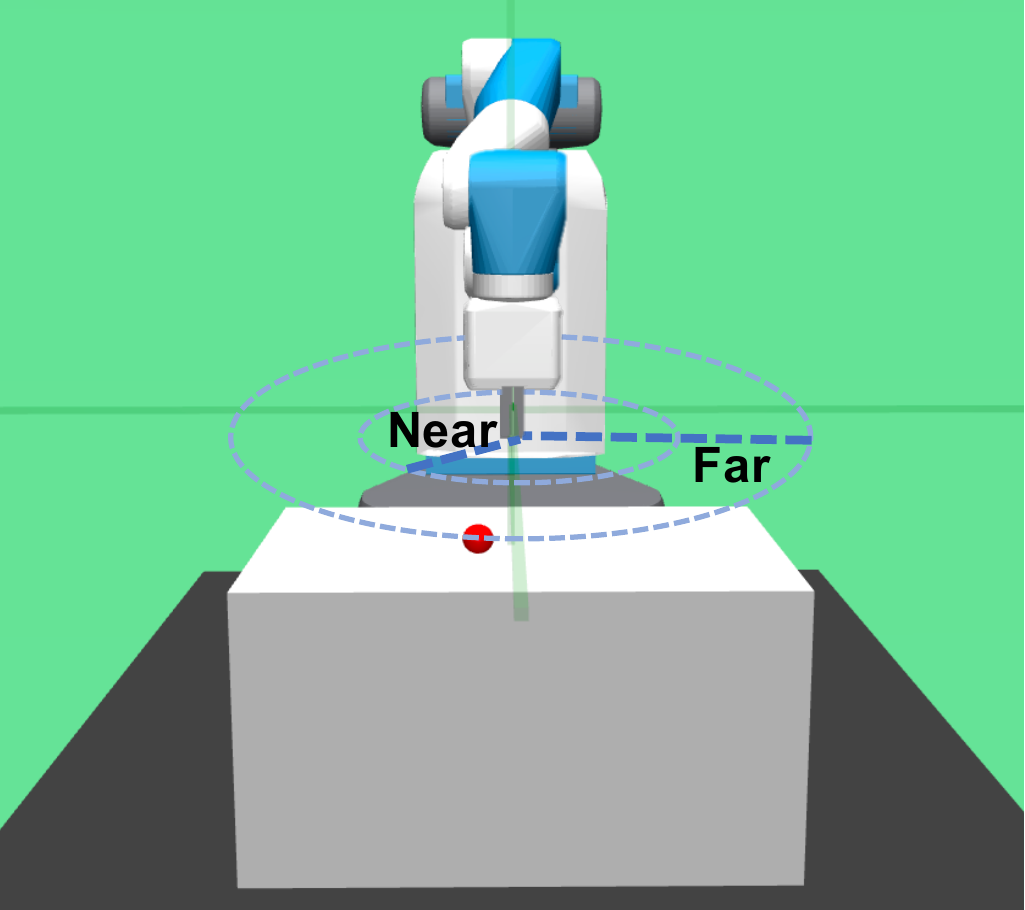}\label{fig:env_reach}}
    \hspace{.14in}
    \subfigure[]{\includegraphics[width=0.17\linewidth]{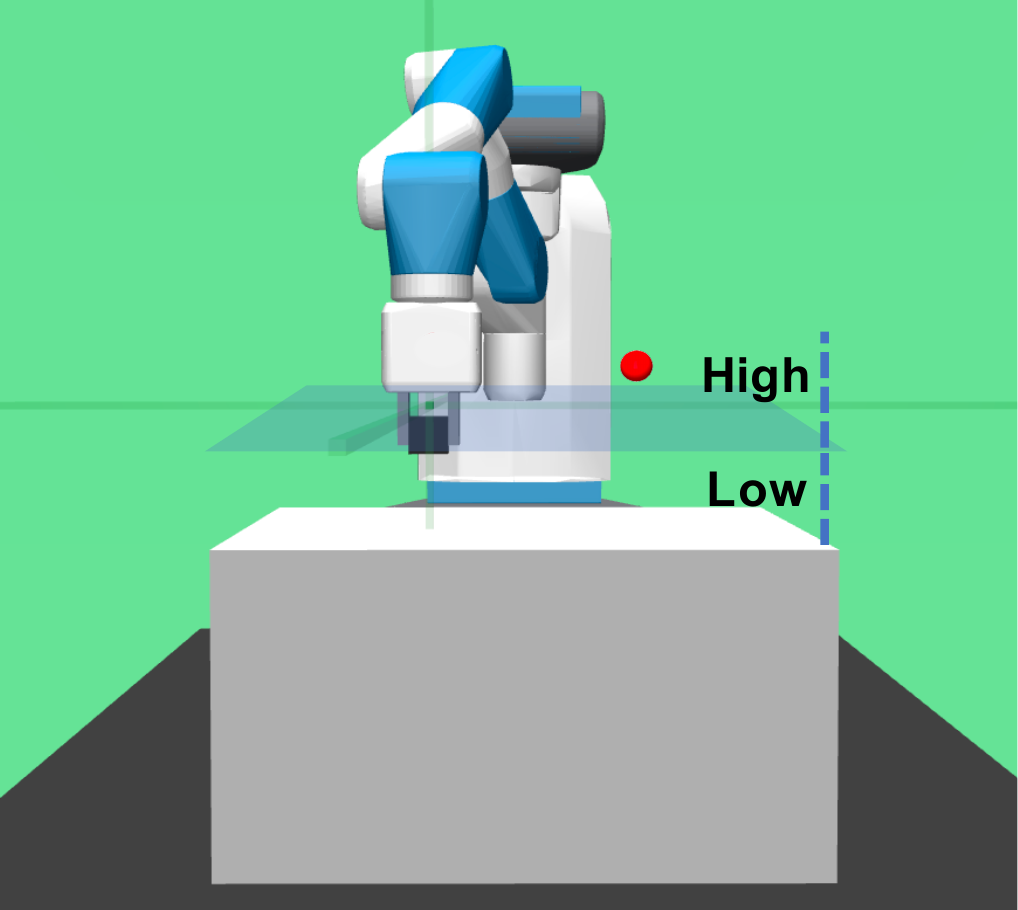}\label{fig:env_pick}}
    \caption{Examples of designed benchmark tasks. (a) Push Left-Right, (b) Slide Near-Far, (c) Reach Near-Far, and (d) Pick Low-High.}
    \label{fig:env_merge}
\end{figure*}

\section{Algorithm}
Motivated by our theoretical insights, we present the GOAT algorithm, which builds upon the weighted imitation learning framework of WGCSL \cite{yang2022rethinking}. The existing framework already incorporates several techniques beneficial for the generalization bound, including goal relabeling, advantage re-weighting, and data selection. To further minimize the generalization bound, GOAT improves the surrogate expert policy through better value function estimation and minimizes the worst-case distribution shift by re-weighting samples with uncertainty, where the uncertainty is introduced as an alternative to the reciprocal of density.

We denote a trajectory of horizon $T$ in the offline dataset as $D=\{(s_t,a_t,r_t, s_{t+1}, g)\}, t\in[1,T]$. As suggested by our theory, we perform hindsight relabeling \cite{andrychowicz2017hindsight} to augment the dataset and obtain the relabeled data $D_{relabel} = \{(s_t,a_t,r_t', s_{t+1}, g')\}, t\in[1,T]$, where $g'=\phi(s_i), r'_t=r(s_t,a_t,\phi(s_i)),i\geq t$. We then perform weighted imitation learning based on $D_{relabel}$.

\paragraph{Weighted Supervised Policy Learning} The overall weighted imitation learning framework is as follows:
\begin{equation}
	J(\pi_{\theta}) = 
	\mathbb{E}_{(s,a,g')\sim D_{relabel}} \big[
	w(s,a,g')
	\log\pi_{\theta}(a|s,g') \big],
    \label{eq:J_wgcsl}
\end{equation}
where the weight $w$ contains three parts, i.e., the uncertainty weight(UW), the exponential advantage weight (EAW) and the data selection weight (DSW). Formally, we define
\begin{equation*}
\begin{aligned}
    w(s,a,g') = u(s,g') \cdot  \exp(\beta A(s,a,g')) \cdot \epsilon(A(s, a, g')) ,
\end{aligned}
\label{eq:weights}
\end{equation*}
where $u$ is the uncertainty weight to replace the density, $A(s, a, g')$ is the advantage function, and $\epsilon(A(s, a, g'))= 1[A(s,a,g') \geq c]$ is the DSW. In DSW, the constant $c$ may be established as the $\alpha$ quantile of advantage values, in recognition of the fact that the best value for $\alpha$ is more consistently applicable across different environments. Additionally, we also discuss an adaptive variant of DSW and we refer the readers to Appendix \ref{ap:adap_DSW}. In the subsequent section, we mainly focus on how to estimate the advantage function and the uncertainty weight.

\paragraph{Ensemble Value Functions} To better estimate the advantage value for both EAW and DSW, we train $N$ randomly initialized value functions. Each of the value function $Q_i(s,a,g), 1 \leq i \leq N $ minimizes the TD loss:
\begin{equation}
\label{eq:td_loss}
\begin{aligned}
    \mathcal{L}_{TD} = &\mathbb{E}_{(s_t,a_t,r'_t,s_{t+1},g')\sim D_{relabel}} [L_2(r'_t +  \\
    &\gamma \hat Q_i(s_{t+1},\pi_{\theta}(s_{t+1},g'),g') - Q_i(s_t,a_t,g'))].
\end{aligned}
\end{equation}
In Eq \eqref{eq:td_loss}, $L_2(u)=u^2$ and $\hat Q_i$ refers to the target network of $Q_i$. Although $\pi_{\theta}$ is regularized to be near the dataset policy, it can still produce OOD actions to affect the value estimation during training. To mitigate this problem, we can replace $L_2$ with the expectile regression (ER): $L_2^\tau(u)=|\tau - 1(u<0)|u^2$, where $\tau \in(0,1)$. 


The group of value function is then leveraged to estimate the advantage value and the uncertainty weight. Specifically, we utilize the mean of the $Q$ functions to estimate $V(s,g')$:
\begin{equation*}
\label{eq:estimate_V}
    V(s,g')=\frac{1}{N} \sum_{i=1}^N Q_i(s,\pi_{\theta}(s,g'),g')
\end{equation*}
Then, the advantage value can be estimated by $A(s_t,a_t,g') = r(s_t,a_t,g') + \gamma V(s_{t+1},g') - V(s_t, g')$.

\paragraph{Uncertainty Estimation}
Estimating the density of high-dimensional state-goal space is generally challenging. In this work, we utilize uncertainty to replace density as a fact that the bootstrapped uncertainty is approximately proportional to the reciprocal of density in tabular MDP \cite{bai2022pessimistic}. The uncertainty is calculated as the standard deviation of value functions:
\begin{equation*}
    \label{eq:uncertainty_std}
    \Std(s,g')=\sqrt{\frac{\sum\nolimits_{i=1}^N \big(Q_i( s, \pi_{\theta}(s,g'),g')- V(s,g'))^2}{N}}
\end{equation*}

However, the range of $\Std(s,g')$ varies for different environments. To make the uncertainty weight stable, we normalize the standard deviation to $[0,1]$: $$\Std_{norm}(s,g')=\frac{\Std(s,g') - \Std_{min}}{\Std_{max}-\Std_{min}},$$ 
where $\Std_{max}, \Std_{min}$ are the maximum and minimum values of $\Std(s,g')$ stored in a First In First Out (FIFO) queue. Finally, we transform $\Std_{norm}(s,g')$ to reduce more weight for data with lower variance and define the uncertainty weight $u(s,g')$ as:
\begin{equation}
    \label{eq:final_UW}
    u(s,g')= \text{clip}(\text{tanh}(\Std_{norm}(s,g') \times w)  + w_{min}, 0, 1)
\end{equation}
where $w_{min}$ is set to $0.5$. Intuitively, $w$ is the hyperparameter to adjust the proportion of ranked samples to downweight, i.e., the smaller $w$ is, the more data will be downweighted, and vice versa.

\begin{table*}[h]
\caption{Average success rates ($\%$) with standard deviation over 5 random seeds. Blue lines and purple lines refer to IID and OOD tasks, respectively. Top two success rates for each task are highlighted.}
\label{tab:success_rate_small}
\begin{center}
\begin{small}
 \begin{adjustbox}{max width=0.95\linewidth}
\begin{tabular}{llccccccccccc}
\toprule
Task Group & Task &  GOAT($\tau$)  & GOAT & WGCSL &  GCSL & BC & GoFAR & DDPG+HER & CQL+HER & MSG+HER  \\
\midrule
\rowcolor{c1!10} \cellcolor{white}  &  Right  & 100.0$\pm$0.0  &  100.0$\pm$0.0  &  100.0$\pm$0.0  &  93.6$\pm$4.3  &  92.0$\pm$3.0  &   100.0$\pm$0.0  &   99.6$\pm$0.6  &  100.0$\pm$0.0  &  99.4$\pm$0.6  \\
\rowcolor{blue!10} \cellcolor{white} Reach Left-Right &  Left &  99.9$\pm$0.2  & 99.0$\pm$2.0  &  97.8$\pm$4.4  &  36.3$\pm$10.9  &  30.4$\pm$15.2  &  54.2$\pm$9.3   &    73.8$\pm$27.6  &  94.5$\pm$6.3 &  85.6$\pm$15.7 \\
  &  Average  & \textbf{99.9} & \textbf{99.5}  &  98.9  &  65.0  &  61.2  &  77.1   &  86.7  &  97.2 &  92.5 \\

\midrule
\rowcolor{c1!10} \cellcolor{white}  &  Near & 100.0$\pm$0.0    &  100.0$\pm$0.0  &  100.0$\pm$0.0  &  79.7$\pm$3.0  &  85.3$\pm$4.3  &  100.0$\pm$0.0  &  95.9$\pm$2.0  &  100.0$\pm$0.0 &  98.6$\pm$2.8  \\
\rowcolor{blue!10} \cellcolor{white} Reach Near-Far  & Far  & 90.9$\pm$1.5   & 97.6$\pm$1.1  &  89.0$\pm$2.1  &  33.5$\pm$5.5  &  37.9$\pm$9.7  &  85.0$\pm$1.9  &    66.8$\pm$6.9  &  88.0$\pm$2.1 & 77.8$\pm$9.7  \\
  &  Average  &  \textbf{95.4}  & \textbf{98.8}  &  94.5  &  56.6  &  61.6  &  92.5    &  81.4  &  94.0 &  88.2  \\

\midrule
\rowcolor{c1!10} \cellcolor{white}  & Right2Right &  96.2$\pm$1.2   &  95.9$\pm$1.2  &  93.2$\pm$0.9  &  82.1$\pm$3.7  &  78.9$\pm$3.8  &   95.9$\pm$1.4 &    60.1$\pm$6.0  &  83.3$\pm$2.7  &  92.8$\pm$0.9 \\
\rowcolor{blue!10} \cellcolor{white} &  Right2Left  & 75.6$\pm$3.6 & 69.3$\pm$6.6  &  63.3$\pm$8.9  &  40.1$\pm$6.0  &  25.6$\pm$2.7  &  43.8$\pm$4.7  & 28.5$\pm$4.3  &  46.2$\pm$7.1 &  52.9$\pm$6.5  \\
\rowcolor{blue!10} \cellcolor{white} Push Left-Right &  Left2Right & 78.8$\pm$6.8   &  76.0$\pm$7.4  &  67.6$\pm$7.1  &  38.8$\pm$6.8  &  33.5$\pm$8.1  &  59.7$\pm$4.3   &   20.6$\pm$11.5  &  40.4$\pm$12.1 &  59.3$\pm$7.7  \\
\rowcolor{blue!10} \cellcolor{white} &  Left2Left   & 75.6$\pm$12.1 &  61.1$\pm$7.6  &  47.7$\pm$7.4  &  35.4$\pm$6.6  &  20.9$\pm$3.2  &  32.5$\pm$5.8  &   27.0$\pm$3.8  &  34.9$\pm$5.9  &  38.8$\pm$7.9 \\
& Average  &   \textbf{81.5} &  \textbf{75.6}  &  68.0  &  49.1  &  39.7  &  58.0 &  34.1  &  51.2 & 61.0  \\

\midrule
\rowcolor{c1!10} \cellcolor{white} & Near2Near  &  97.2$\pm$0.7 &  92.0$\pm$2.6  &   93.5$\pm$1.0  &  77.6$\pm$4.7  &  67.5$\pm$3.6  & 92.6$\pm$2.2 &   39.3$\pm$22.4  &  77.7$\pm$3.9 &  84.7$\pm$6.1 \\
\rowcolor{blue!10} \cellcolor{white} & Near2Far & 78.4$\pm$3.5  &  70.3$\pm$5.7   &  67.0$\pm$5.4  &  43.1$\pm$7.2  &  24.9$\pm$5.9  & 60.9$\pm$3.8  &   30.5$\pm$12.1  &  60.0$\pm$6.2  &  58.4$\pm$2.1 \\
\rowcolor{blue!10} \cellcolor{white} Push Near-Far & Far2Near  &  70.5$\pm$2.4 & 69.5$\pm$3.6  &  68.0$\pm$2.4  &  47.4$\pm$3.5  &  40.2$\pm$7.5  & 65.0$\pm$4.8 &   25.0$\pm$12.8  &  61.1$\pm$4.3  & 56.5$\pm$6.0  \\
\rowcolor{blue!10} \cellcolor{white} & Far2Far  & 55.1$\pm$2.4 &  50.8$\pm$1.8   &  51.1$\pm$4.7  &  27.9$\pm$4.1  &  15.3$\pm$2.7  & 41.3$\pm$3.1  & 18.0$\pm$7.0  &  47.1$\pm$2.4 & 41.7$\pm$5.4 \\
& Average  & \textbf{75.3} & \textbf{70.6}  &  69.9  &  49.0  &  37.0  &   65.0  &  28.2  &  61.5 &  60.3 \\

\midrule
\rowcolor{c1!10} \cellcolor{white} & Right2Right  & 96.5$\pm$1.1   &  97.3$\pm$1.2  &  93.8$\pm$5.3  &  53.4$\pm$14.1  &  52.9$\pm$7.5  &   56.9$\pm$4.3  &    40.4$\pm$13.1  &  91.9$\pm$6.8 &  94.9$\pm$2.2 \\
\rowcolor{blue!10} \cellcolor{white} & Right2Left  &  87.9$\pm$5.1   &  88.6$\pm$1.1  &  89.4$\pm$3.9  &  20.7$\pm$6.9  &  5.6$\pm$2.1  &  9.3$\pm$1.8  &  52.7$\pm$14.9  &  82.4$\pm$12.6  & 89.3$\pm$6.8  \\
\rowcolor{blue!10} \cellcolor{white} Pick Left-Right & Left2Right  & 91.4$\pm$2.3    &  93.9$\pm$1.9  &  90.0$\pm$4.1  &  47.0$\pm$10.9  &  37.2$\pm$6.4  &  51.1$\pm$6.5  &   9.8$\pm$5.7  &  86.4$\pm$8.6 &  60.8$\pm$16.5 \\
\rowcolor{blue!10} \cellcolor{white} & Left2Left & 87.6$\pm$5.7   &  88.3$\pm$3.7  &  87.0$\pm$5.1  &  24.7$\pm$7.8  &  3.3$\pm$1.4  &  6.0$\pm$2.0   & 26.4$\pm$10.9  &  83.5$\pm$9.1 & 66.9$\pm$7.0 \\
& Average  & \textbf{90.8}   &  \textbf{92.0}  &  90.0  &  36.4  &  24.8  &  30.8  &   32.3  &  86.1 & 78.0  \\

\midrule
\rowcolor{c1!10} \cellcolor{white}  & Low  & 99.3$\pm$0.5     &  99.8$\pm$0.2  &  98.6$\pm$1.3  &  84.4$\pm$3.6  &  72.4$\pm$5.4  &  95.2$\pm$1.6  &  50.4$\pm$23.9  &  100.0$\pm$0.0 &  97.3$\pm$2.2 \\
\rowcolor{blue!10} \cellcolor{white}  Pick Low-High & High  & 78.3$\pm$6.3     &  71.9$\pm$6.4  &  66.6$\pm$6.6  &  28.4$\pm$6.9  &  3.0$\pm$1.6  &  7.6$\pm$3.1   & 17.0$\pm$10.2  &  44.6$\pm$9.2 & 23.3$\pm$7.8 \\
& Average  & \textbf{88.8} & \textbf{85.8}  &  82.6  &  56.4  &  37.7  & 51.4  &    33.7  &  72.3 & 60.3  \\

\midrule
\rowcolor{c1!10} \cellcolor{white}  & Right2Right & 82.0$\pm$3.2  &  79.0$\pm$5.8  &  70.8$\pm$13.5  &  62.2$\pm$7.0  &  60.3$\pm$4.7  &   62.6$\pm$8.7  &   4.7$\pm$1.5  &  20.3$\pm$2.5  & 20.8$\pm$5.0 \\
\rowcolor{blue!10} \cellcolor{white}  & Right2Left &  45.1$\pm$8.8     &  41.3$\pm$7.1  &  36.2$\pm$8.6  &  11.5$\pm$2.0  &  15.7$\pm$6.0  &  31.6$\pm$3.9   & 0.3$\pm$0.4  &  8.6$\pm$3.0 & 7.3$\pm$4.9 \\
\rowcolor{blue!10} \cellcolor{white}  Slide Left-Right & Left2Right &   79.6$\pm$2.7    &  59.0$\pm$7.6  &  50.7$\pm$12.7  &  29.1$\pm$4.8  &  41.8$\pm$7.2  &  51.0$\pm$10.5  &   0.2$\pm$0.2  &  1.7$\pm$0.7 & 3.6$\pm$4.3 \\
\rowcolor{blue!10} \cellcolor{white}  & Left2Left  & 52.5$\pm$8.3    &  50.1$\pm$9.5  &  35.3$\pm$11.3  &  25.5$\pm$5.4  &  33.7$\pm$10.6  &  28.2$\pm$2.6   &   2.1$\pm$1.1  &  4.3$\pm$2.5  &  7.1$\pm$3.3 \\
& Average  &  \textbf{64.8}  & \textbf{57.4}  &  48.3  &  32.1  &  37.9  &  43.4  & 1.8  &  8.7 &  9.7 \\

\midrule
\rowcolor{c1!10} \cellcolor{white}  & Near   &  77.4$\pm$4.5 & 76.9$\pm$3.3  &  73.1$\pm$5.8  &  28.0$\pm$7.1  &  26.6$\pm$8.3  &  69.3$\pm$2.8  &   11.3$\pm$4.5  &  43.5$\pm$3.3 &  28.3$\pm$9.5  \\
\rowcolor{blue!10} \cellcolor{white}  Slide Near-Far & Far     &  25.1$\pm$3.9  & 29.0$\pm$4.5  &  17.4$\pm$3.2  &  0.0$\pm$0.0  &  0.0$\pm$0.0  & 24.1$\pm$2.9   & 4.4$\pm$3.7  &  7.4$\pm$3.8 & 2.6$\pm$1.4  \\
& Average  &  \textbf{51.2} & \textbf{53.0}  &  45.2  &  14.0  &  13.3  &  46.7   & 7.8  &  25.5 &  15.4 \\

\midrule
\rowcolor{c1!10} \cellcolor{white}  & Near   & 72.6$\pm$5.3&  71.9$\pm$3.2  &  70.0$\pm$3.6  &  0.0$\pm$0.0  &  0.0$\pm$0.0  &  77.4$\pm$1.7  &   0.0$\pm$0.0  &  1.8$\pm$3.6  & 0.0$\pm$0.0  \\
\rowcolor{blue!10} \cellcolor{white} HandReach Near-Far & Far &  33.1$\pm$4.5&  38.4$\pm$4.1  &  31.8$\pm$3.8  &  0.1$\pm$0.2  &  0.0$\pm$0.0  &  36.9$\pm$3.1   & 0.0$\pm$0.0  &  0.0$\pm$0.0 & 0.0$\pm$0.0  \\
& Average  &  52.8 &  \textbf{55.2}  &  50.9  &  0.0  &  0.0  &    \textbf{57.1}   &  0.0  &  0.9 & 0.0  \\

\midrule
\multirow{2}{*}{Average} &  IID Tasks & \textbf{91.2} & \textbf{90.3} &  88.1 & 62.3 & 59.5  & 83.3 &  44.6 & 68.7 & 68.5 \\
  & OOD Tasks  & \textbf{70.9} & \textbf{67.9}  & 62.1  &  28.8  & 21.7  & 40.5  &   23.7 & 46.5 & 43.1 \\
\bottomrule
\end{tabular}
\end{adjustbox}
\end{small}
\end{center}
\end{table*}

\section{Experiments}
In this section, we introduce a new benchmark consisting of 9 task groups and 26 tasks to evaluate the OOD generalization performance of offline GCRL algorithms.

\subsection{Environments and Experimental Setup}
\paragraph{Environments} The introduced benchmark is modified from MuJoCo robotic manipulation environments \cite{plappert2018multi}. Agents aim to move a box, a robot arm, or a bionic hand to reach desired positions. The reward for each environment is sparse and binary, i.e., 1 for reaching the desired goal and 0 otherwise. As listed in Table \ref{tab:success_rate_small}, there are 9 task groups with a total of 26 tasks, 17 of which are OOD tasks whose goals are not in the training data. For example, as shown in Figure \ref{fig:env_push}, the dataset of Push Left-Right contains trajectories where both the initial object and achieved goals are on the right side of the table. Then the IID task is evaluating agents with object and goals on the right side (i.e., Right2Right). The OOD tasks can be generated by changing the side of the initial object or desired goals. Following \cite{yang2022rethinking}, we collect datasets with the online DDPG+HER agent. More information about the task design and offline datasets can be found in Appendix \ref{ap:data_implementation}.

\paragraph{Experimental Setup} We compare GOAT with current SOTA offline GCRL methods, including WGCSL \cite{yang2022rethinking}, GoFAR \cite{ma2022far}, CQL+HER \cite{chebotar2021actionable}, GCSL \cite{ghosh2019learning}, and DDPG+HER \cite{andrychowicz2017hindsight}. Besides, we also include a SOTA ensemble-based offline RL methods, MSG \cite{ghasemipour2022so}, namely ``MSG+HER". 
To evaluate performance, we assess agents across 200 randomly generated goals for each task and benchmark their average success rates. More details and additional experiments are provided in Appendix \ref{ap:data_implementation} and Appendix \ref{ap:addtional_exp}.

\begin{figure}[htb]
    \centering
    \subfigure[]{\includegraphics[width=0.485\linewidth]{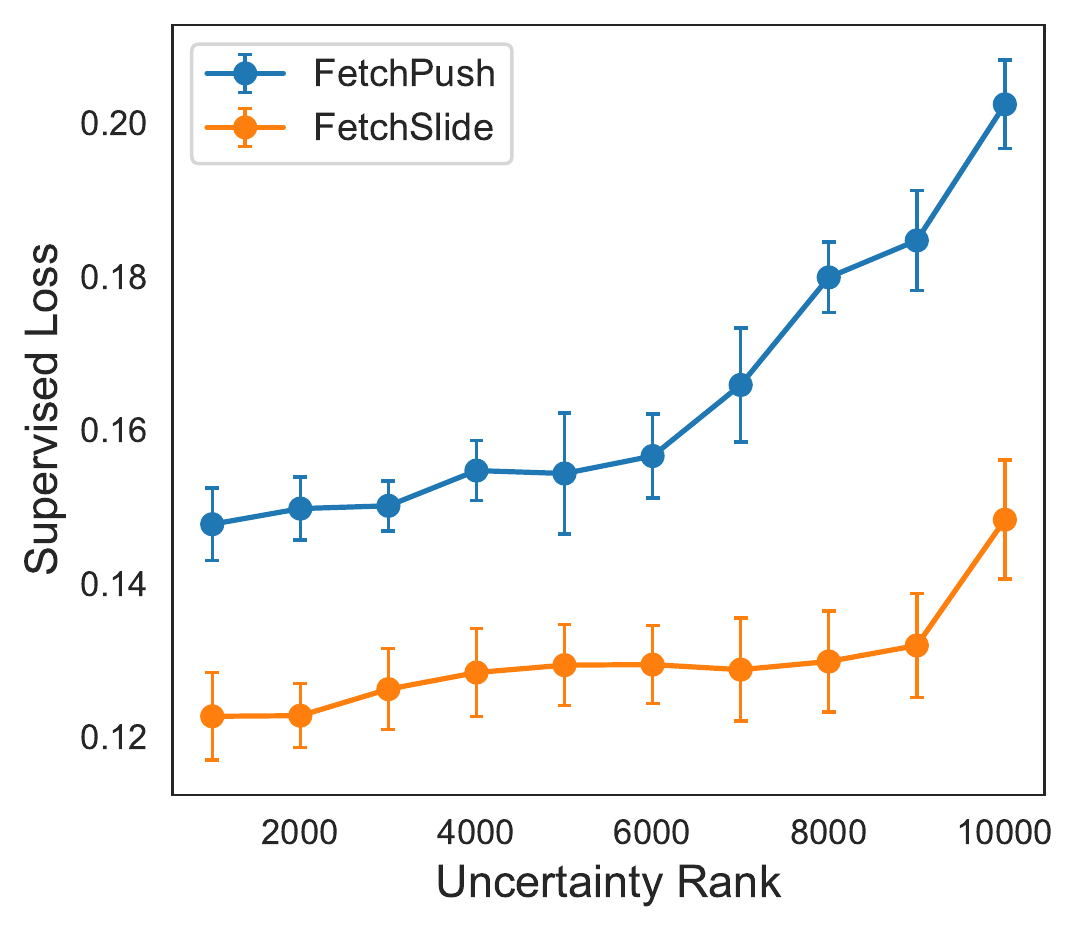}\label{fig:loss_vars}}
    \subfigure[]{\includegraphics[width=0.485\linewidth]{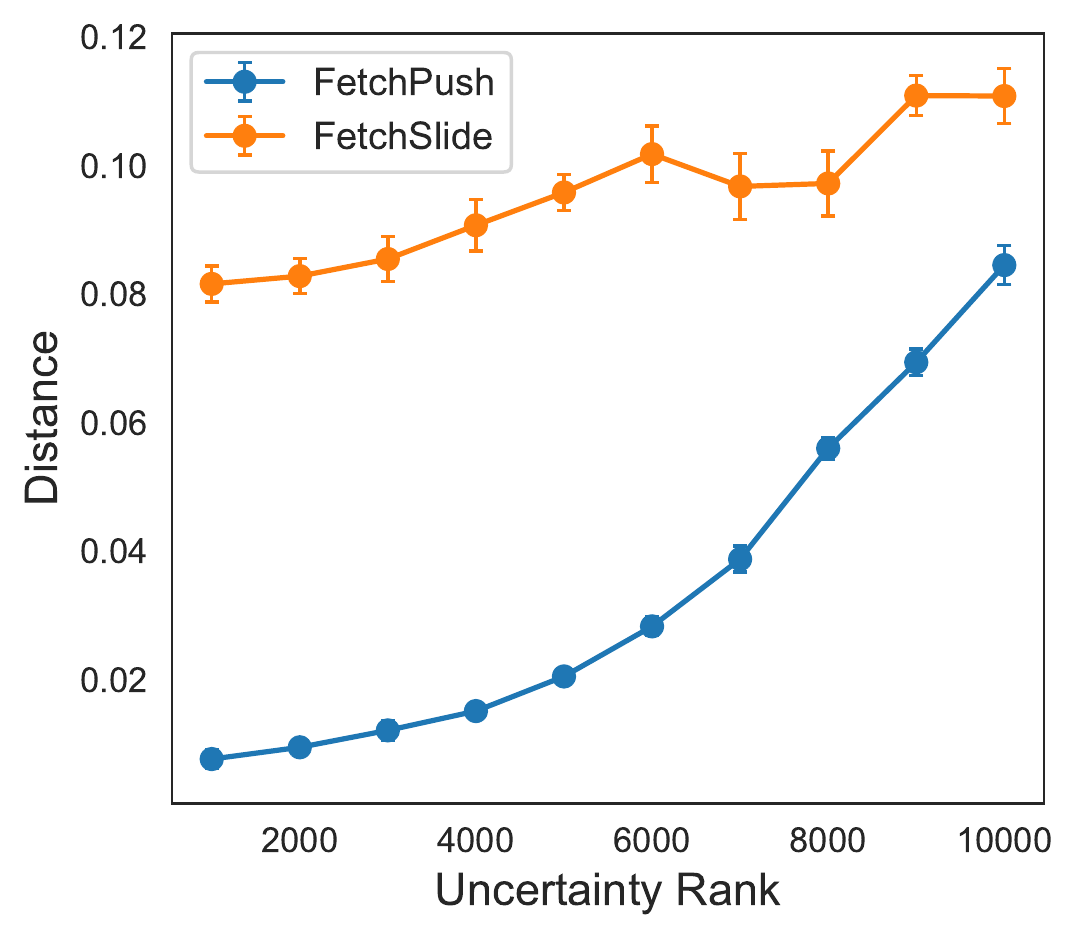}\label{fig:distance_vars}}
    \caption{Correlation between (a) supervised loss, (b) state-goal distance and the uncertainty rank.}
    \label{fig:vars_correlation}
    \vspace{-0.6cm}
\end{figure}

\subsection{Understanding the Uncertainty Weight}
\label{sec:understand_UW}
In our theoretical analysis, the uncertainty weight (UW) has the effect of reducing the worst-case distance between the training and unknown testing distributions. To make it more clear, we collect 10000 relabeled samples $(s,a,g')$ and rank these samples according to the UW in Eq \eqref{eq:final_UW}. For a sample $(s,a,g')$, we record two values, the supervised loss (i.e., $\|a - \pi_{\theta}(s,g')\|_2^2$), and the distance between the desired goal and the achieved goal (i.e., $\|(g'-\phi(s))\|_2^2$, short for ``state-goal distance"). Then, we average their values for every 1000 ranked samples. The results are shown in Figure \ref{fig:vars_correlation}. Interestingly, UW assigns more weights to samples with larger supervised loss, which may also be related to Distributionally Robust Optimization \cite{rahimian2019distributionally,goh2010distributionally}, thereby improving performance on worst-case scenarios. Moreover, UW prefers samples with larger state-goal distance. Since every state-goal pair $(s,g')$ defines a task from $s$ to $g'$, UW enhances harder tasks with larger state-goal distance. In general, OOD goals are relatively further away than IID goals, which also interprets why UW works for OOD generalization.

\begin{figure}[h!]
    \centering
    \includegraphics[width=1\linewidth]{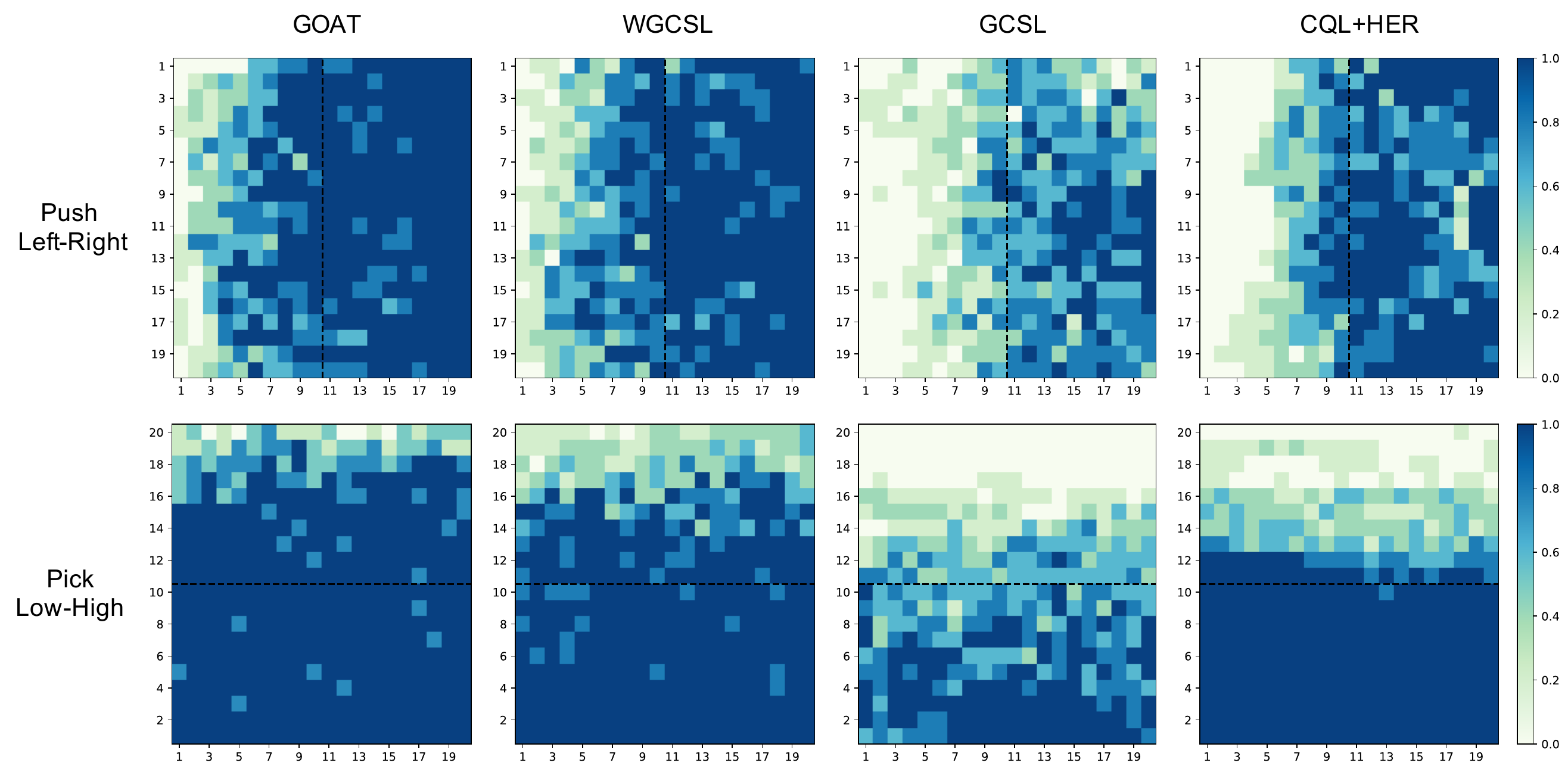}
    \caption{The coverage of successful goals. The darkness of color represents the success rate of each goal for 5 random seeds. The black dotted line is the dividing line between IID and OOD goals. The IID areas are the right half (top row) and the lower half (bottom row) rectangles for the two tasks.}
    \label{fig:success_coverage}
\end{figure}

\subsection{Generalizing to OOD Goals}
\label{sec:benchmark_results}

Table \ref{tab:success_rate_small} reports the average success rates of GOAT and other baselines on the introduced benchmark. We denote GOAT with expectile regression as GOAT($\tau$), where $\tau<0.5$. From the results, we can conclude that OOD generalization is more challenging than IID tasks. For example, the performance of GoFAR, GCSL, and BC drops by more than half on OOD tasks. On the contrary, GOAT and GOAT($\tau$) achieve the highest OOD success rates over 16 out of 17 tasks. Compared with WGCSL, GOAT improves the IID performance slightly but considerably enhances the OOD performance. 

While CQL+HER and MSG+HER exhibit better performance than GCSL and BC, they are worse than weighted imitation learning methods WGCSL and GOAT, possibly due to pessimism restraining generalization. Besides, they fail on hard tasks such as Slide and HandReach. Another observation is that although GOAT, WGCSL, GoFAR are all weighted imitation learning methods, their OOD performance varies significantly, indicating components of weighted imitation learning also matter. To better understand these components, we will present an in-depth ablation analysis in Section \ref{sec:ablations}.

In Figure \ref{fig:success_coverage}, we visualize the coverage of successful goals in Push Left-Right and Pick Low-High tasks, given fixed initial states at the right center and bottom center, respectively. Each small square represents a goal in the goal space, and their darkness represents the average success rate for 5 random seeds. The results demonstrate that GOAT has the largest coverage of successful goals among the baselines, including the strong baseline WGCSL. Notably, both CQL+HER and GCSL exhibit limitations in their capacity to generalize to unseen goals. Specifically, CQL+HER is restricted to the training distribution, whereas GCSL displays inadequate coverage for even IID goals due to overfitting to noise. The observed results are also in alignment with our didactic example in Section \ref{sec:didactic_example}.

\begin{table}[t]
    \centering
    \caption{Ablations of each component of GOAT.}
     \begin{adjustbox}{max width=\linewidth}
    \begin{tabular}{l|ccccccc}
    \toprule
       Success Rate ($\%$)  &  BC  & +HER & +EAW & +DSW & +Ens & + UW & + ER \\
       \midrule
        OOD Tasks &  21.7 & 28.8 & 53.1 & 62.1 & 63.4 & 67.9 & 70.9  \\
       Increment & $+ 0$ & $+ 7.1$ & $+ 24.3$ &  $+ 9.0$ & $+ 1.3$ & $+ 4.5$ & $+ 3.0$ \\
       \midrule
       All Tasks &  34.8 & 50.7 & 65.4 & 71.1 & 72.2 & 75.7 & 77.9  \\
        Increment & $+ 0$ & $+ 15.9$ & $+ 14.7$ &  $+ 5.7$ & $+ 1.1$ & $+ 3.5$ & $+ 2.2$ \\
       \bottomrule
    \end{tabular}
    \end{adjustbox}
    \label{tab:ablations}
\end{table}

\subsection{Ablations}
\label{sec:ablations}
To measure the contribution of each component of GOAT, we gradually add one component from BC to GOAT and record the performance increment caused by each component. As shown in Table \ref{tab:ablations}, the recorded results are average success rates of 17 OOD tasks and all 26 tasks. On average, each component brings improvement for OOD generalization of offline GCRL. For OOD tasks, EAW and DSW contribute the most by improving the surrogate expert policy for imitating. Besides, HER and UW also bring considerable improvement through data augmentation and uncertainty re-weighting. In addition, ensemble technique (Ens) improves the estimation of value functions but has the least effect on the overall performance. Expectile regression (ER) improves the average performance, but slightly reduces OOD performance on hard tasks such as Slide Near-Far and HandReach as shown in Table \ref{tab:success_rate_small}. Furthermore, we also compare variants of GOAT with V functions and $\chi^2$-divergence in Appendix \ref{ap:additional_ablations}.

\begin{figure}[ht]
    \centering
    \includegraphics[width=1\linewidth]{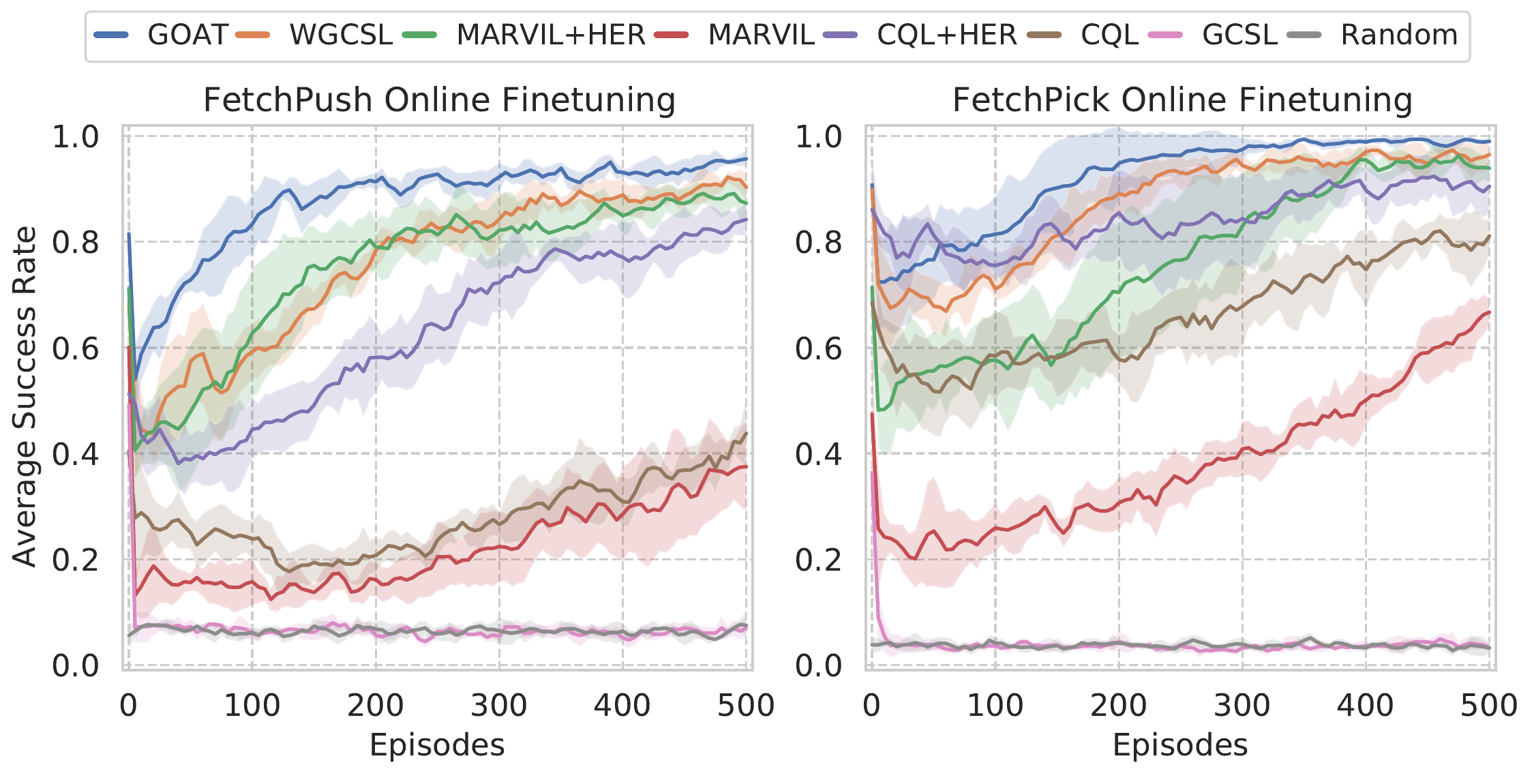}
    \vspace{-15pt}
    \caption{Online fine-tuning using DDPG+HER for different pre-trained agents on FetchPush and FetchPick tasks.}
    \label{fig:online_finetune}
\end{figure}

\subsection{Online Fine-tuning to Unseen Goals}
 We design an experiment to fine-tune pre-trained agents with online samples to verify whether the generalization ability of pre-trained agents is beneficial for online learning. The pre-trained agents are trained on offline datasets with partial coverage (Right2Right) and fine-tuned to full coverage (Right2Right, Right2Left, Left2Right, Left2Left). We apply DDPG+HER to fine-tune the policies and value functions after each episode collection. Additional Gaussian noise and random actions are applied for exploration. More detailed description can be found in Appendix \ref{ap:online_finetune}.

The experimental results are show in Figure \ref{fig:online_finetune}, which demonstrate that (1) most pre-trained agents learn faster than the randomly initialized agent (namely ``random") and (2) different initializations for goal-conditioned agents perform significantly different during fine-tuning. Specifically, GOAT outperforms other methods on the efficiency of online fine-tuning, while CQL, MARVIL \cite{wang2018exponentially} and GCSL result in slow-growing curves. We observe that the performance of GCSL initialization is similar to that of random initialization. It is likely that value networks contain valuable information for DDPG+HER agents to transfer from offline to online. This also explains why GOAT brings improvement, as it enhances value function learning via ensemble and expectile regression.

\section{Related Work}
\paragraph{Goal-conditioned RL} GCRL is a branch of reinforcement learning where agents need to achieve multiple goals sharing the same environmental dynamics~\cite{schaul2015universal,andrychowicz2017hindsight}. Goal relabeling~\citep{andrychowicz2017hindsight,li2020generalized,eysenbach2020rewriting,yang2021mher} is an effective technique that handles the sparse reward problem in GCRL and augments the data for policy learning. To improve the generalization ability, several prior works mainly focus on learning generalizable representations, e.g., combining Successor Feature with UVFA~\cite{ma2018universal,borsa2018universal}, decomposing $Q$ value via Bilinear Value Networks \cite{hong2022bi}, and learning discretization bottleneck representation for goals~\cite{islamdiscrete}. \citet{han2021learning} propose to learn invariant representation via aligned sampling to tackle the spurious feature problem. Our work differs from previous works in that we consider the offline GCRL setting, where pessimism can inhibit OOD generalization.

\paragraph{Offline RL and Offline GCRL} Offline RL handles the distribution shift challenge and learns policies from static datasets \cite{levine2020offline}. Generally, offline RL methods can be divided into two main directions, i.e., policy regularization and value underestimation. The first direction includes methods that constrain the learned policy to be close to the behavior policy under certain distance measure \cite{wang2018exponentially,fujimoto2019off,nair2020awac,yang2021believe,fujimoto2021minimalist}. Another direction is to underestimate values for OOD actions \cite{kumar2020conservative,yu2021combo,an2021uncertainty,bai2022pessimistic,yangrorl2022,ghasemipour2022so}. As for offline GCRL, current methods can also be grouped into policy regularization \cite{yang2022rethinking,ma2022far} and value underestimation \cite{chebotar2021actionable} methods. 
Different from prior works, our work focuses on learning policies from offline data and improving the ability to generalize to out-of-distribution goals.

\paragraph{Domain Generalization (DG)} DG aims to learn a model from training domains that can generalize on unseen testing domains \cite{zhou2021domain,wang2022generalizing}. Solutions to DG include data augmentation \cite{zhou2020learning,zhou2021domain2}, meta learning \cite{li2018learning,balaji2018metareg}, representation learning \cite{li2018deep} and distributionally robust optimization \cite{sagawa2019distributionally}. In reinforcement learning, DG is handled with data augmentation \cite{wang2020improving}, environment generation \cite{jiang2021prioritized}, and representation learning \cite{mazoure2021improving,sonar2021invariant,han2021learning}. Unlike these works, we mainly consider the covariate shift and handle pessimism and generalization simultaneously for OOD generalization of offline GCRL.

\section{Conclusion}
Learning from purely offline datasets and generalizing to unseen goals is one of the pursuits of the RL community. In this paper, we investigate the problem of out-of-distribution (OOD) generalization of offline GCRL. Through theoretical analysis and empirical evaluation, we demonstrate that (1) the choice of offline RL methods, particularly weighted imitation learning, and (2) the techniques to minimize
the generalization bound, are crucial for this problem. With these insights, we propose GOAT, a new weighted imitation learning method that achieves strong OOD generalization performance across a variety of tasks. In the future, we believe our work will inspire more scalable and generalizable reinforcement learning research.

\section{Limitations}
\label{ap:limit}
The major limitation of this work is that we mainly consider algorithmic designs motivated by the OOD generalization theory. There are many interesting future directions not included in this paper, e.g.,  studying representation learning \cite{mazoure2021improving}, goal embeddings \cite{islamdiscrete}, world models \cite{anand2021procedural,ding2022generalizing}, and network designs \cite{lee2022multi,xu2022prompting,hong2022bi} to improve OOD generalization for offline RL and offline GCRL.

\section*{Acknowledgements}
This work is supported by GRF 16310222 and GRF 16201320, in part by Science and Technology Innovation 2030 - “New Generation Artificial Intelligence” Major Project (No. 2018AAA0100904) and the National Natural Science
Foundation of China (62176135). The authors would like to thank the anonymous reviewers for their comments to improve the paper.

\nocite{langley00}

\bibliography{main}
\bibliographystyle{icml2023}

\newpage
\appendix
\onecolumn

\part{Appendix}

\section{Algorithm Pseudo Code}
\begin{algorithm}[htb]
   \caption{GOAT Algorithm}
   \label{alg:wgcsl+}
\begin{algorithmic}
   \STATE Initialize policy $\pi_{\theta}$ and $N$ value functions $Q_1, \dots, Q_N$, and two FIFO queues $B_a=\{\}$\ and $B_{std}=\{\}$ ;
   \FOR{training step $=1,2,\ldots$}
        \STATE Sample a mini-batch from the offline dataset: $\{(s_t, a_t,g, r_t, s_{t+1})\} \sim D$;
	    \STATE Relabel the mini-batch with a probability of $p_{relabel}$: $\{(s_t,a_t,g',r_t', s_{t+1})\} \sim D_{relabel}$ ;
	    \STATE Update value functions $Q_i, i\in[1,N]$ to minimize Eq \eqref{eq:td_loss} with the mini-batch ;
	    \STATE Estimate advantage values $A(s_t,a_t,g')$ using $Q_i, i\in[1,N]$ and store them into the queue $B_a$;
	    \STATE Get the $\alpha$ percentile advantage value from $B_a$ to calculate the DSW;
        \STATE Estimate the bootstrapped uncertainty $\Std(s_t,g)$ and store them into $B_{std}$;
        \STATE Compute the UW according to Eq \eqref{eq:final_UW};
	    \STATE Update policy $\pi_{\theta}$ to maximize the objective in Eq \eqref{eq:J_wgcsl} with the mini-batch:
   \ENDFOR
\end{algorithmic}
\end{algorithm}

\section{Theoretical Proofs}
\label{ap:theory}

\subsection{Useful Lemmas}
\begin{lemma}
\label{lemma:V_diff}
    Assume the maximum reward is $R_{max}$, 
    For any two goal-conditioned policies $\pi$ and $\pi_E$, we have that 
     \begin{equation*}
        \begin{split}
             V^{\pi_E}(s_0,g) - V^{\pi}(s_0,g) 
             \leq
            \frac{2 R_{\max}}{(1-\gamma)^2} \mathbb{E}_{s \sim d_{\pi_{E}}(s|s_0,g)}\left[ D_{\mathrm{TV}}\big(\pi(\cdot|s,g), \pi_E(\cdot|s,g) \big) \right]
        \end{split}
    \end{equation*}
\end{lemma}
\begin{proof}
For any policy $\pi$,
its value function can be formulated as $V^\pi = \frac{1}{1-\gamma} \mathbb{E}_{(s,a) \sim \rho_\pi}[r(s, a)]$ \cite{puterman2014markov}. In the goal-conditioned setting, we also need to include goals into consideration. Then, we can derive
\begin{equation*}
    \begin{split}
        \vert V^{\pi_E}(s_0,g) - V^{\pi}(s_0,g) \vert &= \left|\frac{1}{1-\gamma} \mathbb{E}_{(s,a) \sim \rho_{\pi_E}(\cdot|s_0,g)} [r(s,a,g)] - \frac{1}{1-\gamma} \mathbb{E}_{(s,a,g) \sim \rho_{\pi}(\cdot|s_0,g)} [r(s,a,g)]  \right|
        \\
        &\leq \frac{1}{1-\gamma} \sum_{(s,a) \in \mathcal{S} \times \mathcal{A}} \left| \bigl( \rho_{\pi_E}(s,a|s_0,g) - \rho_{\pi}(s,a|s_0,g) \bigr) r(s,a,g) \right|
        \\
        &\leq \frac{2 R_{\textnormal{max}}}{1-\gamma} D_{\textnormal{TV}} (\rho_{\pi_E}(\cdot|s_0,g), \rho_{\pi}(\cdot|s_0,g)).
    \end{split}
\end{equation*}
With Lemma 5 in \cite{DBLP:conf/nips/XuLY20}, we comlete the proof: 
\begin{equation*}
    \vert V^{\pi_E}(s_0,g) - V^{\pi}(s_0,g) \vert \leq \frac{2 R_{\textnormal{max}}}{1-\gamma} D_{\textnormal{TV}} (\rho_{\pi_E}(\cdot|s_0,g), \rho_{\pi}(\cdot|s_0,g)) \leq \frac{2 R_{\max}}{(1-\gamma)^2} \mathbb{E}_{s \sim d_{\pi_{E}}(s|s_0,g)}\left[ D_{\mathrm{TV}}\big(\pi(\cdot|s,g), \pi_E(\cdot|s,g) \big) \right]
\end{equation*}
\end{proof}

\begin{lemma}
\label{lemma:suboptimal}
    Assume the maximum reward is $R_{max}$, 
    For any two goal-conditioned policies $\pi$ and $\pi_E$, we have that 
     \begin{equation*}
        \begin{split}
        \mathrm{SubOpt}(\pi_E, \pi) =\mathbb{E}_{(s_0,g) \sim P_{X}^{\mathcal{T}}} [V^{\pi_E}(s_0,g) - V^{\pi}(s_0,g)] 
             \leq
            \frac{2 R_{\max}}{(1-\gamma)^2} \mathbb{E}_{(s_0,g) \sim P_{X}^{\mathcal{T}}\atop s \sim d_{\pi_{E}}(s|s_0,g)}\left[ D_{\mathrm{TV}}\big(\pi(\cdot|s,g), \pi_E(\cdot|s,g) \big) \right]
        \end{split}
    \end{equation*}
\end{lemma}
Lemma \ref{lemma:suboptimal} is a direct result of combining Lemma \ref{lemma:V_diff} and the definition of the suboptimality in Eq \eqref{eq:suboptimal}.

\begin{definition}
    For any state-goal distribution $\rho(s,g)$, we define 
    \begin{equation*}
        \varepsilon^{\rho}(\pi_{E}, \pi)=\mathbb{E}_{(s,g) \sim \rho(s,g)} \left[ D_{\mathrm{TV}}\big(\pi(\cdot|s,g), \pi_E(\cdot|s,g) \big)\right].
    \end{equation*}
\end{definition}

\begin{lemma}
\label{lemma:triangle}
    For any three policy $\pi_1, \pi_2, \pi_3$ and any state-goal distribution $\rho$, we have:
    \begin{equation*}
        \varepsilon^{\rho}(\pi_{1}, \pi_2)\leq \varepsilon^{\rho}(\pi_{1}, \pi_3) + \varepsilon^{\rho}(\pi_{3}, \pi_2)
    \end{equation*}
\end{lemma}
\begin{proof}
    This can be proved by noticing that $D_{TV}$ is a distance metric.
    \begin{equation*}
        \begin{aligned}
            \varepsilon^{\rho}(\pi_{1}, \pi_2) &=\mathbb{E}_{(s,g) \sim \rho(s,g)} \left[ D_{\mathrm{TV}}\big(\pi_1(\cdot|s,g), \pi_2(\cdot|s,g) \big)\right] \\
            &\leq  \mathbb{E}_{(s,g) \sim \rho(s,g)} \left[ D_{\mathrm{TV}}\big(\pi_1(\cdot|s,g), \pi_3(\cdot|s,g) \big) + D_{\mathrm{TV}}\big(\pi_3(\cdot|s,g), \pi_2(\cdot|s,g) \big)\right] \\
            &\leq \varepsilon^{\rho}(\pi_{1}, \pi_3) + \varepsilon^{\rho}(\pi_{3}, \pi_2)
        \end{aligned}
    \end{equation*}
\end{proof}

\begin{lemma}
    \label{lemma:basic_policy_gap_bound}
    Assume the expert policy $\pi_E$ is invariant across training and testing domains. For a policy $\pi$, we have
    \begin{equation*}
    \begin{aligned}
        \varepsilon^{\mathcal{T}}(\pi_{E}, \pi) \leq 
        \varepsilon^{\mathcal{S}}(\pi_{E}, \pi) +  d_1(\mathcal{T}, \mathcal{S})
    \end{aligned}
    \end{equation*}
    where the variation divergence $d_1$ between two distribution $S_1$ and $S_2$ is defined as follows:
\begin{align*}
    d_1(S_1, S_2) = 2\sup_{\mathcal{J} \subset \mathcal{X}} \left|\sum_{x \in \mathcal{J}} \left( P_{S_1}(x) -  P_{S_2}(x) \right) \right|,
\end{align*}
\end{lemma}

\begin{proof}
    With the definition of variation divergence, we have 
    \begin{equation*}
        \begin{aligned}
            \varepsilon^{\mathcal{T}}(\pi_{E}, \pi) &= \varepsilon^{\mathcal{T}}(\pi_{E}, \pi) + \varepsilon^{\mathcal{S}}(\pi_{E}, \pi) - \varepsilon^{\mathcal{S}}(\pi_{E}, \pi) 
             \leq \varepsilon^{\mathcal{S}}(\pi_{E}, \pi) + | \varepsilon^{\mathcal{T}}(\pi_{E}, \pi) - \varepsilon^{\mathcal{S}}(\pi_{E}, \pi)| \\
            &= \varepsilon^{\mathcal{S}}(\pi_{E}, \pi) + \frac{1}{2} \sum_{(s_0,g)}  | P_X^{\mathcal{T}}(s_0,g) - P_X^{\mathcal{S}}(s_0,g)| \sum d_{\pi_E}(s|s_0,g)\sum_a | \pi_E(a|s,g)-\pi(a|s,g)| \\
             & \leq \varepsilon^{\mathcal{S}}(\pi_{E}, \pi) + \sum_{(s_0,g)}  | P_X^{\mathcal{T}}(s_0,g) - P_X^{\mathcal{S}}(s_0,g)| \\
            &\leq \varepsilon^{\mathcal{S}}(\pi_{E}, \pi) +  d_1(\mathcal{T}, \mathcal{S})
        \end{aligned}
    \end{equation*}
\end{proof}



\begin{lemma}[Generalization Bound for Finite ERM]
\label{lemma:finite_erm}
Consider finite hypothesis space $\mathcal{F}$ and bounded loss function in $[a,b]$. When optimizing empirical loss function $\hat L(f)=\frac{1}{m}\sum_i^m \mathcal{L}(f(x_i),y_i)$ instead of the expected one $ L(f)=\mathbb{E}_{(x,y)} \mathcal{L}(f(x),y)$, with probability as least $1-\delta$, the true loss can be bounded as:
\begin{equation*}
\begin{aligned}
     L(f) \leq \hat L(f) + \sqrt{\frac{(b-a)^2 (\log 2|\mathcal{F}|+\log \frac{1}{\delta})}{2m}}
\end{aligned}
\end{equation*}
\end{lemma}
Lemma \ref{lemma:finite_erm} is a well-known result from \cite{mohri2018foundations}.

\subsection{Proof of Theorem \ref{tm:performance_gap_bound_finite}}
\label{ap:proof_main}
\begin{proof}
    With Lemma \ref{lemma:suboptimal} and the definition of policy discrepancy on training and testing distributions:
    \begin{equation*}
    \begin{aligned}
    &\varepsilon^{\mathcal{T}}(\pi_{E}, \pi)=\mathbb{E}_{(s_0,g) \sim P_{X}^{\mathcal{T}} \atop s \sim d_{\pi_{E}}(s|s_0,g)}\left[ D_{\mathrm{TV}}\big(\pi(\cdot|s,g), \pi_E(\cdot|s,g) \big)\right], \\
&\varepsilon^{\mathcal{S}}(\pi_{E}, \pi)=\mathbb{E}_{(s_0,g) \sim P_{X}^{\mathcal{S}}\atop s \sim d_{\pi_{E}}(s|s_0,g)}\left[ D_{\mathrm{TV}}\big(\pi(\cdot|s,g), \pi_E(\cdot|s,g) \big)\right].
    \end{aligned}
\end{equation*}
    
    we have 
    \begin{equation*}
         \mathrm{SubOpt}(\pi_E, \pi) 
             \leq
            \frac{2 R_{\max}}{(1-\gamma)^2} \mathbb{E}_{(s_0,g) \sim P_{X}^{\mathcal{T}}\atop s \sim d_{\pi_{E}}(s|s_0,g)}\left[ D_{\mathrm{TV}}\big(\pi(\cdot|s,g), \pi_E(\cdot|s,g) \big) \right] = \frac{2 R_{\max}}{(1-\gamma)^2} \varepsilon^{\mathcal{T}}(\pi_{E}, \pi)
    \end{equation*}
Regarding $\varepsilon^{\mathcal{T}}(\pi_{E}, \pi)$, we use Lemma \ref{lemma:triangle} and Lemma \ref{lemma:basic_policy_gap_bound} to obtain an upper bound:
\begin{equation*}
    \varepsilon^{\mathcal{T}}(\pi_{E}, \pi) \leq \varepsilon^{\mathcal{S}}(\pi_{E}, \pi) + d_1(\mathcal{T}, \mathcal{S}) \leq \varepsilon^{\mathcal{S}}(\hat \pi_{E}, \pi) + \varepsilon^{\mathcal{S}}(\pi_{E}, \hat \pi_E) + d_1(\mathcal{T}, \mathcal{S}) 
\end{equation*}
When we use finite sample to estimate $\varepsilon^{\mathcal{S}}(\hat \pi_{E},  \pi)$, the true loss can be bounded with Lemma \ref{lemma:finite_erm}. Note that the $D_{\mathrm{TV}}$ can be bounded in $[0,1]$.
Therefore, we can complete the proof. With probability at least $1-\delta$,
\begin{equation*}
    \begin{aligned}
        \mathrm{SubOpt}(\pi_E, \pi) \leq  \frac{2R_{max}}{(1-\gamma)^2} \bigg[  \hat \varepsilon^{\mathcal{S}}(\hat \pi_{E}, \pi)
        + \varepsilon^{\mathcal{S}}(\hat \pi_{E}, \pi_E) 
        + d_1(\mathcal{T}, \mathcal{S}) + \sqrt{\frac{\log 2|\Pi|+\log \frac{1}{\delta}}{2m}}  \bigg]
    \end{aligned}
\end{equation*}
\end{proof}

\subsection{Proof of Theorem \ref{tm:distribution_shift}}
\label{ap:proof_distribution}
\begin{proof} 
\textbf{Step 1}. First we explicitly show 
\begin{align}
    \label{eqn:uniform_value}
    \sup_{Z \in \mathcal{Z}} d_1 (Z, \bar S) = 2(1 - \frac{1}{C|\mathcal{X}|}).
\end{align} 
On one side, simple algebra shows that the following distribution $S' \in \mathcal{S}$ will induce the distance shown in Eq~\eqref{eqn:uniform_value}:
$$P_{S'}(x) = 
\begin{cases}
C, \mbox{ if } x \in \mathcal{J},\\
0, \mbox{ otherwise }.
\end{cases}$$
where $|\mathcal{J}| = 1/C < |\mathcal{X}|$. 
On the other hand, we are going to show there is no distribution that can elicit a distance larger than that in Eq~\eqref{eqn:uniform_value}. 
We prove it by contradiction by assuming a distribution $S''$ which has 
\begin{align}
    \label{eqn:contradict_assumption}
     d_1 (S'', \bar S) > 2(1 - \frac{1}{C|\mathcal{X}|}).
\end{align}
Then there exists $\mathcal{J}'' \subset \mathcal{X}$ such that
\begin{align}
    \left|\int_{x \in \mathcal{J}''} \left( P_{S''}(x) -  P_{\bar S}(x) \right) dx \right | > (1 - \frac{1}{C|\mathcal{X}|}).
\end{align}
With out loss of generality, we assume 
\begin{align}
\label{eqn:contradict_larger_0}
    \int_{x \in \mathcal{J}''} \left( P_{S''}(x) -  P_{\bar S}(x) \right) dx > (1 - \frac{1}{C|\mathcal{X}|})
\end{align}
Denote 
$\bar x_{\mathcal{J}''} = \frac{1}{|\mathcal{J}''|} \int_{x \in \mathcal{J}''} P_{S''}(x) dx$. It is clear that $\bar x_{\mathcal{J}''} \leq C$. Then the RHS of Eq~\eqref{eqn:contradict_larger_0} is 
\begin{align}
    |\mathcal{J}''|(\bar x_{\mathcal{J}''} - \frac{1}{|\mathcal{X}|}) = & |\mathcal{J}''|\bar x_{\mathcal{J}''}(1 - \frac{1}{|\mathcal{X}|\bar x_{\mathcal{J}''}} ) \\
    & \leq (1 - \frac{1}{|\mathcal{X}|\bar x_{\mathcal{J}''}} ) | \\
    & \leq (1 - \frac{1}{C|\mathcal{X}|} ),
\end{align}
where the first inequality is due to $|\mathcal{J}''|\bar x_{\mathcal{J}''} \leq 1$ and the second inequality is due to $\bar x_{\mathcal{J}''}\leq C$. Thus we arrive at a contradiction. So Eq~\eqref{eqn:contradict_assumption} does not hold. Putting these together, we show that Eq~\eqref{eqn:uniform_value} holds.

\textbf{Step 2}. We now proceed to show $\forall S \in \mathcal{S}^{-}$,
\begin{align}
    \label{eqn:step_2_target}
    \sup_{Z \in \mathcal{Z}} d_1 (Z, S) > 2(1 - \frac{1}{C|\mathcal{X}|}).
\end{align}
We know that for any $S$ in $\mathcal{S}^{-}$, there exists a subset $\mathcal{J} \subset \mathcal{X}$ such that 
\begin{align}
    \label{eqn:J_condition}
    \int_{x \in \mathcal{J}}P_{S}(x) dx <  |\mathcal{J}| / |\mathcal{X}|
\end{align}
Let $J_M$ be the subset of $\mathcal{X} \slash \mathcal{J}$ which contains the smallest $1/C - |\mathcal{J}|$ points:
\begin{align}
    \mathcal{J}_M := \min_{\mathcal{M} \subset \mathcal{X} \slash \mathcal{J}, |\mathcal{M}| = 1/C - |\mathcal{J}|} \int_{\mathcal{M}} P_S(x) dx.
\end{align}
By the definition of $\mathcal{J}_M$, it is easy to see that  the mean density ratio $\mathcal{X} ~\slash~(\mathcal{J}~\cup ~\mathcal{J}_M)$ is larger than that of $\mathcal{J}_M$, 
\begin{align}
    \label{eqn:small_condition}
    \frac{1}{|\mathcal{X}|-1/C} & \int_{x \in  \mathcal{X} \slash (\mathcal{J} \cup \mathcal{J}_M)}P_S(x) dx \geq \nonumber \\
    & \frac{1}{1/C - |\mathcal{J}|}\int_{x \in   \mathcal{J}_M}P_S(x) dx,
\end{align}
We now proceed to prove (by contradiction) that 
\begin{align}
    \label{eqn:must_C}
    \frac{1}{1/C}\int_{x \in   \mathcal{J}_M \cup \mathcal{J} }P_S(x) dx < 1/|\mathcal{X}|.
\end{align}
We first assume 
\begin{align}
    \label{eqn:J_M_J_contri_ass}
    \frac{1}{1/C}\int_{x \in   \mathcal{J}_M \cup \mathcal{J} }P_S(x) dx \geq  1/|\mathcal{X}|
\end{align}
By Eq~\eqref{eqn:J_condition} and ~\eqref{eqn:J_M_J_contri_ass}, we have
\begin{align}
    \int_{x \in   \mathcal{J}_M}P_S(x) dx > \frac{{1/C - |\mathcal{J}|}}{|\mathcal{X}|},
\end{align}
and further with Eq \eqref{eqn:small_condition} we have
\begin{align}
    \label{eqn:2_condition}
     \int_{x \in  \mathcal{X} \slash (\mathcal{J} \cup \mathcal{J}_M)}P_S(x) dx > \frac{{|\mathcal{X}| - 1/C}}{|\mathcal{X}|}.
\end{align}
Putting Eq \eqref{eqn:2_condition} and Eq \eqref{eqn:J_M_J_contri_ass} together, we have
\begin{align}
    \int_{x \in  \mathcal{X} }P_S(x) dx > 1,
\end{align}
which arrives at a contradiction. So Eq \eqref{eqn:must_C} holds.  
We then construct $Z$ as 
\begin{align}
\label{eqn:construct_Q}
    P_{Z}(x) = 
\begin{cases}
C, \mbox{ if } x \in \mathcal{J} \cup \mathcal{J_M},\\
0, \mbox{ otherwise }.
\end{cases}
\end{align}
With Eq~\eqref{eqn:construct_Q} and Eq~\eqref{eqn:must_C}, we have
\begin{align}
    \int_{x \in \mathcal{J}} \left( P_{Z}(x) -  P_{S}(x) \right) dx > 1 - \frac{1}{C|\mathcal{X}|}.
\end{align}
So we prove Eq \eqref{eqn:step_2_target}. 

Putting Step 1 and 2 together, we finish the proof.
\end{proof}

\section{Offline Datasets and Implementation Details}
\label{ap:data_implementation}

\subsection{Offline Datasets}
For the benchmark tasks, offline datasets are collected by the final online policy trained with HER \cite{andrychowicz2017hindsight}. Additional Gaussian noise with zero mean and $0.2$ standard deviation and random actions with probability 0.3 is used for data collection to increase the diversity, following previous work \citep{yang2022rethinking}. For the FetchSlide task, we only use noise with a standard deviation of 0.1 because the behavior policy is already suboptimal. After data collection, different from \cite{yang2022rethinking}, we need additional data processing to select trajectories whose achieved goals are all in the IID region. The IID region is defined by each task group, which is shown in Table \ref{tab:dataset_info}. A special case is the HandReach task, where we do not divide the dataset due to its high dimensional space and we use different scales of evaluation goals instead. Compared with prior offline GCRL works \cite{yang2022rethinking,ma2022far}, we use relatively smaller datasets to study the OOD generalization problem. Our datasets encompass trajectories of different length, ranging from 200 to 20000, with each trajectory comprising 50 transitions. A comprehensive summary of this information is presented in Table \ref{tab:dataset_info}. The dataset division standard refers to the location requirements of initial states and desired goals for IID tasks (e.g., Right2Right). For OOD tasks, the initial state or the the desired goal are designed to deviate from the IID requirement (e.g., Right2Left, Left2Right, Left2Left).

\begin{table}[ht]
    \centering
    \caption{Information about 9 Task Groups and Datasets.}
     \begin{adjustbox}{max width=\linewidth}
    \begin{tabular}{l|cccccccc}
    \toprule
           Datasets (Task Group) & IID task & OOD task  & Trajectory number &  Size ($M$) & Dataset Division Standard\\
    \midrule
       Reach Left-Right  & Right & Left & 200 & 1.6 & the gripper's y coordinate value $>$ the initial position  \\
       Reach Near-Far  &  Near & Far & 200 & 1.6 & the $l_2$ distance between gripper and the initial position $\leq$ 0.15 \\
       Push Left-Right & Right2Right & Right2Left, Left2Right, Left2Left & 5000 & 67 &  the object's y coordinate value $>$ the initial position\\
       Push Near-Far  & Near2Near & Near2Far, Far2Near, Far2Far &  5000 & 67 & the $l_2$ distance between the object and the initial position $\leq$ 0.15  \\
       Pick Left-Right & Right2Right & Right2Left, Left2Right, Left2Left & 5000 & 67 & the object's y coordinate value $>$ the initial position \\
       Pick Low-High &  Low & High & 5000 &  67 & the object's z coordinate value $<$ 0.6 \\
       Slide Left-Right & Right2Right & Right2Left, Left2Right, Left2Left & 20000 & 266 & the object's y coordinate value $>$ the initial position \\
       Slide Near-Far & Near & Far & 20000 & 266 & the object's x coordinate value $\leq$ 0.14 \\
       HandReach Near-Far  & Near & Far & 10000 & 429 & the range of meeting position for two fingers \\
        \bottomrule
    \end{tabular}
    \end{adjustbox}
    \label{tab:dataset_info}
\end{table}

\begin{table}[ht]
    \centering
    \caption{$w$ and $\tau$ used for GOAT and GOAT($\tau$).}
     \begin{adjustbox}{max width=\linewidth}
    \begin{tabular}{l|cccccccc}
    \toprule
           Task Group & GOAT & GOTA($\tau$)  \\
    \midrule
       Reach Left-Right  & $w=1.5$ & $w=2.5, \tau=0.3$   \\
       Reach Near-Far  &  $w=2.0$ & $w=1.5, \tau=0.1$ \\
       Push Left-Right & $w=2.5$ & $w=1.5, \tau=0.1$ \\
       Push Near-Far  & $w=1.5$ & $w=2.5, \tau=0.1$  \\
       Pick Left-Right & $w=1.0$ & $w=2.5, \tau=0.3$ \\
       Pick Low-High &  $w=2.0$ & $w=1.5, \tau=0.3$  \\
       Slide Left-Right & $w=1.5$ & $w=2.5, \tau=0.1$  \\
       Slide Near-Far & $w=2.0$ & $w=1.5, \tau=0.3$ \\
       HandReach Near-Far  & $w=2.5$ & $w=2.0, \tau=0.1$ \\
        \bottomrule
    \end{tabular}
    \end{adjustbox}
    \label{tab:hyperpara_table}
\end{table}

\subsection{Implementation Details}
\label{ap:implementation_details}
\paragraph{Implementations} Following \cite{yang2022rethinking,ma2022far}, value functions and policy networks (along with their target networks) are all 3-layer MLPs with 256-unit layers and relu activations. We use a batch size of 512, a discount factor of $\gamma=0.98$, and an Adam optimizer with learning rate $5\times 10^{-4}$ for all algorithms. We also normalize the observations and goals with estimated mean and standard deviation. The relabel probability $p_{relabel}=1$ for most environments except for Slide Left-Right and Slide Near-Far, where $p_{relabel}=$ 0.2 and 0.5, respectively. In \textbf{EAW}, the ratio $\beta$ is set to $2$ and EAW is clipped into range $(0, M]$ for numerical stability, where $M$ is set to $10$ in our experiments. For \textbf{DSW}, we utilize a First-In-First-Out (FIFO) queue $B_a$ of size $5\times 10^4$ to store recent calculated advantage values, and the percentile threshold $\alpha$ gradually increases from $0$ to $\alpha_{max}$. We use $\alpha_{max}=80$ for all tasks except HandReach and Slide Left-Right, and $\alpha_{max}=50$ for HandReach, $\alpha_{max}=0$ for Slide Left-Right. When $A(s,a,g') < c$ and $c$ is the $\alpha$ quantile value of $B_a$, we set $\epsilon(A(s, a, g'))=0.05$ instead of $0$ following \cite{yang2022rethinking}. For the uncertainty weight (\textbf{UW}), we use $N=5$ ensemble $Q$ networks to calculate the standard deviation $\Std(s,g)$ and maintain another FIFO queue $B_{std}$ to store recent $\Std(s,g)$ values. The $\Std(s,g)$ values are then normalized to $[0,1]$ with the maximum and minimum values in $B_{std}$. Besides, $w_{min}$ is set to 0.5 and $w$ is searched from $\{1, 1.5, 2, 2.5\}$. In our experiments, we still find $w$ is unstable for different tasks. It is therefore necessary to develop a more stable uncertainty weight estimation method with less hyperparameter tuning in the future. Regarding the expectile regression (\textbf{ER}), we search $\tau\in\{0.1,0.3\}$ because empirical results in Appendix \ref{ap:expectile_regression} shows that $\tau\in\{0.1,0.3\}$ performs the best. The hyperparameters $w$ and $\tau$ for GOAT and GOAT($\tau$) are listed in Table \ref{tab:hyperpara_table}.

\paragraph{Baseline Descriptions}
In our experiments, all the baselines share the same policy and value network structures, as well as hyperparameters. As regards to WGCSL \cite{yang2022rethinking} and GoFAR \cite{ma2022far}, we use their official implementations. Denote the original dataset as $D$ and the relabeled dataset as $D_{relabel}$. In the following part, we introduce several baselines used in our paper. 
\begin{itemize}
    \item \textbf{WGCSL}: using relabeled offline samples to maximize
    $$J_{WGCSL}(\pi_{\theta})=\mathbb{E}_{(s_t,a_t,\phi(s_i))\sim D_{relabel}}[w_{t,i} \cdot \log\pi_{\theta}(a_t|s_t,\phi(s_i))]$$
    where $w_{t,i}= \gamma^{i-t} \exp_{clip}(\beta A(s_t,a_t,\phi(s_i))) \cdot \epsilon( A(s_t,a_t,\phi(s_i)))$. In our paper, since we find DRW (i.e., $\gamma^{i-t}$) is not useful for OOD generalization (see Appendix \ref{ap:DRW}), we just set DRW=1. Besides, $\beta$ is set to 2 as GOAT.
    \item \textbf{GoFAR}:  $$J_{\text{GoFAR}}(\pi_{\theta})=\mathbb{E}_{(s_t,a_t,g)\sim D}[\log\pi_{\theta}(a_t|s_t,g) \max(A(s,a,g)+1,0)],$$
    where the advantage function is estimated by discriminator-based rewards. The discriminator $c$ is learned to minimize $\mathbb{E}_{g\sim p(g)}[\mathbb{E}_{p(s;g)}\big[ \log c(s,g)]+\mathbb{E}_{(s,g)\sim D}[\log (1-c(s,g))] \big]$. The value function $V$ is learned to minimize $(1-\gamma) \mathbb{E}_{(s,g)\sim \mu_0,p(g)}[V(s,g)] + \frac{1}{2} \mathbb{E}_{(s,a,g,s')\sim D} [(r(s;g)+\gamma  V(s';g)-V(s;g)+1)^2]$ for $V \geq 0$.
    \item \textbf{BC}: behavior cloning on original offline samples to maximize
    $J_{BC}(\pi)=\mathbb{E}_{(s_t,a_t,g)\sim D}[ \log\pi(a_t|s_t,g)]$.
    \item \textbf{GCSL}: using relabeled offline dataset to maximize
    $J_{GCSL}(\pi_{\theta})=\mathbb{E}_{(s_t,a_t,g')\sim D_{relabel}}[ \log\pi_{\theta}(a_t|s_t,g')]$.
    \item \textbf{MARWIL+HER}: using relabeled offline dataset to maximize
    $$J_{\text{MARWIL+HER}}(\pi_{\theta})=\mathbb{E}_{(s_t,a_t,g')\sim D_{relabel}}[\log\pi_{\theta}(a_t|s_t,g') \exp(\beta A(s_t,a_t,g'))].$$
    We also clip the exponential weight to $(0,10]$ for numerical stability. $\beta$ is set to 2 similar to WGCSL and GOAT.

     \item \textbf{IQL+HER}:
     using expectile regression $L_2^{\tau}$, an additional $V$ network and weighted imitation learning to remove OOD actions from value estimation. Hyperparameters are set the same as WGCSL, MARVIL and GOAT.
     \begin{equation*}
     L_V(\psi)=\mathbb{E}_{(s_t,a_t,g')\sim D_{relabel}}[L_2^{\tau}(Q_{\phi}(s_t,a_t,g')-V_{\psi}(s_t,g'))]
     \end{equation*}
     \begin{equation*}
L_Q(\phi)=\mathbb{E}_{(s_t,a_t, s_{t+1}, g')\sim D_{relabel}}[(r(s_t,a,g')+\gamma V_{\psi}(s_{t+1},g') - Q_{\phi}(s_t,a_t,g'))^2]
     \end{equation*}
     \begin{equation*}
J(\pi_{\theta})=\mathbb{E}_{(s_t,a_t, g')\sim D_{relabel}}[\exp{(\beta Q_{\phi}(s_t,a_t,g')-V_{\psi}(s,g'))}\log \pi_{\theta}(a|s_t,g') ]
     \end{equation*}

    \item \textbf{CQL+HER}: For a fair comparison, we implement CQL on top of DDPG+HER. The objective of CQL+HER is:
    $$
    J_{\text{CQL+HER}}(\pi_{\theta})=\mathbb{E}_{(s_t,g')\sim D_{relabel}}[Q(s_t, \pi_{\theta}(s_t,g'),g')]$$
    The $Q$ function of CQL+HER is learned by minimizing the following loss:
    \begin{equation*}
    \begin{aligned}
        L_{\text{CQL+HER}} = \mathbb{E}_{(s_t,a_t, s_{t+1},g')\sim D_{relabel}} \big[(Q(s_t,a_t,g') - \mathcal{B}^{\pi}Q(s_t,a_t,g'))^2 \big] + \alpha \mathbb{E}_{(s_t,g')\sim D_{relabel}, a\sim \exp{(Q)}} \big[Q(s_t,a,g'] \big] 
    \end{aligned}
    \end{equation*}
    where $\mathcal{B}^{\pi}$ is the Bellman operator. $\alpha$ is the ratio to balance the CQL loss and the TD loss. Another baseline DDPG+HER is exactly $\alpha=0$.

    \item \textbf{MSG+HER}: We implement MSG based on ensemble DDPG with $N=5$ independent $Q$ networks and an LCB objective. Each $Q$ network learns to minimize the TD loss in Eq \eqref{eq:td_loss}. The policy learns to maximize $$J_{MSG+HER}(\pi_{\theta})=\mathbb{E}_{(s_t,g')\sim D_{relabel}, a\sim \pi_{\theta}(s_t,g')}[\frac{1}{N} \sum_{i=1}^N  Q_i(s_t,a,g') - c \cdot \sqrt{ \mathrm{Var}(Q_1(s_t,a,g'),\ldots, Q_N(s_t,a,g'))} ]$$
\end{itemize}

\begin{figure}[t]
    \centering
\includegraphics[width=1\linewidth]{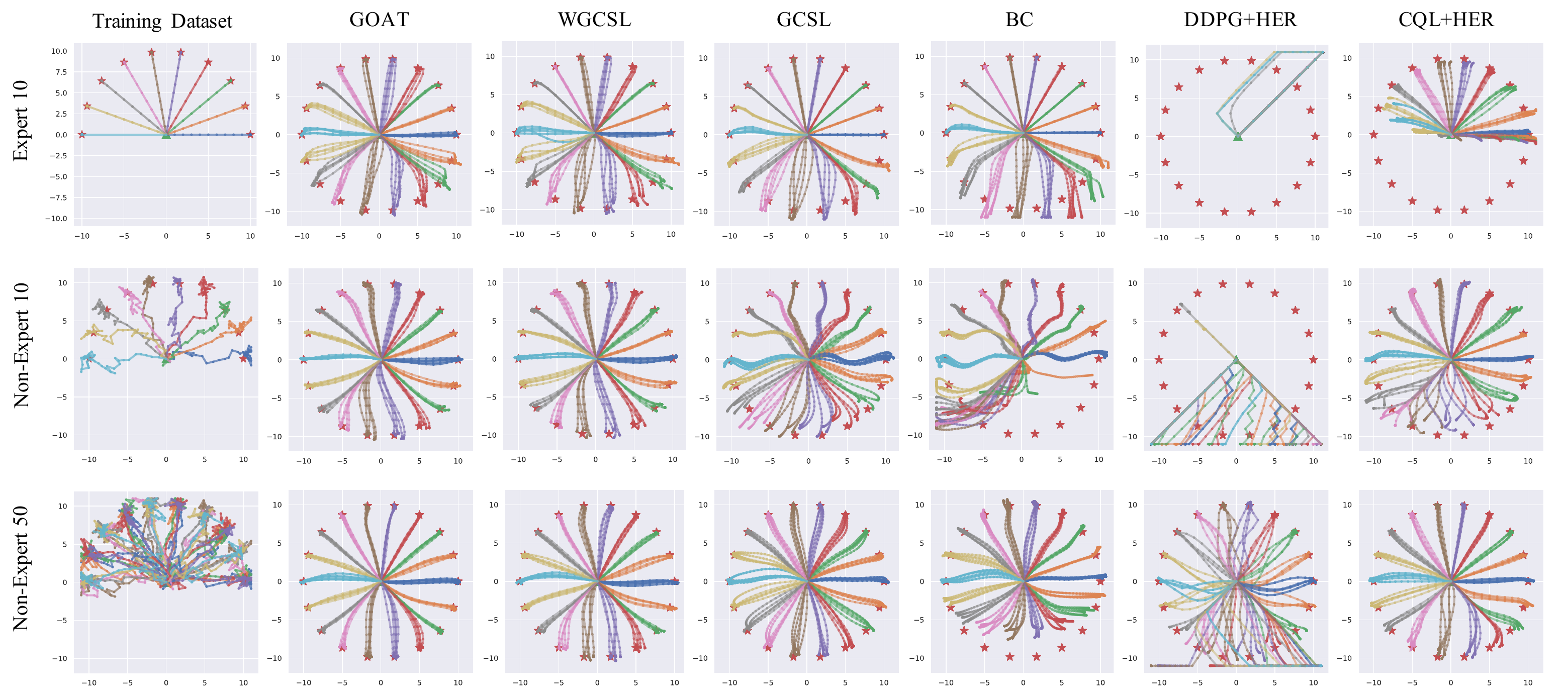}
    \caption{Evaluation of different agents on 2D PointReach task over 5 random seeds. ``Expert 10" refers to the clean expert dataset, ``Non-Expert N" refers to the noisy dataset with $N = 10, 50$ trajectories, respectively. Training trajectories are mainly on the upper semicircle, and the evaluation goals are on the full circle of radius 10.}
    \label{fig:point_all}
\end{figure}

\begin{table}[htb]
    \centering
    \caption{Final average success rates of different agents trained on three types of PointReach datasets. The results are averaged over 5 random seeds.}
    \label{tab:point_success}
    \begin{adjustbox}{max width=\linewidth}
    \begin{tabular}{llccccccc}
    \toprule
   &   &  GOAT(ours) &  WGCSL (tuned) &  WGCSL &  GCSL &  Goal BC &  HER &  CQL+HER \\
    \midrule
  \multirow{2}{*}{Expert 10} &  R10& \textbf{0.86 $\pm$ 0.08} & \textbf{0.94 $\pm$ 0.06} & 0.82 $\pm$ 0.07 & 0.72 $\pm$ 0.09 & 0.67 $\pm$ 0.02 & 0.0 $\pm$ 0.0 & 0.39 $\pm$ 0.05\\
   & R20& \textbf{0.48 $\pm$ 0.21} & 0.30 $\pm$ 0.14 & 0.40 $\pm$ 0.19 & 0.03 $\pm$ 0.06 & \textbf{0.45 $\pm$ 0.10} & 0.0 $\pm$ 0.0 & 0.01 $\pm$ 0.02\\

\midrule
\multirow{2}{*}{Non-Expert 10} & R10& \textbf{0.94 $\pm$ 0.08} & 0.89 $\pm$ 0.12 & \textbf{0.94 $\pm$ 0.1} & 0.66 $\pm$ 0.18 & 0.29 $\pm$ 0.06 & 0.0 $\pm$ 0.0 & 0.52 $\pm$ 0.05\\
& R20& \textbf{0.69 $\pm$ 0.18} & \textbf{0.63 $\pm$ 0.19} & 0.53 $\pm$ 0.1 & 0.12 $\pm$ 0.09 & 0.06 $\pm$ 0.04 & 0.0 $\pm$ 0.0 & 0.16 $\pm$ 0.06\\

\midrule
\multirow{2}{*}{Non-Expert 50} & R10 & \textbf{1.00 $\pm$ 0.00} & \textbf{0.98 $\pm$ 0.04} & 0.96 $\pm$ 0.08 & \textbf{0.98 $\pm$ 0.04} & 0.57 $\pm$ 0.02 & 0.60 $\pm$ 0.30 & 0.84 $\pm$ 0.12\\
 & R20& \textbf{0.92 $\pm$ 0.04} & \textbf{0.91 $\pm$ 0.14} & 0.81 $\pm$ 0.14 & 0.33 $\pm$ 0.12 & 0.10 $\pm$ 0.04 & 0.16 $\pm$ 0.16 & 0.75 $\pm$ 0.08\\
    \bottomrule
    \end{tabular}
    \end{adjustbox}
\end{table}

\section{Additional Experiments}
\label{ap:addtional_exp}
In this section, we include the following experiments:
\begin{itemize}
    \item[1.] 2D PointReach Task;
    \item[2.] Uncertainty and Density;
    \item[3.] Ablation on the Ensemble Size;
    \item[4.] Additional Ablations of GOAT;
    \item[5.] The Effectiveness of Expectile Regression;
    \item[6.] Ablations of Ensemble and HER for Other GCRL Algorithms;
    \item[7.] Ablations of GoFAR;
    \item[8.] Combining Weighted Imitation Learning with Value Underestimation;
    \item[9.] Discounted Relabeling Weight (DRW);
    \item[10.] Adaptive Data Selection Weight;
    \item[11.] MSG+HER with Varying Hyper-parameter;
    \item[12.] Online Fine-tuning;
    \item[13.] Training Time;
    \item[14.] Random Network Distillation as the Uncertainty Measurement;
    \item[15.] Full Benchmark Experiments.
\end{itemize}

\subsection{2D PointReach Task}
\label{ap:2d_reach}
\paragraph{Average Success Rates and Cumulative Returns} In Table \ref{tab:point_success} and Table \ref{tab:point_return}, we provide average success rates and average cumulative returns of different algorithms on ``Expert 10", ``Non-Expert 10", and ``Non-Expert 50" datasets. We also include the performance of our method GOAT for comparison. As demonstrated in the two tables, GOAT achieves the highest average success rates and average returns on all three types of training data, surpassing the strong baseline WGCSL. In addition, we visualize trajectories collected by these agents in Figure \ref{fig:point_all}. The results also show that GOAT performs better than WGCSL with a smaller trajectorie variance on the lower semicircle. Other conclusions keep the same as the didactic example in Section \ref{sec:didactic_example}.

\begin{table}[h]
    \centering
    \caption{Final average returns of different agents trained on three types of PointReach datasets. The results are averaged over 5 random seeds.}
    \label{tab:point_return}
    \begin{adjustbox}{max width=\linewidth}
    \begin{tabular}{llccccccc}
    \toprule
   &   &  GOAT(ours) &  WGCSL (tuned) &  WGCSL &  GCSL &  Goal BC &  HER &  CQL+HER \\
    \midrule
  \multirow{2}{*}{Expert 10} & R10& \textbf{34.69 $\pm$ 2.81} & \textbf{38.37 $\pm$ 2.36} & 33.76 $\pm$ 3.07 & 30.23 $\pm$ 3.22 & 28.48 $\pm$ 0.75 & 0.05 $\pm$ 0.0 & 15.22 $\pm$ 1.16\\
 & R20 & \textbf{15.34 $\pm$ 5.35} & 10.35 $\pm$ 3.96 & 12.99 $\pm$ 5.49 & 1.75 $\pm$ 2.24 & \textbf{14.62 $\pm$ 2.77} & 0.00 $\pm$ 0.0 & 0.14 $\pm$ 0.28\\

\midrule
\multirow{2}{*}{Non-Expert 10} & R10& \textbf{37.57 $\pm$ 3.21} & 35.50 $\pm$ 4.88 & \textbf{37.34 $\pm$ 3.56} & 22.42 $\pm$ 4.75 & 12.62 $\pm$ 1.89 & 0.21 $\pm$ 0.32 & 20.90 $\pm$ 1.06\\
 & R20& \textbf{20.88 $\pm$ 5.84} & \textbf{19.27 $\pm$ 5.52} & 16.19 $\pm$ 2.89 & 2.42 $\pm$ 2.09 & 1.58 $\pm$ 0.75 & 0.00 $\pm$ 0.00 & 3.93 $\pm$ 1.20\\

\midrule
\multirow{2}{*}{Non-Expert 50} & R10& \textbf{39.99 $\pm$ 0.1} & \textbf{39.22 $\pm$ 1.46} & 38.23 $\pm$ 3.29 & 32.84 $\pm$ 1.75 & 18.72 $\pm$ 0.90 & 24.70 $\pm$ 12.41 & 32.50 $\pm$ 4.33\\
 & R20& \textbf{27.89 $\pm$ 1.1} & \textbf{27.34 $\pm$ 4.19} & 23.93 $\pm$ 4.37 & 6.53 $\pm$ 2.41 & 3.65 $\pm$ 0.78 & 5.21 $\pm$ 4.96 & 21.09 $\pm$ 2.53\\
    \bottomrule
    \end{tabular}
    \end{adjustbox}
\end{table}

\begin{table}[h]
    \centering
    \caption{Final average success rates of CQL+HER agents with different hyperparameter $\alpha$. The results are averaged over 5 random seeds.}
    \label{tab:point_CQL_success}
    \begin{adjustbox}{max width=\linewidth}
    \begin{tabular}{llccccc}
    \toprule
  CQL+HER &   &  $\alpha=5$ &  $\alpha=2$ & $\alpha=1$ &   $\alpha=0.1$ &  $\alpha=0.01$  \\
    \midrule
  \multirow{2}{*}{Expert 10} &  R10 & 0.34 $\pm$ 0.07 & \textbf{0.45 $\pm$ 0.03} & \textbf{0.39 $\pm$ 0.05} & 0.36 $\pm$ 0.07 & 0.21 $\pm$ 0.08 \\
   & R20& \textbf{0.09 $\pm$ 0.08} & \textbf{0.08 $\pm$ 0.04} & 0.01 $\pm$ 0.02 & 0.06 $\pm$ 0.06 & 0.06 $\pm$ 0.07 \\

\midrule
\multirow{2}{*}{Non-Expert 10} & R10& 0.28 $\pm$ 0.17 & \textbf{0.53 $\pm$ 0.07} & 0.52 $\pm$ 0.05 & \textbf{0.60 $\pm$ 0.10} & 0.10 $\pm$ 0.15 \\
& R20& 0.07 $\pm$ 0.04 & \textbf{0.19 $\pm$ 0.07} & 0.16 $\pm$ 0.06 & \textbf{0.32 $\pm$ 0.04} & 0.02 $\pm$ 0.04 \\

\midrule
\multirow{2}{*}{Non-Expert 50} & R10 & 0.57 $\pm$ 0.04 & 0.76 $\pm$ 0.05 & \textbf{0.84 $\pm$ 0.12} & \textbf{0.95 $\pm$ 0.10} & 0.77 $\pm$ 0.21 \\
 & R20& 0.21 $\pm$ 0.12 & 0.52 $\pm$ 0.12 & \textbf{0.75 $\pm$ 0.08} & \textbf{0.66 $\pm$ 0.16} & 0.46 $\pm$ 0.24 \\
    \bottomrule
    \end{tabular}
    \end{adjustbox}
\end{table}

\paragraph{Hyper-parameter Tuning for CQL+HER} It is also interesting to check whether tuning the ratio $\alpha$ of the CQL loss can enable better OOD generalization. Specifically, we tune $\alpha$ in $\{0.01, 0.1, 1, 2, 5\}$. The results are shown in Table \ref{tab:point_CQL_success}. The generalization performance of CQL+HER drops as $\alpha$ becomes large, e.g., $\alpha=5$, and as $\alpha$ becomes small, e.g., $\alpha=0.01$. When the data coverage is insufficient (i.e., the ``Expert 10" setting), CQL+HER cannot generalize on the full circle of radius 20 (``R20"), no matter how $\alpha$ is adjusted. The tuned results of CQL+HER are still incomparable to weighted imitation learning methods such as GOAT and WGCSL.

\begin{figure}[h!]
    \centering
    \includegraphics[width=1\linewidth]{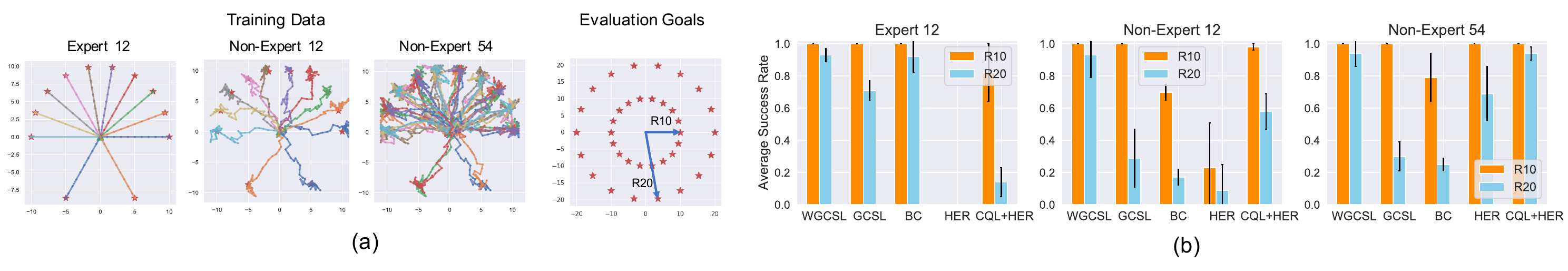}
    \vspace{-10pt}
    \caption{(a) Visualization of three 2D goal-reaching datasets, and two groups of evaluation goals, ``R10" and ``R20", with a  radius of 10 and 20, respectively. (b) Average success rates of different agents over 5 random seeds. }
    \label{fig:imbalance_data_result}
\end{figure}

\paragraph{Additional Tasks} In addition to tasks introduced before, we include another three datasets in Figure \ref{fig:imbalance_data_result} (a), where we have few trajectories on the lower semicircle. As shown in Figure \ref{fig:imbalance_data_result} (b), these tasks are relatively easy compared with those used in our didactic example, achieving higher average success rates. The conclusions are also consistent with Section \ref{sec:didactic_example}.

\begin{figure}[h!]
    \centering
\includegraphics[width=0.55\linewidth]{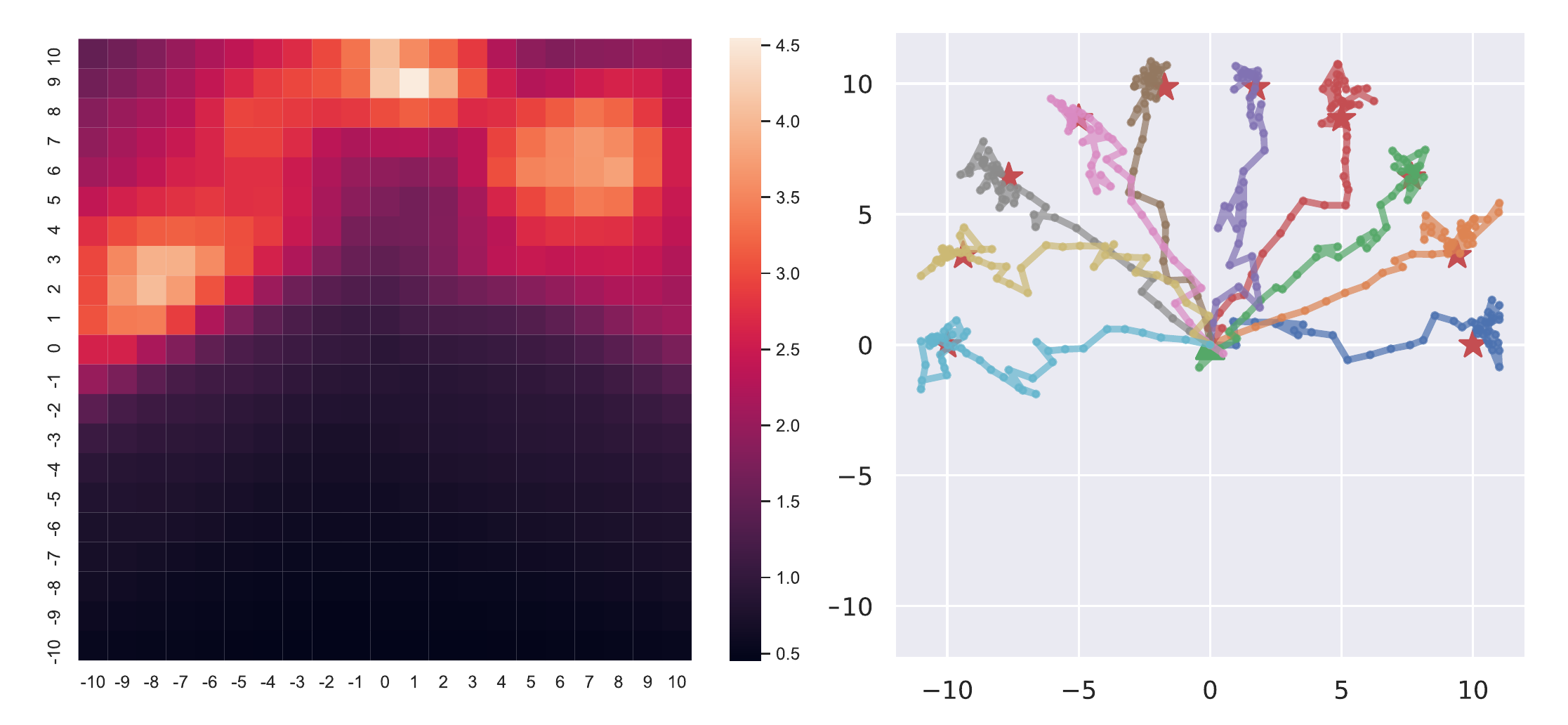}
    \caption{The reciprocal of uncertainty is an estimation of density. The uncertainty is measured by the variance of ensemble value functions for initial state $(0,0)$ and goals on $[-10,10]\times [-10,10]$.}
    \label{fig:density}
\end{figure}

\subsection{Uncertainty and Density}
To verify if the estimated uncertainty can approximate density, we visualize the value of $\frac{1}{\Std(s_0,g)}$ in Figure \ref{fig:density}, where $\Std$ is the standard deviation of a group of 5 value networks. To calculate values on the figure, we set the state $s_0$ as $(0,0)$ and use the value functions of GOAT to calculate each $\Std(s_0,g)$, where $g$ is set on a $[-10,10]\times [-10,10]$ grid with equal intervals of 1. We can observe that positions with more achieved goals (i.e., near the red stars) have larger values $\frac{1}{\Std(s_0,g)}$, which validates the relationship between uncertainty and density.

\begin{table}[ht]
    \centering
    \caption{Ablation on ensemble size for GOAT over all 26 tasks.}
    \begin{tabular}{l|ccccc}
    \toprule
        Success Rate ($\%$) & $N=2$  & $N=3$ & $N=5$ (default) & $N=7$       \\
        \midrule 
        Average IID & 90.4 &  90.4 & 90.3 & 90.6 \\
        Average OOD &   64.0 & 66.7 & 67.9 & 65.3\\
    \bottomrule
    \end{tabular}
    \label{tab:ensemble_size}
\end{table}

\subsection{Ablation on the Ensemble Size}
The ensemble size $N$ of GOAT can affect the value estimation and the uncertainty estimation. As shown in Table \ref{tab:ensemble_size}, though the IID performance is similar, $N=5$ works better than $N\in \{2,3,7\}$ on the average OOD success rate.
Note that with different $N$, the average OOD success rate is still better than that of WGCSL (i.e., 62.1 for OOD tasks).

\begin{table*}[h]
\caption{Additional ablations of GOAT. Average success rates ($\%$) with standard deviation over 5 random seeds.}
\label{tab:success_rate_addition_ablation}
\begin{center}
\begin{small}
 \begin{adjustbox}{max width=0.52\linewidth}
\begin{tabular}{llccccccccccc}
\toprule
Task Group & Task &  GOAT & GOAT(V) & GOAT(V+$\chi^2$)  \\
\midrule
\rowcolor{c1!10} \cellcolor{white}  &  Right  &  100.0$\pm$0.0  &   99.9 $\pm$ 0.2  &  100.0 $\pm$ 0.0   \\
\rowcolor{blue!10} \cellcolor{white} Reach Left-Right &  Left &  99.0$\pm$2.0  &  94.5 $\pm$ 7.5  &   96.6 $\pm$ 2.5   \\
  &  Average  & \textbf{99.5}  &  97.2  &  98.3  \\

\midrule
\rowcolor{c1!10} \cellcolor{white}  &  Near &  100.0$\pm$0.0  &  99.6 $\pm$ 0.8  &  100.0 $\pm$ 0.0 \\
\rowcolor{blue!10} \cellcolor{white} Reach Near-Far  & Far  &  97.6$\pm$1.1  &  89.8 $\pm$ 4.5  &  90.6 $\pm$ 1.0  \\
  &  Average  &  \textbf{98.8}  &  94.7 & 95.3  \\

\midrule
\rowcolor{c1!10} \cellcolor{white}  & Right2Right &    95.9$\pm$1.2  &  92.6 $\pm$ 3.4  &  93.9 $\pm$ 1.6   \\
\rowcolor{blue!10} \cellcolor{white} &  Right2Left  & 69.3$\pm$6.6  &  48.9 $\pm$ 5.8  &   52.0 $\pm$ 5.8  \\
\rowcolor{blue!10} \cellcolor{white} Push Left-Right &  Left2Right   &  76.0$\pm$7.4  &  56.2 $\pm$ 5.1  &  56.6 $\pm$ 11.2  \\
\rowcolor{blue!10} \cellcolor{white} &  Left2Left   &   61.1$\pm$7.6  &  39.5 $\pm$ 4.0  &  34.1 $\pm$ 5.1  \\
& Average  &    \textbf{75.6}  &  59.3  &  59.2   \\

\midrule
\rowcolor{c1!10} \cellcolor{white} & Near2Near  &  92.0$\pm$2.6  &   92.2 $\pm$ 2.2  &  87.5 $\pm$ 1.7  \\
\rowcolor{blue!10} \cellcolor{white} & Near2Far   &  70.3$\pm$5.7   &  59.4 $\pm$ 5.1  &  63.8 $\pm$ 3.9   \\
\rowcolor{blue!10} \cellcolor{white} Push Near-Far & Far2Near  &   69.5$\pm$3.6  &  65.0 $\pm$ 3.9  &  63.4 $\pm$ 3.7   \\
\rowcolor{blue!10} \cellcolor{white} & Far2Far  &  50.8$\pm$1.8   &  41.3 $\pm$ 2.2  &  43.5 $\pm$ 4.3   \\
& Average  &  \textbf{70.6}  &  64.5  &  64.6   \\

\midrule
\rowcolor{c1!10} \cellcolor{white} & Right2Right  &  97.3$\pm$1.2  &  82.1 $\pm$ 3.2 &    82.5 $\pm$ 8.3  \\
\rowcolor{blue!10} \cellcolor{white} & Right2Left  &   88.6$\pm$1.1  &  52.7 $\pm$ 9.2  &  51.4 $\pm$ 8.2    \\
\rowcolor{blue!10} \cellcolor{white} Pick Left-Right & Left2Right   &  93.9$\pm$1.9  &  62.4 $\pm$ 5.2  &  64.1 $\pm$ 15.7   \\
\rowcolor{blue!10} \cellcolor{white} & Left2Left   &  88.3$\pm$3.7  &  43.8 $\pm$ 13.4  &  42.6 $\pm$ 9.8  \\
& Average   &  \textbf{92.0}  &  60.3  &  60.2    \\

\midrule
\rowcolor{c1!10} \cellcolor{white}  & Low     &  99.8$\pm$0.2  &  99.1 $\pm$ 0.9  &  98.9 $\pm$ 1.0  \\
\rowcolor{blue!10} \cellcolor{white}  Pick Low-High & High  &  71.9$\pm$6.4  &  17.5 $\pm$ 2.5  &  24.8 $\pm$ 4.9   \\
& Average & \textbf{85.8}  &  58.3  &  61.8  \\

\midrule
\rowcolor{c1!10} \cellcolor{white}  & Right2Right & 79.0$\pm$5.8  &  77.5 $\pm$ 6.7  &  71.8 $\pm$ 5.0  \\
\rowcolor{blue!10} \cellcolor{white}  & Right2Left   &  41.3$\pm$7.1  &  45.7 $\pm$ 6.3  &  40.5 $\pm$ 7.8   \\
\rowcolor{blue!10} \cellcolor{white}  Slide Left-Right & Left2Right  &  59.0$\pm$7.6  &  52.2 $\pm$ 10.9  &  31.1 $\pm$ 7.8   \\
\rowcolor{blue!10} \cellcolor{white}  & Left2Left    &  50.1$\pm$9.5  &  41.9 $\pm$ 8.7  &  30.9 $\pm$ 4.9   \\
& Average  &  \textbf{57.4}  &  54.3  &  43.6 \\

\midrule
\rowcolor{c1!10} \cellcolor{white}  & Near    & 76.9$\pm$3.3  &  73.2 $\pm$ 4.8  &  72.7 $\pm$ 9.0   \\
\rowcolor{blue!10} \cellcolor{white}  Slide Near-Far & Far     & 29.0$\pm$4.5  &  26.4 $\pm$ 4.1  &  32.0 $\pm$ 4.4   \\
& Average  & \textbf{53.0}  &  49.8  &  52.4 \\

\midrule
\rowcolor{c1!10} \cellcolor{white}  & Near  &  71.9$\pm$3.2  &  74.4 $\pm$ 2.9  &  76.8 $\pm$ 2.2    \\
\rowcolor{blue!10} \cellcolor{white} HandReach Near-Far & Far &   38.4$\pm$4.1  &  37.3 $\pm$ 6.5  &  43.1 $\pm$ 4.2  \\
& Average   &  55.2  &  55.9 &  \textbf{60.0}  \\

\midrule
\multirow{2}{*}{Average} &  IID Tasks  & \textbf{90.3} &  87.8 & 87.1  \\
  & OOD Tasks  &  \textbf{67.9}  & 51.4  &  50.7  \\
\bottomrule
\end{tabular}
\end{adjustbox}
\end{small}
\end{center}
\end{table*}

\subsection{Additional Ablations of GOAT}
\label{ap:additional_ablations}
Taking into account additional design considerations, we also revisit two design choices, namely the value function and the exponential weight. Specifically, in GOAT, we adopt the approach proposed by WGCSL \cite{yang2022rethinking} to learn Q functions for estimating the advantage value, which is given by $A(s_t,a_t,g)=r(s_t,a_t,g)+\gamma Q(s_{t+1}, \pi(s_{t+1},g) ,g) - Q(s_{t}, \pi(s_{t},g) ,g)$. The difference is that Q values are averaged by an ensemble of Q functions. Alternatively, we can estimate the advantage values by learning V functions $V(s,g)$, which can be expressed as $A(s_t,a_t,g)=r(s_t,a_t,g)+\gamma V(s_{t+1},g) - V(s_t,g)$. Note that the learned action $\pi$ is not needed in the value function learning for $V(s,g)$. We use the notation ``GOAT(V)" to describe this variant of GOAT with V functions. Furthermore, the exponential advantage weight is a special case of maximizing expected value and minimizing the $f$-divergence. Following \cite{ma2022far}, we also consider replacing the exponential weight with $\max (A(s,a,g)+c_{\chi^2}, 0)$, which is derived by considering the $\chi^2$-divergence. We denote GOAT with both V function and $\chi^2$-divergence as ``GOAT(V+$\chi^2$)". After the parameter search, we find $c_{\chi^2}=0$ performs well for GOAT(V+$\chi^2$). 

The final results are reported in Table \ref{tab:success_rate_addition_ablation}. Our results indicate that, on average across the 17 out-of-distribution tasks, the use of a V function in GOAT significantly reduces the OOD generalization performance, with negligible impact from the $\chi^2$-divergence. It is worth noting, however, that for the high-dimensional HandReach task, the best performance is achieved by incorporating both a V function and $\chi^2$-divergence into GOAT. The rationale behind this finding lies in the high-dimensional nature of the state-goal and action spaces in the HandReach task. As a result of this high dimensionality, the multimodal problem is more pronounced, rendering the learning of a V function useful for achieving better stability than the Q function. Moreover, the integration of a weighting function induced by the $\chi^2$-divergence serves to eliminate inferior data and also alleviate the multimodal problems arising in such high-dimensional spaces. Therefore, the choice of value function and weighting function also depends on the task characteristics.

\begin{figure}[ht] 
 \begin{minipage}[c]{0.48\textwidth} 
    \centering
    \caption{Ablation of Expectile Regression}
    \begin{adjustbox}{max width=\linewidth}
    \begin{tabular}{c|cccc}
    \toprule
      Success Rate (\%)   & WGCSL & WGCSL+ER & GOAT& GOAT+ ER \\
      \midrule
      Average  & 71.1 & 73.6  & 75.7  &  77.9 \\
      Average OOD & 62.1 &  65.1   & 67.9 &   70.9 \\
    \bottomrule
    \end{tabular}
    \end{adjustbox}
    \label{tab:iql}
  \end{minipage} 
  \hspace{0.7cm}
  \begin{minipage}[c]{0.43\textwidth} 
    \centering 
\includegraphics[width=1\textwidth]{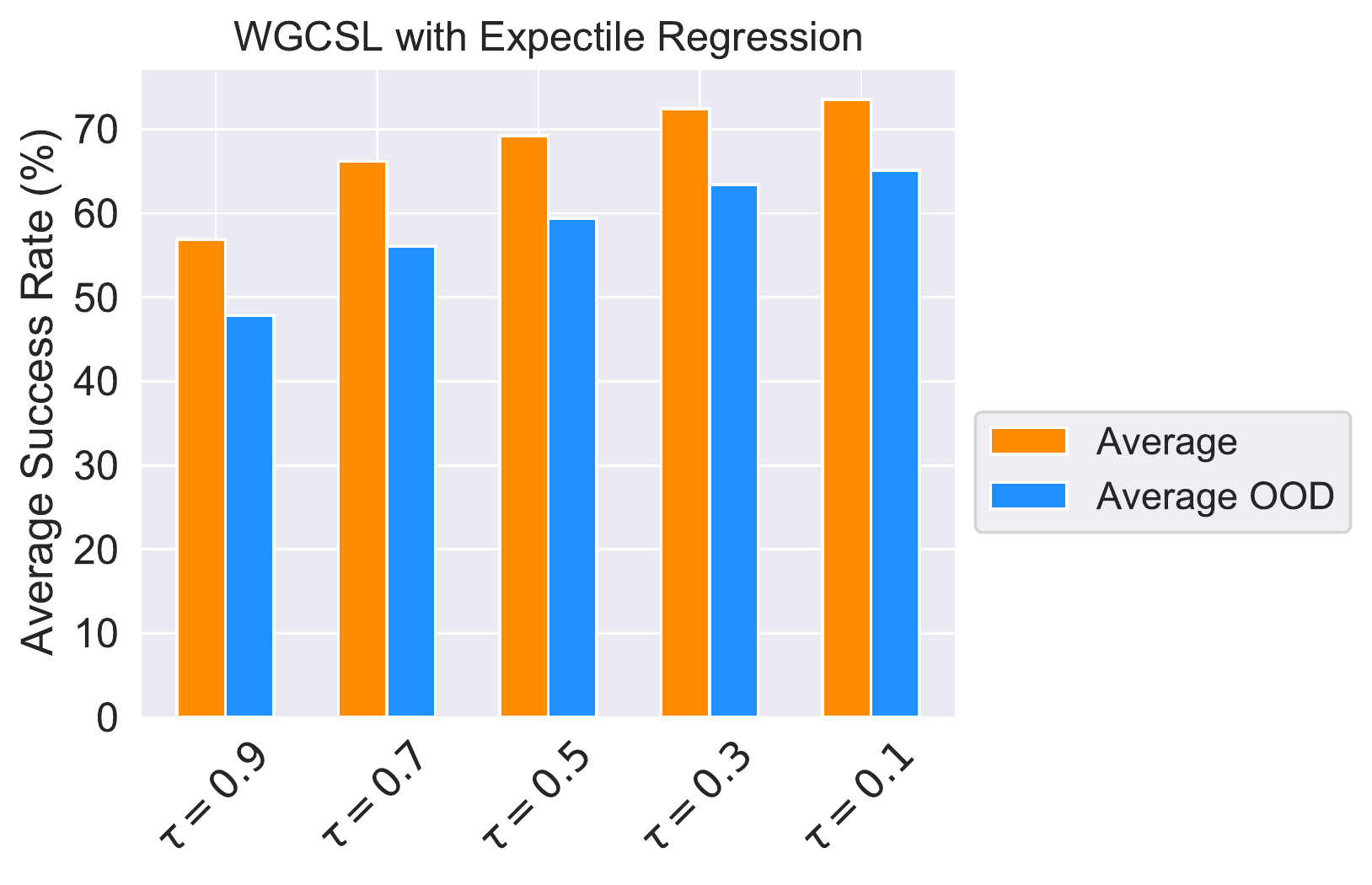} 
    \caption{Comparing differnt $\tau$ in Expectile Regression for Weighted Imitation Learning.} 
    \label{fig:IQL_tau} 
  \end{minipage}%
\end{figure}

\subsection{The Effectiveness of Expectile Regression}
\label{ap:expectile_regression}
\citet{kostrikov2021offline} proposed IQL to combine expectile regression with weighted imitation learning in offline RL. The difference is that we do not learn additional $V$ function to avoid OOD actions when learning value functions and we validate its effectiveness for improving advantage value estimation of offline GCRL. As shown in Figure \ref{tab:iql}, expectile regression (ER) improves both WGCSL and GOAT by around 3 points on average OOD success rates. Besides, in Figure \ref{fig:IQL_tau} we show that the performance of WGCSL+ER improves with the decrease of $\tau$. We conjecture that the smaller $\tau$ is, the more accurate the relative relationship of the advantage values are, and thus the better the estimation of the expert policy for imitation.

\begin{figure*}[h!]
    \centering
    \subfigure[]{\includegraphics[width=0.3\linewidth,height=120pt]{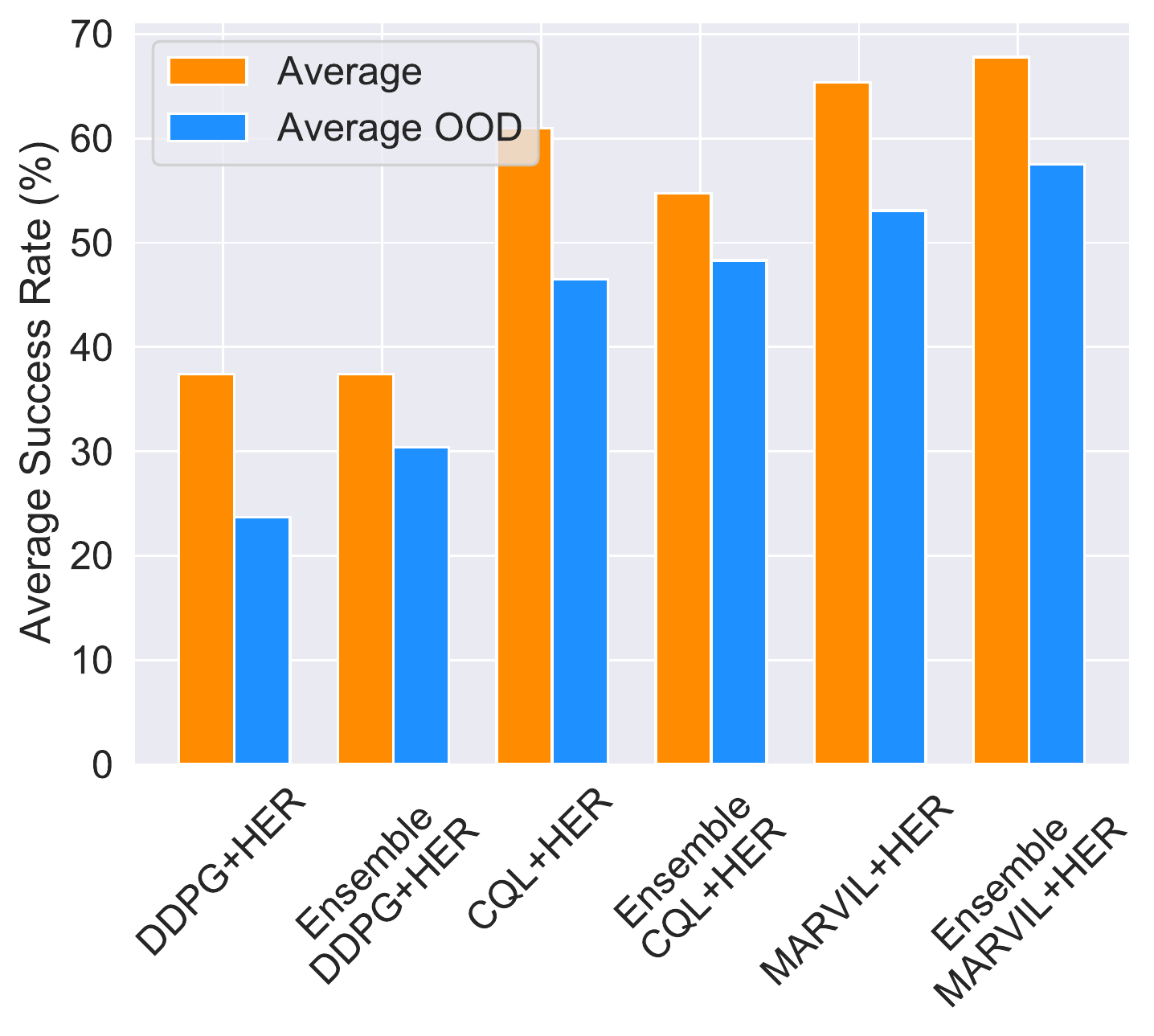}\label{fig:ensemble_compare}}
    \subfigure[]{\includegraphics[width=0.3\linewidth, height=120pt]{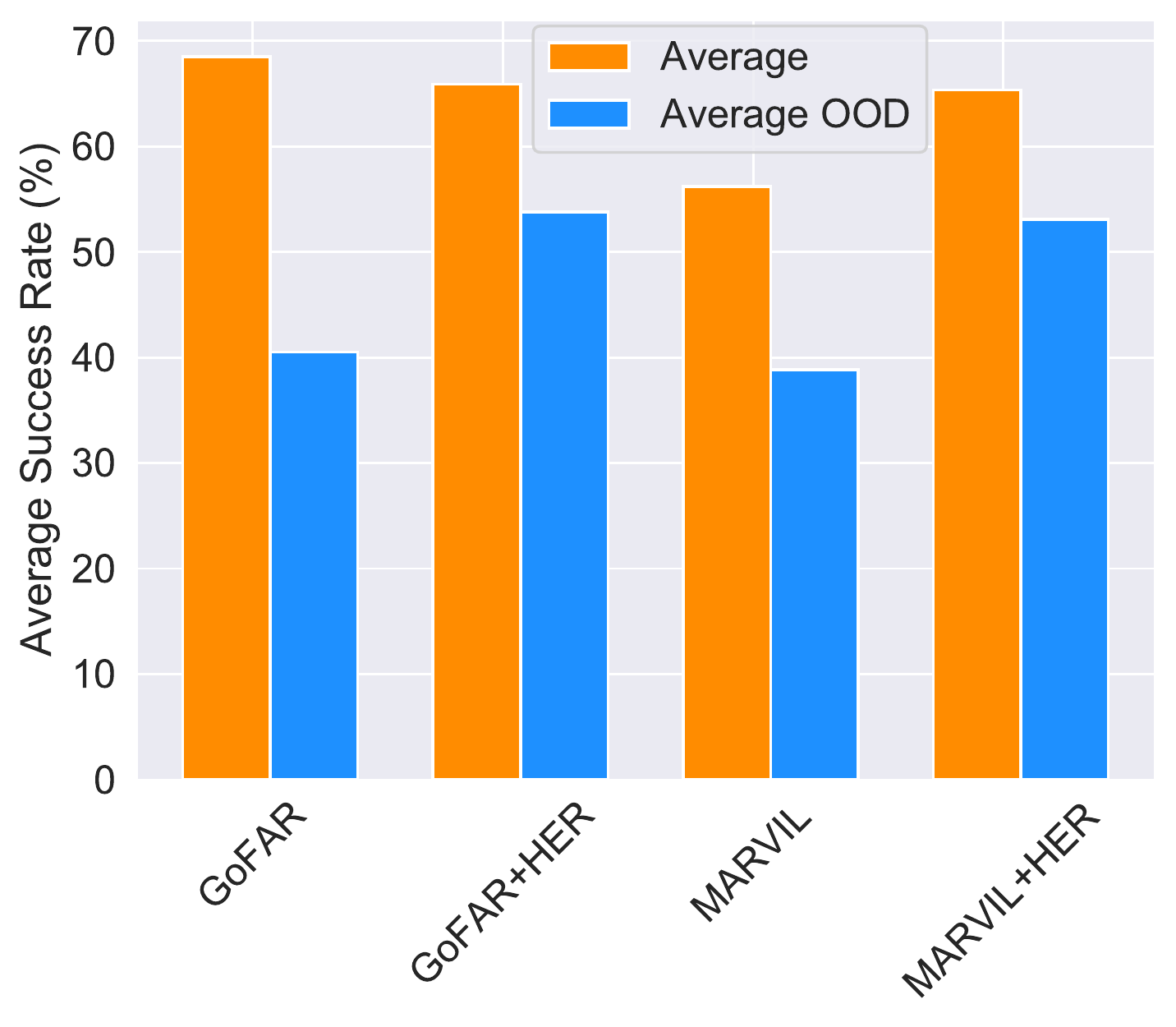}\label{fig:gofar_her}}
    \subfigure[]{\includegraphics[height=120pt]{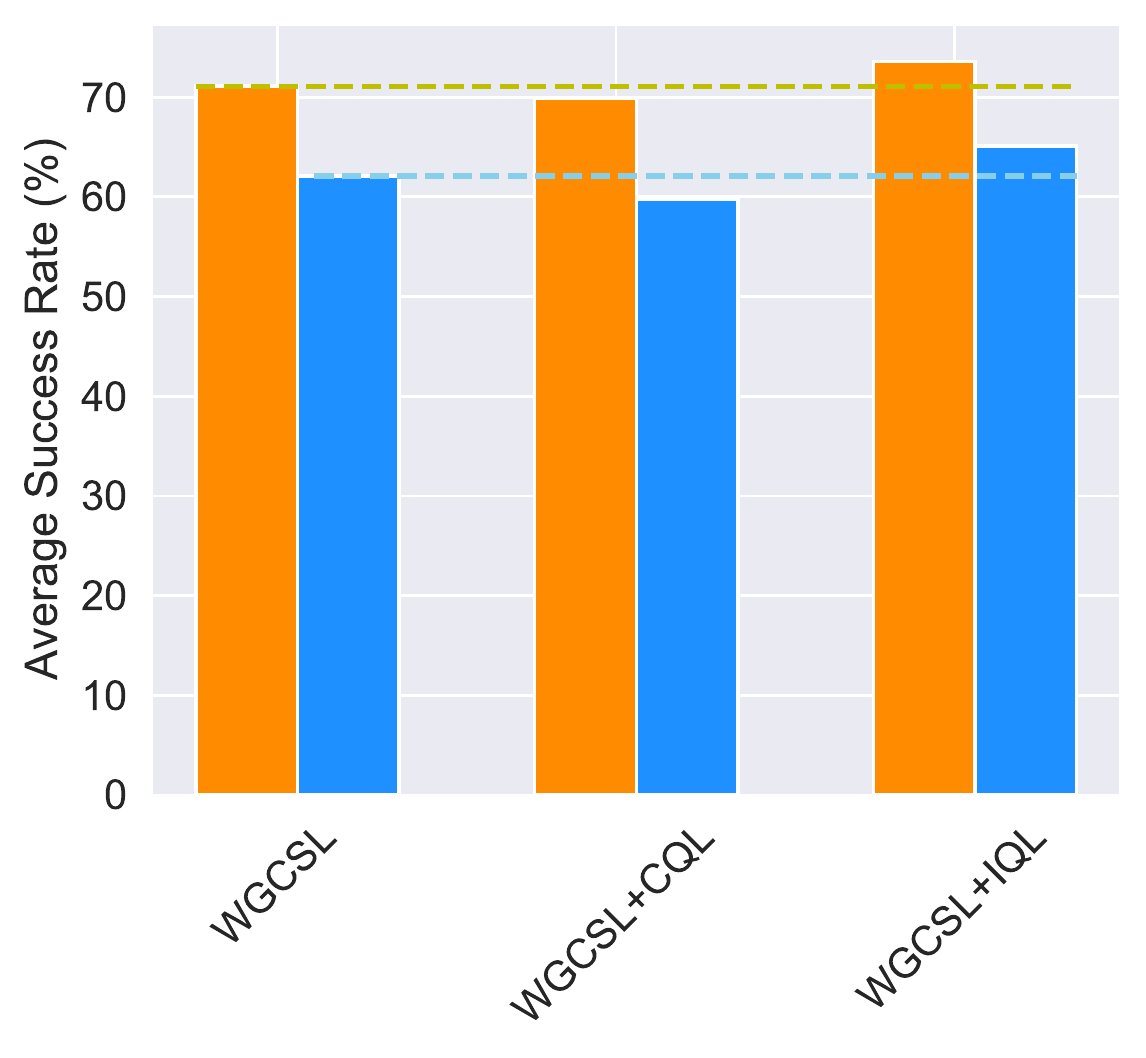}\label{fig:wgcsl+cql}}
    \vspace{-0.3cm}
    \caption{(a) Comparison between methods with and without value function ensemble. (b) Comparison of GoFAR and MARVIL with and without HER. (c) Comparison between WGCSL and WGCSL with value underestimation.}
    \label{fig:ablations_more}
    \vspace{-0.3cm}
\end{figure*}

\subsection{Ablations of Ensemble and HER for Other GCRL Algorithms}
In Figure \ref{fig:ensemble_compare} and Figure \ref{fig:gofar_her}, we demonstrate that ensemble value functions and HER can improve the performance of different RL algorithms, indicating that they are generally useful techniques for OOD generalization of offline GCRL.

\begin{table*}[h!]
\caption{Ablations of GoFAR on the HandReach task. Average success rates ($\%$) with standard deviation over 5 random seeds.}
\label{tab:success_rate_handreach_gofar}
\begin{center}
\begin{small}
 \begin{adjustbox}{max width=0.95\linewidth}
\begin{tabular}{llccccccccccc}
\toprule
Task Group & Task &  GoFAR  & GoFAR(binary) & GoFAR(binary+Q) &  GoFAR(binary+exp) & GoFAR(binary+relabel) &  GoFAR(binary+exp+relabel) \\
\midrule
\rowcolor{c1!10} \cellcolor{white}  & Near   & 77.4$\pm$1.7 &  78.9$\pm$3.3  &  4.9$\pm$5.0   &  10.8$\pm$2.7  &  62.0$\pm$7.9 & 57.1$\pm$7.4  \\
\rowcolor{blue!10} \cellcolor{white} HandReach Near-Far & Far &  36.9$\pm$3.1 &  39.2$\pm$4.3   & 3.1$\pm$1.8  &  2.6$\pm$1.4  &  29.3$\pm$8.2  & 24.0$\pm$4.2  \\
& Average  &  57.1 &  59.0  &  4.0  & 6.7  &  45.7   & 40.5   \\
\bottomrule
\end{tabular}
\end{adjustbox}
\end{small}
\end{center}
\end{table*}

\subsection{Ablations of GoFAR}
GoFAR \cite{ma2022far} is a recent offline GCRL method which is also based on weighted imitation learning, but we observe its performance drops significantly on OOD tasks. We conjecture the primary reason is that GoFAR does not use goal relabeling. To validate our conjecture, we compare GoFAR and GoFAR+HER in Figure \ref{fig:gofar_her}. The results show that HER does not improve the overall performance of GoFAR but increases the average OOD success rates by a large margin. In Figure \ref{fig:gofar_her}, the performance increment of GoFAR+HER over GoFAR also matches that of MARVIL+HER over MARVIL.

In our benchmark experiments, we observe that GoFAR has an advantage on the high-dimensional HandReach task over WGCSL and GOAT. Given this observation, we investigate and analyze the roles played by the key techniques of GoFAR's design, namely the discriminator-based rewards, the advantage estimation, and the weighting function. We subsequently compare the following variants:
\begin{itemize}
    \item GoFAR: it employs discriminator-based rewards, learns $V(s,g)$ function to estimate the advantage value $A(s_t,a_t,g)=r(s_t,a_t,g)+\gamma V(s_{t+1},g) - V(s_t,g)$, weights the imitation loss with $\max (A(s_t,a_t,g)+1, 0)$, and does not use goal relabeling;
    \item GoFAR(binary): it replaces the discriminator-based rewards with binary rewards $r(s_t,a_t,g)=1[\|\phi(s_t)-g\|_2^2 \leq \epsilon]$;
    \item GoFAR(binary+Q): it uses binary rewards to learn $Q(s,a,g)$ functions instead of $V$ functions, then the advantage value are estimated by $A(s_t,a_t,g)=r(s_t,a_t,g)+\gamma Q(s_{t+1}, \pi(s_{t+1},g) ,g) - Q(s_{t}, \pi(s_{t},g) ,g)$;
    \item GoFAR(binary+exp): it also employs binary rewards to learn $V$ functions, but includes a weighting term of $\exp(A(s,a,g))$ based on the KL divergence;
    \item GoFAR(binary+relabel): it uses binary rewards and goal relabeling for GoFAR;
    \item GoFAR(binary+exp+relabel): it utilizes binary rewards and goal relabeling to learn $V$ functions, and the exponential weighting function $\exp(A(s,a,g))$ for weighted imitation learning.
\end{itemize}
Table \ref{tab:success_rate_handreach_gofar} presents the findings of this study, which suggest that the value function and weighting function are the most crucial components of GoFAR for the high-dimensional HandReach task. ``GoFAR(binary+Q)" and ``GoFAR(binary+exp)" both fail on this task. We posit that this maybe due to the following reasons: (1) with high-dimensional state-goal and action spaces, the Q value function has more dimensions as input than the V function and is therefore less stable and harder to train. Additionally, the learned policy $\pi$ via weighted imitation learning is prone to interpolation into out-of-distribution actions, which causes imprecise advantage value estimation using $Q(s, \pi(s,g),g)$. (2) The weighting function $\max (A(s,a,g)+1, 0)$ also serves the purpose of clearing poor quality data from our weighted imitation learning, whereas the exponential weighting function is more sensitive to the multimodal problem. Furthermore, our results indicate that the discriminator-based rewards play no significant role and may even decrease performance when compared to binary rewards. It is also worth noting that the effectiveness of goal relabeling varies with the type of weighting function. While it diminishes performance for $\max (A(s,a,g)+1, 0)$, it enhances the performance of weighting with $\exp(A(s,a,g))$.

\subsection{Combining Weighted Imitation Learning with Value Underestimation} 
The value function learning has an impact on the estimation of the expert policy for imitation. It is interesting to see whether simply underestimating values for weighted imitation learning is also helpful for OOD generalization. In Figure \ref{fig:wgcsl+cql}, we consider two methods to be on top of WGCSL, CQL \cite{kumar2020conservative}, and IQL \cite{kostrikov2021offline} (specifically, the expectile regression technique). We demonstrate that while CQL is not helpful for WGCSL, while the expectile regression in IQL is a good choice for better value function estimation of offline GCRL.

\subsection{Discounted Relabeling Weight (DRW)} 
\label{ap:DRW}
\citet{yang2022rethinking} introduced the Discounted Relabeling Weight (DRW) for offline GCRL. For a relabeled transition $(s_t,a_t,\phi(s_{i})), i\geq t$, DRW is defined as $\gamma^{i-t}$, which has an effect of optimizing a tighter lower bound for offline GCRL. Intuitively, DRW assigns relatively larger weights on closer relabeling goals with smaller $i$. In Table \ref{tab:gamma}, we find DRW can slightly improve the IID performance but reduce the average OOD performance of WGCSL. It is reasonable because DRW assigns larger weights for closer goals, which contradicts the Uncertainty Weight as discussed in Section \ref{sec:understand_UW} and has the risk of overfitting simpler goals.

\begin{table}[h]
    \centering
    \caption{Comparison of WGCSL and WGCSL+DRW on 26 tasks.}
    \begin{adjustbox}{max width=\linewidth}
    \begin{tabular}{c|cccc}
    \toprule
      Success Rate (\%)   & WGCSL & WGCSL+DRW\\
      \midrule
      Average IID  & 88.1 & 88.4 \\
      Average OOD & 62.1 &  59.5  \\
    \bottomrule
    \end{tabular}
    \end{adjustbox}
    \label{tab:gamma}
\end{table}

\subsection{Adaptive Data Selection Weight} 
\label{ap:adap_DSW}
\citet{yang2022rethinking} introduced the Data Selection Weight (DSW) to tackle the multi-modal problem in multi-goal datasets, which also improves the OOD generalization performance through narrowing the expert estimation gap. However, the introduced approach utilizes a global threshold for all $(s,g)$ pairs. \textbf{Is it helpful to include an adaptive threshold function for different $(s,g)$ pairs?} This can be done using expectile regression for advantage values, i.e., learning a function $f(s,g)$ to estimate the $\beta$ expectile value of the distribution of $A(s,g,a)$.
\begin{equation*}
    \mathcal{L}_{f} = \mathbb{E}_{(s,a,g) \sim D_{relabel}} [L_2^{\beta}(A(s,a,g) - f(s,g))]
\end{equation*}
$\beta$ is the hyper-parameter similar to $\alpha$ in the original DSW controlling the quality of data used for weighted imitation learning. Finally, the adaptive data selection weight is $\epsilon(A(s,a,g))=1[A(s,a,g)\geq f(s,g)]$. We compare WGCSL and WGCSL with adaptive data selection weight (ADSW) on Push and Pick task groups. As shown in Table \ref{tab:success_rate_adaptive}, WGCSL with ADSW only brings slight improvement over global threshold. For most tasks, WGCSL also obtains top two scores with a simpler data selection method. Considering the extra computation of learning the threshold function $f$, we do not include it in GOAT. But we believe it is a good start for future research on how to achieve more efficient adaptive data selection for offline goal-conditioned RL.

\begin{table*}[h]
\caption{Comparison with Adaptive Data Selection Weight (ADSW). Average success rates ($\%$) with standard deviation over 5 random seeds. Top two scores for each task are highlighted.}
\label{tab:success_rate_adaptive}
\begin{center}
\begin{small}
 \begin{adjustbox}{max width=0.7\linewidth}
\begin{tabular}{llccccc}
\toprule
Task Group & Task &  WGCSL  & ADSW, $\beta=0.8$ & ADSW, $\beta=0.9$ & ADSW, $\beta=0.95$ \\

\midrule
\rowcolor{c1!10} \cellcolor{white}  & Right2Right &  93.2$\pm$0.9   &    95.1$\pm$1.4  & 95.5$\pm$0.8   &  92.8$\pm$1.9   \\
\rowcolor{blue!10} \cellcolor{white} &  Right2Left  & 63.3$\pm$8.9 &    65.3$\pm$4.7 &  69.5$\pm$6.4   &   62.5$\pm$7.4  \\
\rowcolor{blue!10} \cellcolor{white} Push Left-Right &  Left2Right & 67.6$\pm$7.1  &  74.9$\pm$8.6   &  74.6$\pm$6.0  &  64.3$\pm$13.1  \\
\rowcolor{blue!10} \cellcolor{white} &  Left2Left   & 47.7$\pm$7.4 &    59.8$\pm$3.7 &   66.4$\pm$5.9 &  52.1$\pm$6.9  \\
& Average  &   68.0 &  \textbf{73.8}   &  \textbf{76.5}  &   67.9 \\

\midrule
\rowcolor{c1!10} \cellcolor{white} & Near2Near  &  93.5$\pm$1.0    &    90.7$\pm$2.8 &   93.4$\pm$1.0   &  91.9$\pm$1.2  \\
\rowcolor{blue!10} \cellcolor{white} & Near2Far & 67.0$\pm$5.4    &   68.1$\pm$4.2  &   69.4$\pm$3.8  & 63.4$\pm$5.0 \\
\rowcolor{blue!10} \cellcolor{white} Push Near-Far & Far2Near  &  68.0$\pm$2.4  &  66.1$\pm$2.6   &  67.6$\pm$2.3   & 62.1$\pm$4.0  \\
\rowcolor{blue!10} \cellcolor{white} & Far2Far  & 51.1$\pm$4.7  &  47.4$\pm$3.4   &   53.1$\pm$1.9 & 40.2$\pm$6.1    \\
& Average  & \textbf{69.9}  &  68.1   & \textbf{70.9}   &  64.4   \\

\midrule
\rowcolor{c1!10} \cellcolor{white} & Right2Right  & 93.8$\pm$5.3   & 96.6$\pm$2.6   &  94.4$\pm$4.1   &   96.0$\pm$2.0    \\
\rowcolor{blue!10} \cellcolor{white} & Right2Left  &  89.4$\pm$3.9   & 83.5$\pm$6.4  &  71.1$\pm$9.4    &  78.9$\pm$11.2  \\
\rowcolor{blue!10} \cellcolor{white} Pick Left-Right & Left2Right  & 90.0$\pm$4.1    &  89.8$\pm$4.8  &  91.6$\pm$4.7  & 89.3$\pm$2.8   \\
\rowcolor{blue!10} \cellcolor{white} & Left2Left & 87.0$\pm$5.1   &    83.0$\pm$5.7 & 73.3$\pm$8.6  &  78.6$\pm$8.4   \\
& Average  & \textbf{90.0}   &  \textbf{88.2}  &  82.6  &   85.7   \\

\midrule
\rowcolor{c1!10} \cellcolor{white}  & Low  & 98.6$\pm$1.3     &   99.8$\pm$0.2  & 99.1$\pm$0.6  & 99.4$\pm$0.6  \\
\rowcolor{blue!10} \cellcolor{white}  Pick Low-High & High  & 66.6$\pm$6.6    & 63.4$\pm$5.1  & 66.8$\pm$13.5   &  63.5$\pm$8.6 \\
& Average  & \textbf{82.6} &  81.6  &   \textbf{83.0}  &    81.4  \\

\midrule
\multirow{2}{*}{Average} &  IID Tasks & 94.8 &  \textbf{95.6} &  \textbf{95.6} & 95.0 \\
  & OOD Tasks  & 69.8 & \textbf{70.1}  & \textbf{70.3}  &  65.5  \\
\bottomrule
\end{tabular}
\end{adjustbox}
\end{small}
\end{center}
\end{table*}

\begin{table}[htb]
    \centering
    \caption{Varying hyper-parameter of MSG over all 26 tasks.}
    \begin{tabular}{l|ccccc}
    \toprule
        Success Rate ($\%$) & $c=1$  & $c=3$ & $c=5$  & $c=7$       \\
        \midrule 
        Average IID &  60.1 &  65.4 &  68.5 & 69.9 \\
        Average OOD &  41.0 & 42.4  &   43.1 & 39.4 \\
    \bottomrule
    \end{tabular}
    \label{tab:msg_parameter}
\end{table}

\subsection{MSG+HER with Varying Hyper-parameter}
MSG \cite{ghasemipour2022so} is a recent SOTA ensemble-based offline RL method, which learns a group of independent $Q$ networks and estimates a Lower Confidence Bound (LCB) objective with the standard deviation (std) of $Q$ networks. One important hyper-parameter for MSG+HER is the weight parameter $c$ for the std (see Appendix \ref{ap:implementation_details}). We compare the performance of MSG+HER with varying $c\in \{1,3,5,7\}$ in Table \ref{tab:msg_parameter}. The results demonstrate that $c=5$ achieves the best OOD generalization results. When $c$ is larger than $5$ (i.e., $c=7$), the agent is too conservative to generalize, leading to improvement on IID tasks but decrease on OOD tasks. GOAT still outperforms MSG+HER by a large margin, which also supports that pessimism-based offline RL method can inhibit generalization.

\begin{figure*}[htb]
    \centering
    \includegraphics[width=0.75\linewidth]{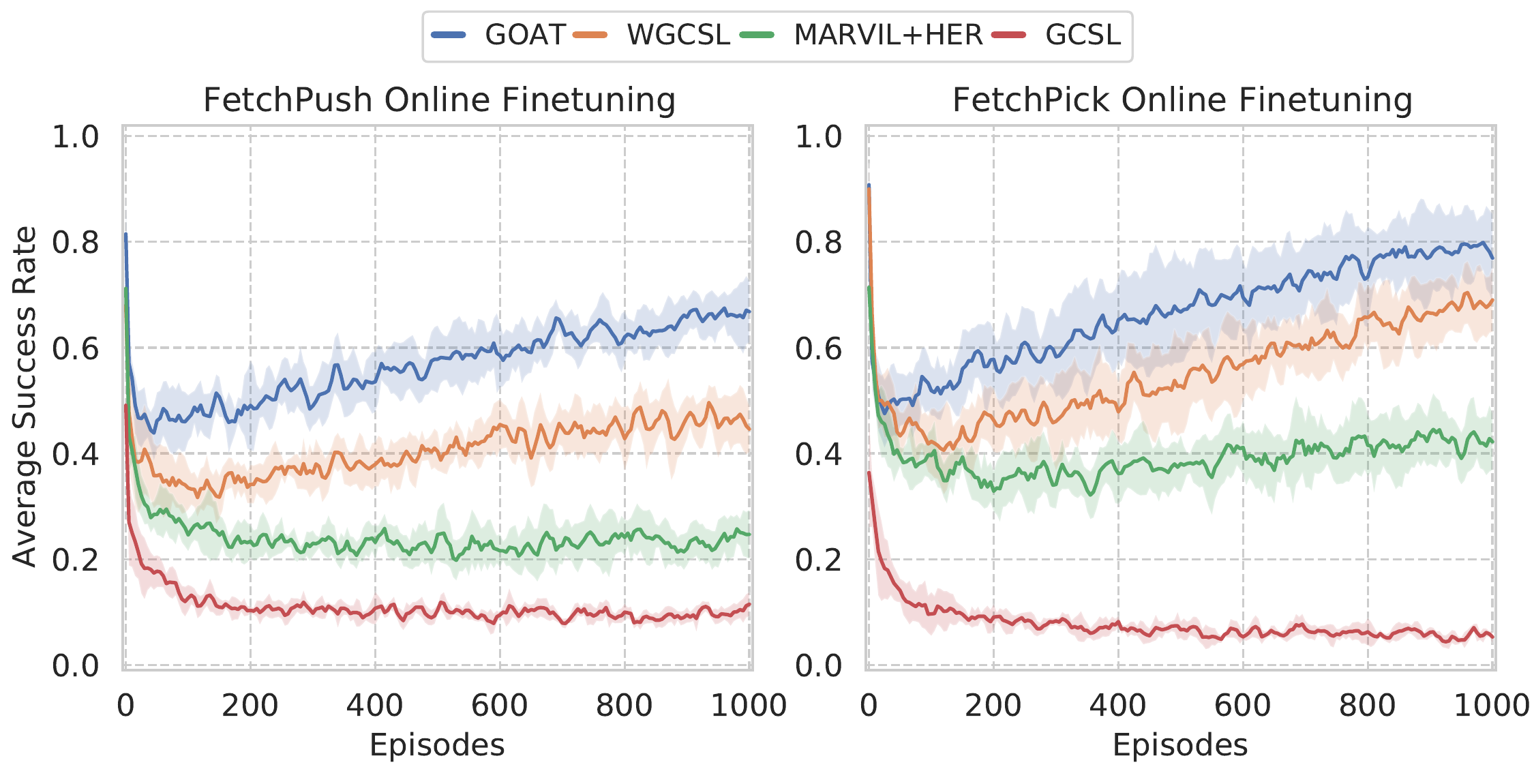}
    \caption{Online fine-tuning of different supervised learning methods in FetchPush Left-Right and FetchPick Left-Right tasks.}
    \label{fig:online_finetune_supervised}
    \vspace{-10pt}
\end{figure*}

\subsection{Online Fine-tuning}
\label{ap:online_finetune}
To understand the effect of the generalization ability of pre-trained agents for online learning, we design an experiment to fine-tune pre-trained agents with online samples. The pre-trained agents are trained on offline datasets with partial coverage (e.g., Right2Right) and evaluated with full coverage goals (Right2Right, Right2Left, Left2Right, Left2Left). In the fine-tuning period, agents explore with additional Gaussian noise (zero mean and 0.2 standard deviation) and random actions (with a probability of 0.3). For all pre-trained agents, we fine-tune the policy and value function for 10 (FetchPick) or 20 (FetchPush) batches after every a trajectory collected. The training batch size is 512, the learning rate is $5\times 10^{-4}$, and the optimizer is Adam. Note that we do not use offline datasets during online fine-tuning, and we fine-tune agents to goals not seen in the pre-training phase, which is different from prior offline-to-online setting \cite{nair2020awac}.

For the first experiment in Figure \ref{fig:online_finetune}, all pre-trained agents are fine-tuned with DDPG+HER \cite{andrychowicz2017hindsight}, which is a general baseline in the online setting. In addition to DDPG+HER, we apply different supervised learning methods, i.e, GOAT, WGCSL, MARVIL+HER, GCSL for online fine-tuning in Figure \ref{fig:online_finetune_supervised}. The fine-tuning algorithms are the same as their pre-training algorithms. Other settings are kept the same as the above fine-tuning experiments. Comparing Figure \ref{fig:online_finetune_supervised} with Figure \ref{fig:online_finetune}, we can also conclude that (1) off-policy method (i.e., DDPG+HER) is more efficient than supervised methods for online fine-tuning, (2) GOAT substantially outperforms other supervised methods such as WGCSL and MARVIL+HER when fine-tuned using their respective pre-training algorithms.

\subsection{Training Time}
We consider the training time as a measure of computational cost in Figure \ref{fig:comparison_time}. For our experiments, we use one single GPU (NVIDIA GeForce RTX 2080 Ti 11 GB) and one cpu core (Intel Xeon W-2245 CPU @ 3.90GHz). Among all the algorithms, BC and GCSL require the least training time due to their simplicity, but they suffer to generalize given non-expert datasets. WGCSL, MARVIL+HER and DDPG+HER need more training time because they are equipped with additional $Q$ networks. Besides, MSG leverages an ensemble of $Q$ networks and averages their gradients to the policy, leading to the longest training time. CQL+HER requires the second longest training time because of the OOD action sampling and the logsumexp approximation procedures. 
Though GoFAR is also a weighted imitation method similar to WGCSL, it is the second slowest method because it uses additional discriminator for reward estimation and it is implemented based on Torch. GOAT introduces ensemble networks, expectile regression and uncertainty estimation on top of WGCSL, thus increasing the computational cost. Thanks to the efficient implementation based on tensorflow, GOAT is still more efficient than CQL+HER, and GoFAR, and requires affordable computational cost.

\begin{figure}
    \centering
    \includegraphics[width=0.45\linewidth]{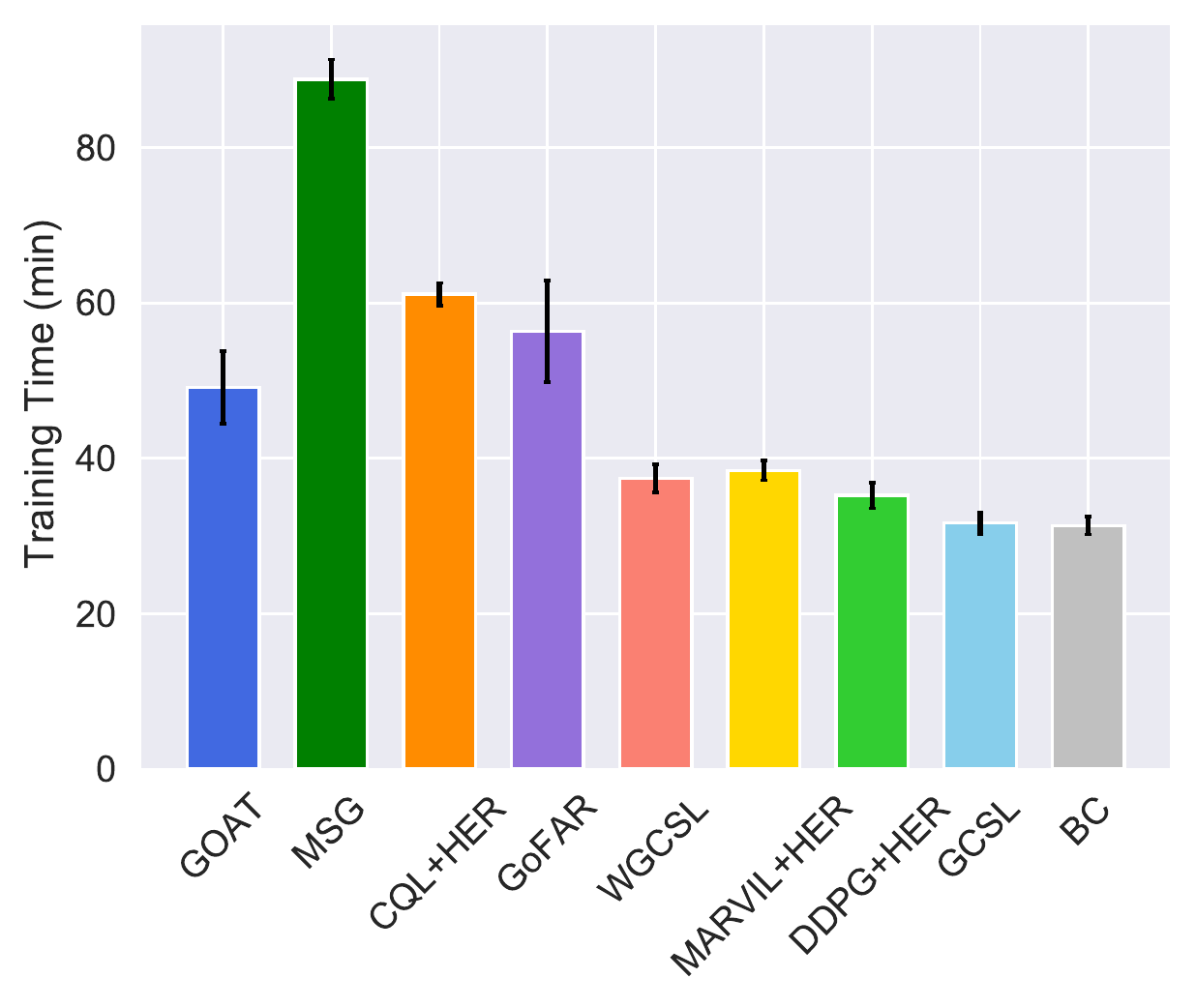}
    \caption{Comparison of training time on FetchPush task.}
    \label{fig:comparison_time}
\end{figure}

\subsection{Random Network Distillation as the Uncertainty Measurement}
The uncertainty weight is an empirical instance to estimate the density under our framework. Random Network Distillation (RND) \cite{burda2018exploration} is a potential alternative for density estimation in the continuous state space. We implement a RND version of GOAT for a comparison and use the same hyperparameter search range. In Table \ref{tab:res_rnd_goat}, GOAT(RND) outperforms WGCSL, but is worse than the version with ensemble Q functions, which may be because naive applications of RND do not yield good results for vector-input tasks \cite{rezaeifar2022offline,nikulin2023anti}.

\begin{table}[h]
\caption{Average success rates ($\%$) with standard deviation over 5 random seeds.}
\label{tab:res_rnd_goat}
\begin{center}
\begin{small}
 \begin{adjustbox}{max width=0.95\linewidth}
\begin{tabular}{llccccccccccc}
\toprule
 Tasks &   GOAT & GOAT(RND) &  WGCSL  \\
\midrule
Reach Left-Right    & 99.5  &  \textbf{99.7}   &  98.9\\

Reach Near-Far  & \textbf{98.8} &  95.1  &  94.5  \\

Push Left-Right &  75.6  & \textbf{76.3}    & 68.0    \\

Push Near-Far & \textbf{70.6} &  69.2  &  69.9  \\

Pick Left-Right & \textbf{92.0}  &  91.6  &  90.0   \\

 Pick Low-High & \textbf{85.8}  &  85.1  & 82.6  \\

 Slide Left-Right & \textbf{57.4}  &  53.1  & 48.3  \\

 Slide Near-Far &  \textbf{53.0}  &  49.3    &  45.2  \\

HandReach Near-Far &  \textbf{55.2} &   53.9  &  50.9   \\

\midrule
 Average IID Tasks &  \textbf{90.3} & 89.4 &  88.1\\
 Average OOD Tasks  & \textbf{67.9} & 66.0  & 62.1  \\
\bottomrule
\end{tabular}
\end{adjustbox}
\end{small}
\end{center}
\end{table}

\subsection{Full Benchmark Experiments}
As demonstrated in Table \ref{tab:success_rate_all} and Table \ref{tab:return_all}, we include more baselines (i.e., IQL+HER and MARVIL+HER) and additional measure (i.e., average cumulative return) for the benchmark experiments. GOAT achieves the highest average success rate and average return on the benchmark. Other conclusions are consistent with Section \ref{sec:benchmark_results}.

\begin{table*}[h]
\caption{Average success rates ($\%$) with standard deviation over 5 random seeds. Blue lines and purple lines refer to IID and OOD tasks, respectively. Top two scores for each task are highlighted.}
\label{tab:success_rate_all}
\begin{center}
\begin{small}
 \begin{adjustbox}{max width=1\linewidth}
\begin{tabular}{llccccccccccc}
\toprule
Task Group & Task &  GOAT($\tau$)  & GOAT & WGCSL &  GCSL & BC & GoFAR & MARVIL+HER & IQL+HER & DDPG+HER & CQL+HER & MSG+HER  \\
\midrule
\rowcolor{c1!10} \cellcolor{white}  &  Right  & 100.0$\pm$0.0  &  100.0$\pm$0.0  &  100.0$\pm$0.0  &  93.6$\pm$4.3  &  92.0$\pm$3.0  &   100.0$\pm$0.0  &  99.9$\pm$0.2  & 100.0$\pm$0.0  &  99.6$\pm$0.6  &  100.0$\pm$0.0  &  99.4$\pm$0.6  \\
\rowcolor{blue!10} \cellcolor{white} Reach Left-Right &  Left &  99.9$\pm$0.2  & 99.0$\pm$2.0  &  97.8$\pm$4.4  &  36.3$\pm$10.9  &  30.4$\pm$15.2  &  54.2$\pm$9.3   &  75.9$\pm$18.6  &  89.2$\pm$5.1  &  73.8$\pm$27.6  &  94.5$\pm$6.3 &  85.6$\pm$15.7 \\
  &  Average  & \textbf{99.9} & \textbf{99.5}  &  98.9  &  65.0  &  61.2  &  77.1  &  87.9 &  94.6  &  86.7  &  97.2 &  92.5 \\

\midrule
\rowcolor{c1!10} \cellcolor{white}  &  Near & 100.0$\pm$0.0    &  100.0$\pm$0.0  &  100.0$\pm$0.0  &  79.7$\pm$3.0  &  85.3$\pm$4.3  &  100.0$\pm$0.0  &  99.8$\pm$0.4  &  100.0$\pm$0.0  & 95.9$\pm$2.0  &  100.0$\pm$0.0 &  98.6$\pm$2.8  \\
\rowcolor{blue!10} \cellcolor{white} Reach Near-Far  & Far  & 90.9$\pm$1.5   & 97.6$\pm$1.1  &  89.0$\pm$2.1  &  33.5$\pm$5.5  &  37.9$\pm$9.7  &  85.0$\pm$1.9  &  75.4$\pm$3.9  &  84.2$\pm$4.0 &  66.8$\pm$6.9  &  88.0$\pm$2.1 & 77.8$\pm$9.7  \\
  &  Average  &  \textbf{95.4}  & \textbf{98.8}  &  94.5  &  56.6  &  61.6  &  92.5  &  87.6 & 92.1   &  81.4  &  94.0 &  88.2  \\

\midrule
\rowcolor{c1!10} \cellcolor{white}  & Right2Right &  96.2$\pm$1.2   &  95.9$\pm$1.2  &  93.2$\pm$0.9  &  82.1$\pm$3.7  &  78.9$\pm$3.8  &   95.9$\pm$1.4 &  95.0$\pm$1.8  & 97.2$\pm$1.5 &  60.1$\pm$6.0  &  83.3$\pm$2.7  &  92.8$\pm$0.9 \\
\rowcolor{blue!10} \cellcolor{white} &  Right2Left  & 75.6$\pm$3.6 & 69.3$\pm$6.6  &  63.3$\pm$8.9  &  40.1$\pm$6.0  &  25.6$\pm$2.7  &  43.8$\pm$4.7 &  64.9$\pm$8.5  &  67.7$\pm$8.8 & 28.5$\pm$4.3  &  46.2$\pm$7.1 &  52.9$\pm$6.5  \\
\rowcolor{blue!10} \cellcolor{white} Push Left-Right &  Left2Right & 78.8$\pm$6.8   &  76.0$\pm$7.4  &  67.6$\pm$7.1  &  38.8$\pm$6.8  &  33.5$\pm$8.1  &  59.7$\pm$4.3   &  67.1$\pm$2.9  &  71.4$\pm$9.1 & 20.6$\pm$11.5  &  40.4$\pm$12.1 &  59.3$\pm$7.7  \\
\rowcolor{blue!10} \cellcolor{white} &  Left2Left   & 75.6$\pm$12.1 &  61.1$\pm$7.6  &  47.7$\pm$7.4  &  35.4$\pm$6.6  &  20.9$\pm$3.2  &  32.5$\pm$5.8  &  57.9$\pm$4.9  &  58.7$\pm$4.2 & 27.0$\pm$3.8  &  34.9$\pm$5.9  &  38.8$\pm$7.9 \\
& Average  &   \textbf{81.5} &  \textbf{75.6}  &  68.0  &  49.1  &  39.7  &  58.0  &  71.2  & 73.8 &  34.1  &  51.2 & 61.0  \\

\midrule
\rowcolor{c1!10} \cellcolor{white} & Near2Near  &  97.2$\pm$0.7 &  92.0$\pm$2.6  &   93.5$\pm$1.0  &  77.6$\pm$4.7  &  67.5$\pm$3.6  & 92.6$\pm$2.2 &  89.4$\pm$3.0  &  96.0$\pm$0.9  &  39.3$\pm$22.4  &  77.7$\pm$3.9 &  84.7$\pm$6.1 \\
\rowcolor{blue!10} \cellcolor{white} & Near2Far & 78.4$\pm$3.5  &  70.3$\pm$5.7   &  67.0$\pm$5.4  &  43.1$\pm$7.2  &  24.9$\pm$5.9  & 60.9$\pm$3.8  &  42.4$\pm$6.4  &  74.0$\pm$2.6 &  30.5$\pm$12.1  &  60.0$\pm$6.2  &  58.4$\pm$2.1 \\
\rowcolor{blue!10} \cellcolor{white} Push Near-Far & Far2Near  &  70.5$\pm$2.4 & 69.5$\pm$3.6  &  68.0$\pm$2.4  &  47.4$\pm$3.5  &  40.2$\pm$7.5  & 65.0$\pm$4.8 &  57.9$\pm$2.6  &  68.8$\pm$4.1  & 25.0$\pm$12.8  &  61.1$\pm$4.3  & 56.5$\pm$6.0  \\
\rowcolor{blue!10} \cellcolor{white} & Far2Far  & 55.1$\pm$2.4 &  50.8$\pm$1.8   &  51.1$\pm$4.7  &  27.9$\pm$4.1  &  15.3$\pm$2.7  & 41.3$\pm$3.1  & 25.4$\pm$4.9  &  48.8$\pm$4.1 & 18.0$\pm$7.0  &  47.1$\pm$2.4 & 41.7$\pm$5.4 \\
& Average  & \textbf{75.3} & 70.6  &  69.9  &  49.0  &  37.0  &   65.0  &  53.8  & \textbf{71.9} &  28.2  &  61.5 &  60.3 \\

\midrule
\rowcolor{c1!10} \cellcolor{white} & Right2Right  & 96.5$\pm$1.1   &  97.3$\pm$1.2  &  93.8$\pm$5.3  &  53.4$\pm$14.1  &  52.9$\pm$7.5  &   56.9$\pm$4.3  &  91.5$\pm$2.6  & 88.8$\pm$5.4 &  40.4$\pm$13.1  &  91.9$\pm$6.8 &  94.9$\pm$2.2 \\
\rowcolor{blue!10} \cellcolor{white} & Right2Left  &  87.9$\pm$5.1   &  88.6$\pm$1.1  &  89.4$\pm$3.9  &  20.7$\pm$6.9  &  5.6$\pm$2.1  &  9.3$\pm$1.8  &  58.4$\pm$8.0  & 65.2$\pm$11.9 & 52.7$\pm$14.9  &  82.4$\pm$12.6  & 89.3$\pm$6.8  \\
\rowcolor{blue!10} \cellcolor{white} Pick Left-Right & Left2Right  & 91.4$\pm$2.3    &  93.9$\pm$1.9  &  90.0$\pm$4.1  &  47.0$\pm$10.9  &  37.2$\pm$6.4  &  51.1$\pm$6.5  &  85.2$\pm$4.2  &  80.2$\pm$4.0 &  9.8$\pm$5.7  &  86.4$\pm$8.6 &  60.8$\pm$16.5 \\
\rowcolor{blue!10} \cellcolor{white} & Left2Left & 87.6$\pm$5.7   &  88.3$\pm$3.7  &  87.0$\pm$5.1  &  24.7$\pm$7.8  &  3.3$\pm$1.4  &  6.0$\pm$2.0  &  50.7$\pm$8.8  & 60.7$\pm$9.7 & 26.4$\pm$10.9  &  83.5$\pm$9.1 & 66.9$\pm$7.0 \\
& Average  & \textbf{90.8}   &  \textbf{92.0}  &  90.0  &  36.4  &  24.8  &  30.8  &  71.4 & 73.7  &  32.3  &  86.1 & 78.0  \\

\midrule
\rowcolor{c1!10} \cellcolor{white}  & Low  & 99.3$\pm$0.5     &  99.8$\pm$0.2  &  98.6$\pm$1.3  &  84.4$\pm$3.6  &  72.4$\pm$5.4  &  95.2$\pm$1.6  &  98.9$\pm$0.6  &  98.0$\pm$0.5 & 50.4$\pm$23.9  &  100.0$\pm$0.0 &  97.3$\pm$2.2 \\
\rowcolor{blue!10} \cellcolor{white}  Pick Low-High & High  & 78.3$\pm$6.3     &  71.9$\pm$6.4  &  66.6$\pm$6.6  &  28.4$\pm$6.9  &  3.0$\pm$1.6  &  7.6$\pm$3.1  &  64.5$\pm$9.2  & 58.0$\pm$12.0  & 17.0$\pm$10.2  &  44.6$\pm$9.2 & 23.3$\pm$7.8 \\
& Average  & \textbf{88.8} & \textbf{85.8}  &  82.6  &  56.4  &  37.7  & 51.4  &  81.7  & 78.0 &  33.7  &  72.3 & 60.3  \\

\midrule
\rowcolor{c1!10} \cellcolor{white}  & Right2Right & 82.0$\pm$3.2  &  79.0$\pm$5.8  &  70.8$\pm$13.5  &  62.2$\pm$7.0  &  60.3$\pm$4.7  &   62.6$\pm$8.7  &  76.5$\pm$3.1  &  76.7$\pm$3.6 &  4.7$\pm$1.5  &  20.3$\pm$2.5  & 20.8$\pm$5.0 \\
\rowcolor{blue!10} \cellcolor{white}  & Right2Left &  45.1$\pm$8.8     &  41.3$\pm$7.1  &  36.2$\pm$8.6  &  11.5$\pm$2.0  &  15.7$\pm$6.0  &  31.6$\pm$3.9  &  43.5$\pm$6.2  &  43.8$\pm$4.6 & 0.3$\pm$0.4  &  8.6$\pm$3.0 & 7.3$\pm$4.9 \\
\rowcolor{blue!10} \cellcolor{white}  Slide Left-Right & Left2Right &   79.6$\pm$2.7    &  59.0$\pm$7.6  &  50.7$\pm$12.7  &  29.1$\pm$4.8  &  41.8$\pm$7.2  &  51.0$\pm$10.5  &  55.5$\pm$5.7  & 71.6$\pm$3.1  &  0.2$\pm$0.2  &  1.7$\pm$0.7 & 3.6$\pm$4.3 \\
\rowcolor{blue!10} \cellcolor{white}  & Left2Left  & 52.5$\pm$8.3    &  50.1$\pm$9.5  &  35.3$\pm$11.3  &  25.5$\pm$5.4  &  33.7$\pm$10.6  &  28.2$\pm$2.6   &  39.3$\pm$7.5  & 43.9$\pm$3.8 &  2.1$\pm$1.1  &  4.3$\pm$2.5  &  7.1$\pm$3.3 \\
& Average  &  \textbf{64.8}  & 57.4  &  48.3  &  32.1  &  37.9  &  43.4  &  53.7  &  \textbf{59.0} & 1.8  &  8.7 &  9.7 \\

\midrule
\rowcolor{c1!10} \cellcolor{white}  & Near   &  77.4$\pm$4.5 & 76.9$\pm$3.3  &  73.1$\pm$5.8  &  28.0$\pm$7.1  &  26.6$\pm$8.3  &  69.3$\pm$2.8  &  73.7$\pm$6.2  & 80.2$\pm$3.2 &  11.3$\pm$4.5  &  43.5$\pm$3.3 &  28.3$\pm$9.5  \\
\rowcolor{blue!10} \cellcolor{white}  Slide Near-Far & Far     &  25.1$\pm$3.9  & 29.0$\pm$4.5  &  17.4$\pm$3.2  &  0.0$\pm$0.0  &  0.0$\pm$0.0  & 24.1$\pm$2.9  &  10.8$\pm$3.6  &  6.8$\pm$1.3 & 4.4$\pm$3.7  &  7.4$\pm$3.8 & 2.6$\pm$1.4  \\
& Average  &  \textbf{51.2} & \textbf{53.0}  &  45.2  &  14.0  &  13.3  &  46.7  &  42.2  & 43.5  & 7.8  &  25.5 &  15.4 \\

\midrule
\rowcolor{c1!10} \cellcolor{white}  & Near   & 72.6$\pm$5.3&  71.9$\pm$3.2  &  70.0$\pm$3.6  &  0.0$\pm$0.0  &  0.0$\pm$0.0  &  77.4$\pm$1.7  &  72.2$\pm$4.0  & 67.2$\pm$9.8 & 0.0$\pm$0.0  &  1.8$\pm$3.6  & 0.0$\pm$0.0  \\
\rowcolor{blue!10} \cellcolor{white} HandReach Near-Far & Far &  33.1$\pm$4.5&  38.4$\pm$4.1  &  31.8$\pm$3.8  &  0.1$\pm$0.2  &  0.0$\pm$0.0  &  36.9$\pm$3.1   &  28.3$\pm$6.1  & 27.5$\pm$3.8 & 0.0$\pm$0.0  &  0.0$\pm$0.0 & 0.0$\pm$0.0  \\
& Average  &  52.8 &  \textbf{55.2}  &  50.9  &  0.0  &  0.0  &    \textbf{57.1}   & 50.3  & 47.4 &  0.0  &  0.9 & 0.0  \\

\midrule
\multirow{2}{*}{Average} &  IID Tasks & \textbf{91.2} & \textbf{90.3} &  88.1 & 62.3 & 59.5  & 83.3 & 88.5 & 89.3 & 44.6 & 68.7 & 68.5 \\
  & OOD Tasks  & \textbf{70.9} & \textbf{67.9}  & 62.1  &  28.8  & 21.7  & 40.5  &  53.1  & 60.0 &  23.7 & 46.5 & 43.1 \\
\bottomrule
\end{tabular}
\end{adjustbox}
\end{small}
\end{center}
\end{table*}

\begin{table*}[ht]
\caption{Average cumulative return with standard deviation over 5 random seeds. Blue lines refer to IID tasks and purple lines indicate OOD tasks. Top two scores for each task are highlighted. }
\label{tab:return_all}
\begin{center}
\begin{small}
 \begin{adjustbox}{max width=\linewidth}
\begin{tabular}{llccccccccccc}
\toprule
Task Group & Task & GOAT($\tau$) & GOAT& WGCSL &  GCSL & BC & GoFAR & MARVIL+HER & IQL+HER & DDPG+HER & CQL+HER & MSG+HER \\
\midrule
\rowcolor{c1!10} \cellcolor{white}  &  Right  & 46.5$\pm$0.1  &  46.4$\pm$0.0  &  46.5$\pm$0.0  &  39.8$\pm$1.9  &  40.4$\pm$2.0  &  46.7$\pm$0.2  &  45.1$\pm$0.3  & 45.8$\pm$0.1  & 46.7$\pm$0.3  &  46.1$\pm$0.2 & 46.6$\pm$0.3 \\
\rowcolor{blue!10} \cellcolor{white} Reach Left-Right &  Left  & 46.2$\pm$0.2  &  45.8$\pm$0.9  &  45.1$\pm$1.9  &  16.1$\pm$4.3  &  14.0$\pm$6.5  &  25.5$\pm$4.3  &  33.8$\pm$8.4  & 40.7$\pm$2.5 &  34.8$\pm$12.3  &  42.9$\pm$3.0  & 40.1$\pm$7.3 \\
  &  Average   & \textbf{46.4} &  \textbf{46.1}  &   45.8  &  28.0  &  27.2  &  36.1   &  39.5 & 43.3  &  40.8  &  44.5  &  43.3  \\

\midrule
\rowcolor{c1!10} \cellcolor{white}  &  Near  &  46.7$\pm$0.1  & 46.7$\pm$0.0  &  46.7$\pm$0.1  &  32.6$\pm$1.1  &  35.9$\pm$1.4  &  47.0$\pm$0.2   &  45.4$\pm$0.2  & 45.9$\pm$0.1 & 45.6$\pm$1.2  &  46.4$\pm$0.1 &   46.7$\pm$1.1 \\
\rowcolor{blue!10} \cellcolor{white} Reach Near-Far   &  Far  & 39.4$\pm$0.8   &  43.1$\pm$0.8  &  38.6$\pm$1.0  &  10.5$\pm$1.6  &  14.5$\pm$2.8  &  37.1$\pm$1.0   &  31.2$\pm$1.6  &  35.4$\pm$1.9 & 30.6$\pm$2.7  &  37.8$\pm$1.0 & 35.1$\pm$4.2 \\
  &  Average  & \textbf{43.0} & \textbf{44.9}  &  42.6  &  21.5  &  25.2  &  42.0  &  38.3  &  40.6 & 38.1  &  42.1  & 40.9  \\

\midrule
\rowcolor{c1!10} \cellcolor{white}  & Right2Right  & 39.4$\pm$0.5    &  38.9$\pm$1.0  &  38.0$\pm$0.3  &  28.5$\pm$1.8  &  26.8$\pm$2.2  &  39.1$\pm$0.9   &  37.7$\pm$0.5  &   39.2$\pm$0.5&  25.0$\pm$2.1  &  33.9$\pm$1.5 &  37.8$\pm$0.6 \\
\rowcolor{blue!10} \cellcolor{white} &  Right2Left  & 27.2$\pm$1.4     &  24.9$\pm$2.4  &  22.6$\pm$3.0  &  11.5$\pm$1.8  &  7.5$\pm$1.0  &  14.8$\pm$1.4  &  21.2$\pm$2.5  & 23.0$\pm$3.4 &  10.4$\pm$2.1  &  16.9$\pm$2.6  & 19.6$\pm$2.7 \\
\rowcolor{blue!10} \cellcolor{white} Push Left-Right & Left2Right  & 25.9$\pm$3.0   &  24.3$\pm$2.7  &  21.3$\pm$2.1  &  10.4$\pm$2.0  &  8.9$\pm$2.6  &  18.2$\pm$2.5  &  20.0$\pm$1.6  & 21.7$\pm$3.2 &  6.8$\pm$3.5  &  12.1$\pm$3.8  & 18.8$\pm$3.2  \\
\rowcolor{blue!10} \cellcolor{white} &  Left2Left  & 29.3$\pm$5.0   &  23.0$\pm$2.8  &  18.8$\pm$3.2  &  12.6$\pm$1.7  &  8.5$\pm$1.4  &  12.7$\pm$1.8 &  21.5$\pm$1.8  & 21.4$\pm$1.8 &  11.4$\pm$1.4  &  14.2$\pm$2.3 & 15.8$\pm$2.6  \\
& Average  & \textbf{30.5}  &  \textbf{27.8}  &   25.2  &  15.7  &  12.9  &  21.2 &  25.1  & 26.3 &  13.4  &  19.2 & 23.0  \\

\midrule
\rowcolor{c1!10} \cellcolor{white} & Near2Near  & 37.8$\pm$0.6     &  34.9$\pm$1.3   &  35.9$\pm$0.2  &  25.6$\pm$2.1  &  21.6$\pm$1.4  &  36.1$\pm$0.9   &  35.6$\pm$0.8  & 36.6$\pm$0.7 & 13.9$\pm$8.0  &  29.2$\pm$1.9  & 31.4$\pm$3.0 \\
\rowcolor{blue!10} \cellcolor{white} & Near2Far  & 27.9$\pm$1.2    &  25.3$\pm$2.1  &  24.1$\pm$2.2  &  12.8$\pm$2.4  &  7.2$\pm$1.3  &  20.9$\pm$1.6  &  23.1$\pm$1.3  &  25.5$\pm$1.2 & 11.0$\pm$3.8  &  21.4$\pm$2.4 &  21.0$\pm$0.9 \\
\rowcolor{blue!10} \cellcolor{white} Push Near-Far & Far2Near  &  22.9$\pm$0.9   &  23.0$\pm$1.4 &  22.4$\pm$1.1  &  12.7$\pm$1.5  &  10.2$\pm$1.9  &  21.1$\pm$1.7 &  20.5$\pm$1.1  & 22.1$\pm$1.3  & 7.8$\pm$4.0  &  20.3$\pm$1.5  & 18.1$\pm$1.9 \\
\rowcolor{blue!10} \cellcolor{white} & Far2Far  & 17.3$\pm$0.5  & 16.2$\pm$0.6  &  16.4$\pm$1.4  &  7.5$\pm$1.5  &  4.3$\pm$0.7  &  12.8$\pm$1.3   &  14.2$\pm$0.9  & 14.9$\pm$1.0 & 6.3$\pm$2.2  &  15.3$\pm$1.0  & 13.5$\pm$2.3 \\
& Average  & \textbf{26.5} & \textbf{24.9}  &   24.7  &  14.7  &  10.8  & 22.7  &  23.4  & 24.8 &  9.8  &  21.6 & 21.0  \\

\midrule
\rowcolor{c1!10} \cellcolor{white} & Right2Right  & 36.8$\pm$0.2    &  36.7$\pm$0.6  &  36.1$\pm$1.0  &  18.5$\pm$4.6  &  16.6$\pm$2.6  &  24.0$\pm$2.1  &  34.0$\pm$0.4  & 32.3$\pm$1.3 & 15.2$\pm$4.4  &  36.1$\pm$1.7  & 35.9$\pm$0.8 \\
\rowcolor{blue!10} \cellcolor{white} & Right2Left  & 32.6$\pm$1.8    &  32.3$\pm$0.7  &  32.8$\pm$1.2  &  6.6$\pm$2.5  &  1.3$\pm$0.5  &   3.2$\pm$0.7    &  22.8$\pm$2.0  & 23.4$\pm$2.9  & 19.6$\pm$5.2  &  32.0$\pm$4.0  & 32.7$\pm$2.3 \\
\rowcolor{blue!10} \cellcolor{white} Pick Left-Right & Left2Right   & 32.5$\pm$1.6   &  33.4$\pm$0.6  &  32.6$\pm$0.8  &  14.7$\pm$3.1  &  10.7$\pm$1.3  &   18.7$\pm$2.4  &  28.2$\pm$1.3  &  27.1$\pm$1.4  &  2.9$\pm$1.8  &  30.8$\pm$2.7 & 19.2$\pm$5.3 \\
\rowcolor{blue!10} \cellcolor{white} & Left2Left  & 32.5$\pm$2.3   &  32.3$\pm$1.5  &  31.7$\pm$2.1  &  8.6$\pm$2.6  &  1.2$\pm$0.4  &  2.2$\pm$0.9   &  20.3$\pm$2.4  &  21.9$\pm$3.3 &  8.6$\pm$3.3  &  31.9$\pm$2.4 & 23.1$\pm$2.1 \\
& Average  &  \textbf{33.6} & \textbf{33.7}  &  33.3  &  12.1  &  7.5  &  12.0  &  26.3  &  26.2 & 11.6  &  32.7 &  27.7 \\

\midrule
\rowcolor{c1!10} \cellcolor{white}  & Low    & 40.0$\pm$0.1 & 40.2$\pm$0.2  &  39.8$\pm$0.3  &  24.9$\pm$1.3  &  24.1$\pm$1.6  &  38.8$\pm$0.6   &  38.0$\pm$0.3  &  37.4$\pm$0.3  & 19.4$\pm$9.1  &  40.9$\pm$0.2 & 39.5$\pm$1.0 \\
\rowcolor{blue!10} \cellcolor{white}  Pick Low-High & High  & 28.2$\pm$2.3   &  26.2$\pm$2.4  &  23.9$\pm$2.2  &  8.4$\pm$1.3  &  0.7$\pm$0.3  &  2.5$\pm$0.8  &  21.9$\pm$2.9  &  20.0$\pm$4.4 & 6.0$\pm$3.7  &  16.1$\pm$2.9 & 8.7$\pm$2.8 \\
&Average  & \textbf{34.1}  & \textbf{33.2}  &  31.9  &  16.7  &  12.4  &  20.7  &  30.0  &  28.7  & 12.7  &  28.5 & 24.1  \\

\midrule
\rowcolor{c1!10} \cellcolor{white}  & Right2Right  & 23.6$\pm$2.2   &  23.8$\pm$1.5  &  21.8$\pm$3.5  &  13.8$\pm$1.3  &  13.6$\pm$1.4  &   25.5$\pm$0.8   &  22.2$\pm$2.4  & 26.7$\pm$0.8 & 2.4$\pm$0.9  &  8.6$\pm$1.2  & 9.2$\pm$1.9 \\
\rowcolor{blue!10} \cellcolor{white}  & Right2Left & 12.0$\pm$2.2 & 10.1$\pm$1.8  &  11.0$\pm$3.8  &  2.0$\pm$0.3  &  2.9$\pm$1.2  &  10.1$\pm$1.2   &  10.5$\pm$2.2  &  10.8$\pm$0.8  & 0.1$\pm$0.2  &  3.0$\pm$1.1 & 2.5$\pm$1.2 \\
\rowcolor{blue!10} \cellcolor{white}  Slide Left-Right & Left2Right  &  17.8$\pm$0.6   &  14.4$\pm$1.4  &  12.2$\pm$2.1  &  6.5$\pm$1.5  &  8.6$\pm$1.6  &  14.9$\pm$2.8   &  13.1$\pm$2.0  &  19.3$\pm$1.1 & 0.2$\pm$0.2  &  0.6$\pm$0.2 &  1.1$\pm$0.8 \\
\rowcolor{blue!10} \cellcolor{white}  & Left2Left  & 13.3$\pm$1.4   &  13.3$\pm$2.2  &  10.9$\pm$2.9  &  5.5$\pm$1.0  &  7.8$\pm$2.6  &  9.1$\pm$2.5  &  10.7$\pm$1.1  & 12.8$\pm$0.8 & 1.2$\pm$0.5  &  2.4$\pm$0.8  &  3.3$\pm$1.0  \\
& Average   & \textbf{16.7} &  15.4  &  14.0  &  7.0  &  8.2  &  14.9  &  14.1  & \textbf{17.4} & 1.0  &  3.7 &  4.0  \\

\midrule
\rowcolor{c1!10} \cellcolor{white}  & Near   &  21.3$\pm$1.2  &  22.2$\pm$1.7  &  21.5$\pm$1.6  &  5.2$\pm$1.3  &  5.4$\pm$1.6  &  22.4$\pm$1.3  &  17.9$\pm$2.4  &   21.6$\pm$1.1 & 5.5$\pm$2.0  &  14.8$\pm$2.3 &  9.4$\pm$2.5  \\
\rowcolor{blue!10} \cellcolor{white}  Slide Near-Far & Far    & 5.0$\pm$0.9 &    5.4$\pm$0.8  &  3.5$\pm$1.2  &  0.0$\pm$0.0  &  0.0$\pm$0.0  &  4.4$\pm$0.9&  1.6$\pm$0.5  & 1.0$\pm$0.1 & 1.3$\pm$0.9  &  2.5$\pm$1.4  & 0.7$\pm$0.4 \\
& Average   & 13.2 &  \textbf{13.8}  &  12.5  &  2.6  &  2.7  &  \textbf{13.4}  &  9.8  & 11.3 & 3.4  &  8.6 & 5.1   \\

\midrule
\rowcolor{c1!10} \cellcolor{white}  & Near   & 33.6$\pm$2.0  & 32.9$\pm$1.0  &  32.7$\pm$1.6  &  0.1$\pm$0.0  &  0.1$\pm$0.1  &  36.0$\pm$0.8  &  32.4$\pm$1.8  &  30.4$\pm$4.0  & 0.0$\pm$0.0  &  0.8$\pm$1.6 &  0.2$\pm$0.1 \\
\rowcolor{blue!10} \cellcolor{white} HandReach Near-Far & Far  & 15.0$\pm$1.8 & 17.0$\pm$1.7  &  14.2$\pm$1.4  &  0.0$\pm$0.1  &  0.0$\pm$0.0  &  16.7$\pm$1.3   &  12.1$\pm$2.5  & 11.8$\pm$1.5 & 0.0$\pm$0.0  &  0.0$\pm$0.0  & 0.0$\pm$0.0 \\
& Average   & 24.3 & \textbf{25.0}  &  23.5  &  0.0  &  0.0  &  \textbf{26.3}  &  22.2  & 21.1 &  0.0  &  0.4 & 0.1  \\

\midrule
\multirow{2}{*}{Average}  & IID Tasks & \textbf{36.2} & \textbf{35.9} & 35.4 & 21.0 & 20.5 & 35.1 & 34.3 & 35.1 & 19.3 &  28.5  & 28.5 \\
  & OOD Tasks &  \textbf{25.0} &\textbf{24.1}  & 22.5  &  8.6 & 6.4 & 14.4  & 19.2 & 20.8 & 9.4 & 17.7 & 16.1 \\
\bottomrule
\end{tabular}
\end{adjustbox}
\end{small}
\end{center}
\end{table*}

\section{Additional Related Works}

In ML community, there are different types of OOD studied by prior works \cite{nair2020contextual,ma2022vip,han2021learning,hansen2022bisimulation,pitismocoda}, e.g., handling spurious feature, assuming Factored MDPs, or learning generalizable representations for different objects or scenes. Different from these works, our work focus on the OOD goal generalization problem, which is essentially a type of covariate shift. In practical applications, more than one type of OOD is generally involved. The work \cite{hong2022bi} studies a similar goal generalization setting of our work, but it is in the online setting with exploration. Instead, we consider the offline setting, where online interaction is prohibited and commonly used pessimism-based method can inhibit OOD generalization.


\end{document}